\documentclass{article}
\PassOptionsToPackage{numbers}{natbib}

 \usepackage[main, final]{neurips_2026}
\usepackage[utf8]{inputenc} % allow utf-8 input
\usepackage[T1]{fontenc}    % use 8-bit T1 fonts
\usepackage{hyperref}       % hyperlinks
\hypersetup{
    pdftitle={Counterfactual Online Conformal Prediction Under Adaptive Logging},
    pdfauthor={Xinyu Qiao, Yichen Lin, Kaihong Ji, Xue Wang, Tao Yao},
    colorlinks=true,
    linkcolor=blue,
    filecolor=magenta,
    urlcolor=cyan,
    citecolor=blue,
}

\usepackage{url}            % simple URL typesetting
\usepackage{booktabs}       % professional-quality tables
\usepackage{amsfonts}       % blackboard math symbols
\usepackage{nicefrac}       % compact symbols for 1/2, etc.
\usepackage{microtype}      % microtypography
\usepackage{xcolor}         % colors
\usepackage{graphicx}       % graphics
\usepackage{amsmath,amssymb,amsthm}
\usepackage{mathtools}
\usepackage{algorithm}
\usepackage{algorithmic}
\usepackage{enumitem}
\usepackage{tikz}
\usepackage{wrapfig}
\usepackage{subcaption}

\usetikzlibrary{arrows.meta,positioning,fit,backgrounds}

\theoremstyle{plain}
\newtheorem{theorem}{Theorem}[section]
\newtheorem{proposition}[theorem]{Proposition}
\newtheorem{lemma}[theorem]{Lemma}
\newtheorem{corollary}[theorem]{Corollary}
\theoremstyle{definition}
\newtheorem{definition}[theorem]{Definition}
\newtheorem{assumption}[theorem]{Assumption}
\theoremstyle{remark}
\newtheorem{remark}[theorem]{Remark}

\newcommand{\R}{\mathbb{R}}
\newcommand{\E}{\mathbb{E}}
\newcommand{\Prob}{\mathbb{P}}
\newcommand{\1}{\mathbf{1}}
\newcommand{\cC}{\mathcal{C}}

\newcommand{\cX}{\mathcal{X}}
\newcommand{\cY}{\mathcal{Y}}
\newcommand{\cA}{\mathcal{A}}
\newcommand{\cH}{\mathcal{H}}
\newcommand{\cF}{\mathcal{F}}
\newcommand{\cN}{\mathcal{N}}

\newcommand{\cS}{\mathcal{S}}
\newcommand{\cW}{\mathcal{W}}
\newcommand{\DTV}{D_{\mathrm{TV}}}
\newcommand{\MCov}{\mathrm{MCov}}
\newcommand{\ACov}{\mathrm{ACov}}
\newcommand{\CCov}{\mathrm{CCov}}
\DeclareMathOperator*{\argmin}{arg\,min}

\newcommand{\rowlabel}[1]{\makebox[1.5em]{\rotatebox{90}{\small{#1}}}}

\title{Counterfactual Online Conformal Prediction Under Adaptive Logging}

\author{%
  Xinyu Qiao\textsuperscript{1} \quad
  Yichen Lin\textsuperscript{1} \quad
  Kaihong Ji\textsuperscript{1} \quad
  Xue Wang\textsuperscript{2} \quad
  Tao Yao\textsuperscript{1}\thanks{Corresponding author.} \\[0.5ex]
  {\normalfont\textsuperscript{1}Shanghai Jiao Tong University, Shanghai, China} \\
  {\normalfont\textsuperscript{2}Meta} \\[0.5ex]
  {\normalfont\textsuperscript{1}\texttt{\{qxinyu25,linyichen,jkh040607,taoyao\}@sjtu.edu.cn}} \\
  {\normalfont\textsuperscript{2}\texttt{xuewang0725@meta.com}}
}

\begin{document}
\maketitle

\begin{abstract}
Online conformal prediction can fail when predictions shape actions and actions determine which outcomes enter calibration. Standard adaptive methods may retain marginal coverage while systematically miscovering the counterfactual outcomes of rarely selected actions. This paper formalizes the failure through counterfactual coverage and introduces Propensity-Weighted Online Conformal Prediction, an inverse-propensity-weighted recursion that debiases calibration. A doubly robust variant further reduces nuisance bias to the product of outcome-model and propensity errors. Under positivity and known logging probabilities, both methods attain the minimax expected coverage rate for nondegenerate calibration. Experiments on synthetic decision tasks, Open Bandit data, and financial rebalancing demonstrate counterfactual coverage gains under selection-biased logging, with prediction-set widths close to those of the oracle in the reported sharpness study.
\end{abstract}

\section{Introduction}\label{sec:intro}
Modern decision systems use conformal prediction to guide decisions in clinical trials, content recommendation, and inventory pricing \citep{lei2021conformal,taufiq2022conformal,patel2023conformal,bao2025optimal}. These decisions also shape subsequent calibration data: a logging policy selects which action is taken, and only that action's outcome enters calibration. Standard online conformal methods \citep{gibbs2021adaptive,gibbs2022conformal,angelopoulos2023conformal} preserve average coverage on the realized stream, but the rarely selected actions, often exactly the alternatives a downstream rule must compare, may remain miscalibrated. Context-dependent selection, rather than rarity alone, drives this bias and can also affect frequently selected actions.

The relevant target is therefore not the marginal coverage of the played action but the \emph{counterfactual coverage} of every action: would the prediction set have covered the outcome had that action been chosen \citep{lei2021conformal,taufiq2022conformal,zhang2023conformal}? We show that this gap is structural under endogenous logging. Logged-quantile per-action calibration targets the action-conditional quantile of the \emph{logged} score distribution rather than the counterfactual one. In our construction, the two differ by an amount that does not vanish with the horizon.

Our correction, Propensity-Weighted Online Conformal Prediction (PW-OCP), replaces the observed error indicator with an inverse-propensity-weighted estimator of counterfactual miscoverage, realigning the recursion's fixed point with the counterfactual target. A doubly robust variant further reduces nuisance bias to the product of outcome-model and propensity errors. Under positivity and with known propensities, PW-OCP and DR-OCP attain the minimax expected counterfactual coverage rate for nondegenerate calibration. We also bound downstream robust pseudo-regret through a decomposition involving prediction-set accuracy, logged coverage, and exploration.

\textbf{Contributions.} Our contributions are fourfold:

\textbf{(1).} We formulate \emph{counterfactual coverage} for online CP under endogenous logging and prove that per-action calibrators tracking the logged quantile suffer a counterfactual gap of order $\Omega(\rho)$, where $\rho$ quantifies the action-induced shift in the outcome law.

\textbf{(2).} We introduce PW-OCP and DR-OCP, proving high-probability counterfactual coverage bounds and a matching minimax expected rate for nondegenerate calibration with known propensities.

\textbf{(3).} We bound robust pseudo-regret through sharpness, logged coverage, and exploration, and excess oracle robust risk through a bounded prediction-set error. The decompositions allow adaptive logging, with explicit rates under stationary logging and local regularity.

\textbf{(4).} Across seven benchmark settings covering synthetic tasks, Open Bandit data, and financial rebalancing, we evaluate counterfactual coverage, propensity robustness, and prediction-set quality. PW-OCP and DR-OCP substantially reduce coverage gaps in the selection-biased synthetic settings, while gains vary across benchmarks.

\section{Problem Formulation}\label{sec:problem}
Let $\cX$ denote the context space, $\cA$ a finite action space with $|\cA|=K$, and $\cY$ the outcome space. The interaction lasts $T$ rounds. Write $\cH_{t-1}=(X_s,a_s,Y_s)_{s<t}$ for the observed history before round $t$. The logging policy may depend on the observed history and need not be stationary. Within this sequential sampling model, we impose the causal and positivity conditions below.

\textbf{Nonconformity scores and prediction sets.} Fix a nonconformity score $s:\cX\times\cY\times\cA\to\R_{\geq 0}$ \citep{vovk2005algorithmic}; standard choices include residual scores $s(x,y;a)=|y-\hat\mu(x,a)|$ and scores used in conformalized quantile regression \citep{romano2019conformalized}. Given a threshold $\hat q_t(a)$, the prediction set for action $a$ at round $t$ is
\begin{align}\label{eq:pred-set}
    \cC_t(a)=\{y\in\cY: s(X_t,y;a)\leq \hat q_t(a)\}.
\end{align}
The learner maintains one threshold $\hat q_t(a)$ per action at target miscoverage level $\alpha\in(0,1)$.

\textbf{Sequential interaction.} At each round $t$:
\begin{enumerate}[leftmargin=1.5em,itemsep=1pt,topsep=2pt]
    \item the environment reveals $X_t$;
    \item the learner outputs $\{\cC_t(a)\}_{a\in\cA}$;
    \item the logger draws $a_t\sim\pi_t(\cdot\mid X_t,\cH_{t-1})$, where $\pi_t(a\mid X_t,\cH_{t-1}):=\Prob(a_t=a\mid X_t,\cH_{t-1})$;
    \item the environment reveals $Y_t\sim P(\cdot\mid X_t,a_t)$.
\end{enumerate}
The classical online conformal setting corresponds to $P(\cdot\mid X_t,a_t)=P(\cdot\mid X_t)$ \citep{gibbs2021adaptive}.

\subsection{Potential Outcomes and Structural Assumptions}
For each $a\in\cA$, let $Y_t(a)\in\cY$ denote the potential outcome under the intervention $a_t=a$ \citep{rubin1974estimating}, so that $Y_t=Y_t(a_t)$. The environment specifies a joint law over $(X_t,\{Y_t(a)\}_{a\in\cA})$.
The environment is nonanticipating and may depend on the full past. Conditional on that past, $X_t$, and the current potential outcomes, the logger uses fresh randomization with probabilities $\pi_t(\cdot\mid X_t,\cH_{t-1})$. Thus neither temporal independence of outcomes nor equality of observed-history and full-history outcome kernels is needed for PW-OCP.

\begin{assumption}[No unmeasured confounders]\label{ass:nuc}
$\{Y_t(a)\}_{a\in\cA} \perp\!\!\!\perp a_t\mid X_t,\cH_{t-1}.$
\end{assumption}

This is sequential ignorability for contextual bandits. It holds whenever every variable used by the logger is included in $(X_t,\cH_{t-1})$, e.g., randomized trials and $\varepsilon$-greedy production logging.

\begin{assumption}[Positivity]\label{ass:pos}
$\pi_t(a\mid X_t,\cH_{t-1})\geq \pi_{\min}>0$ for all $t$, $a\in\cA$, and $X_t\in\cX$.
\end{assumption}

\begin{remark}[Design role of $\pi_{\min}$]\label{rem:pos-design}
$\pi_{\min}$ is a known design constant in randomized assignment, A/B allocation, and fairness-constrained exposure. The regime $\pi_{\min}\to 0$ (e.g., UCB, Thompson sampling without forced exploration) is excluded; Theorem~\ref{thm:lower} quantifies the information cost of rare observations for nondegenerate calibration. The floor trades exploration cost against effective sample size $T\pi_{\min}$. In our tested setting, $\pi_{\min}=0.01$ is a practical reference, not a universal threshold.
\end{remark}

\subsection{Coverage Targets}
\begin{definition}[Coverage targets]
\begin{enumerate}[leftmargin=1.5em,itemsep=1pt,topsep=1pt]
    \item \textbf{Marginal coverage.} The realized marginal coverage is
    \begin{align}\label{def:mcov}
        \MCov_T:=\frac{1}{T}\sum_{t=1}^T\1\{Y_t\in \cC_t(a_t)\}.
    \end{align}

    \item \textbf{Action-conditional coverage.} Let $N_T(a)=\sum_{t=1}^T\1\{a_t=a\}$. For each action $a\in\cA$ with $N_T(a)>0$, the action-conditional coverage is
    \begin{align}\label{def:acov}
        \ACov_T(a):=
        \frac{\sum_{t=1}^T\1\{a_t=a,Y_t(a)\in\cC_t(a)\}}{N_T(a)}.
    \end{align}
    If $N_T(a)=0$, this diagnostic is undefined and omitted.

    \item \textbf{Counterfactual coverage.} For each action $a\in\cA$, the counterfactual coverage is
    \begin{align}\label{def:ccov}
        \CCov_T(a):=\frac{1}{T}\sum_{t=1}^T\1\{Y_t(a)\in\cC_t(a)\}.
    \end{align}
\end{enumerate}
\end{definition}

Both $\MCov_T$ and $\ACov_T(a)$ for sampled actions are observable. $\ACov_T(a)$ averages over rounds with $a_t=a$ and therefore reflects the policy-induced context distribution, which can differ from the marginal context distribution relevant to $\CCov_T(a)$. $\CCov_T(a)$ averages counterfactual coverage indicators over all rounds and is the sequential analogue of batch counterfactual coverage \citep{lei2021conformal}.

 The three quantities are not deterministically ordered. When $\rho>0$ (Definition~\ref{def:rho}), Theorem~\ref{thm:impossibility} shows that control of $\MCov_T$ or $\ACov_T(a)$ need not imply control of $\CCov_T(a)$.

\subsection{Endogeneity Coefficient}
\begin{definition}[Endogeneity coefficient]\label{def:rho}
\begin{align}\label{eq:rho}
    \rho := \sup_{x\in\cX}\max_{a\neq a'}\DTV\!\bigl(P(\cdot\mid x,a),\,P(\cdot\mid x,a')\bigr) \;\in\; [0,1],
\end{align}
where $\DTV(P,Q)=\tfrac{1}{2}\int|p(y)-q(y)|\,dy$ is the total-variation distance.
\end{definition}

The case $\rho=0$ means that the outcome law is action-invariant conditional on the context. It does not exclude context selection by the logger or imply exchangeability over time. 
%For $\rho>0$, the logged scores are drawn from a policy-dependent mixture rather than the per-action laws required for counterfactual calibration; calibrating $\{\cC_t(a)\}_{a\in\cA}$ becomes a causal estimation problem under selection-biased sampling.
For $\rho>0$, the logged scores are drawn from a policy-dependent mixture rather than the per-action laws required for counterfactual calibration. Calibrating $\{\cC_t(a)\}_{a\in\cA}$ then becomes a causal estimation problem under selection-biased sampling.

\section{Failure of Standard Online CP Under Endogenous Logging}\label{sec:impossibility}

Standard online conformal methods calibrate the recursion against the realized logged error stream. This controls the logged empirical coverage $\MCov_T$ (Theorem~\ref{thm:marginal}). However, Theorem~\ref{thm:impossibility} shows that calibrators tracking logged per-action quantiles can retain a counterfactual coverage gap that does not vanish with $T$.

\subsection{Logged Empirical Coverage Survives}

\begin{theorem}[Long-run empirical coverage robustness, \citep{gibbs2021adaptive}]\label{thm:marginal}
Fix $\alpha\in(0,1)$, step size $\gamma>0$, and initial value $\alpha_1\in[0,1]$. Let $\hat Q_t:\R\to\R\cup\{\pm\infty\}$ be a nondecreasing $(\cH_{t-1},X_t)$-measurable score-quantile map satisfying the ACI boundary convention
\begin{align}\label{eq:aci-boundary}
    \hat Q_t(u)=-\infty \quad\text{for } u\leq 0,\qquad
    \hat Q_t(u)=+\infty \quad\text{for } u\geq 1 .
\end{align}
Set $\hat q_t\!:=\!\hat Q_t(1\!-\!\alpha_t)$, $\cC_t(a)\!:=\!\{y\!\in\!\cY\!:\!s(X_t,y;a)\!\leq\! \hat q_t\}$, $\mathrm{err}_t\!:=\!\1\{Y_t\notin\cC_t(a_t)\}$, and run the ACI update
\begin{align}\label{eq:aci-update}
    \alpha_{t+1}=\alpha_t+\gamma(\alpha-\mathrm{err}_t).
\end{align}
Then, for every sample path $(X_t,a_t,Y_t)_{t=1}^T$,
\begin{align}\label{eq:marginal-bound}
|\MCov_T - (1-\alpha)| \leq
\frac{\max\{\alpha_1,\,1-\alpha_1\}+\gamma}{\gamma T}.
\end{align}
\end{theorem}

% The bound \eqref{eq:marginal-bound} is pathwise and distribution-free. It controls the long-run logged empirical average $\MCov_T$, not a per-round probabilistic marginal-coverage guarantee, and is agnostic to how $\hat Q_t$ is constructed: global, per-arm, or IPW-reweighted quantiles all satisfy the same telescoping identity. 
% This robustness is also its limitation. Because $\MCov_T$ averages over the realized action sequence, \eqref{eq:marginal-bound} can hold even when $\alpha_t$ leaves $[0,1]$ and \eqref{eq:aci-boundary} forces $\cC_t$ to degenerate to $\cY$ or $\emptyset$ on a non-vanishing fraction of rounds. The bound thus imposes no constraint on rarely selected actions. The proof is in Appendix~\ref{app:marginal_proof}.

The bound \eqref{eq:marginal-bound} follows the pathwise telescoping argument of ACI. It controls the logged empirical average $\MCov_T$, not a per-round marginal-coverage guarantee, and is agnostic to how $\hat Q_t$ is constructed: global, per-arm, or IPW-reweighted quantiles all satisfy the same telescoping identity. This robustness does not ensure counterfactual validity. Because $\MCov_T$ averages over the realized action sequence, \eqref{eq:marginal-bound} can hold even when $\alpha_t$ leaves $[0,1]$ and \eqref{eq:aci-boundary} forces $\cC_t$ to degenerate to $\cY$ or $\emptyset$ on a non-vanishing fraction of rounds. The bound thus imposes no constraint on rarely selected actions. The proof is in Appendix~\ref{app:marginal_proof}.

\subsection{Counterfactual Coverage Breaks}
\begin{theorem}[Impossibility for logged-quantile per-action calibration]\label{thm:impossibility}
Fix $\alpha\in(0,1)$. There exist absolute constants $\rho_0,L>0$ and a constant $c(\alpha)>0$ such that, for every $\rho\in(0,\rho_0]$, there is a stationary two-context, two-action Gaussian instance satisfying Assumptions~\ref{ass:nuc}--\ref{ass:pos} with endogeneity coefficient $\rho$ on which the following holds. Let $\{\hat q_t(a)\}_{t,a}$ be any sequence of predictable per-action thresholds under the absolute-residual score, and let $q_a^{\log}$ denote the $(1-\alpha)$ quantile of the score distribution conditional on $a_t=a$. Let $F_{\log}^{(a)}$ denote the CDF of this logged score distribution, and define the tracking error
\begin{align}\label{eq:eps-T}
    \varepsilon_T:=\max_{a\in\cA}\E\!\left[\frac{1}{T}\sum_{t=1}^T |F_{\log}^{(a)}(\hat q_t(a))-(1-\alpha)|\right].
\end{align}
Then
\begin{align}\label{eq:impossibility-bound}
\E\!\left[\max_{a\in\cA} |\CCov_T(a) - (1-\alpha)|\right]
\;\geq\; c(\alpha)\,\rho - L\,\varepsilon_T.
\end{align}
In particular, $\varepsilon_T=O(T^{-1/2})$ implies $\E[\max_a|\CCov_T(a)-(1-\alpha)|]\geq c(\alpha)\rho-O(T^{-1/2})$.
\end{theorem}

Theorem~\ref{thm:impossibility} pinpoints the failure source. The per-action update tracks the logged quantile $q_a^{\log}$ rather than the counterfactual one. In the constructed instance, context-dependent logging and action-dependent outcome laws make the logged and counterfactual quantiles differ by $\Theta(\rho)$, producing a corresponding gap in $\CCov_T(a)$.

The bound \eqref{eq:impossibility-bound} is non-vacuous in the regime $\rho \gg \varepsilon_T$.

The result applies to per-action calibrators through their logged-CDF tracking error. A persistent $\Omega(\rho)$ gap follows when this error is small relative to $\rho$, including for smoothed or aggregated methods whenever they satisfy the same tracking property. The proof, which uses a two-context Gaussian variance-shift construction, is in Appendix~\ref{app:impossibility-proof}. For a concrete example, at $\pi_{\min}=0.001$, ACI's worst-action gap is $0.4785$: a nominal $90\%$ set covers only about $42\%$ of that action's counterfactual outcomes.

\section{Propensity-Weighted Online Conformal Prediction}\label{sec:pwocp}
Theorem~\ref{thm:impossibility} traces standard online CP failure to the calibration target rather than the conformal quantile. PW-OCP corrects the target by replacing the logged error indicator $\mathrm{err}_t$ in the ACI recursion \eqref{eq:aci-update} with an inverse-propensity-weighted estimator of counterfactual miscoverage \citep{horvitz1952generalization,wang2017optimal}.

\subsection{The IPW Error Signal}

\begin{definition}[IPW error signal]\label{def:ipw}
\begin{align}\label{eq:ipw}
     Z_t(a) \;:=\; \frac{\1\{a_t = a\}}{\pi_t(a\mid X_t,\cH_{t-1})} \cdot \1\{Y_t \notin \cC_t(a)\}.
\end{align}
\end{definition}

\begin{figure}[htbp]
    \hspace{1.5em}%
    \begin{subfigure}{0.45\textwidth}\centering\normalsize{\qquad\quad Logged coverage}\end{subfigure}\hfill
    \begin{subfigure}{0.45\textwidth}\centering\normalsize{Counterfactual coverage}\end{subfigure}
    \\[2pt] 
    \rowlabel{\qquad\qquad \normalsize{MCov Gap}}%
     \begin{subfigure}[b]{0.45\textwidth}\centering
        \includegraphics[width=\textwidth]{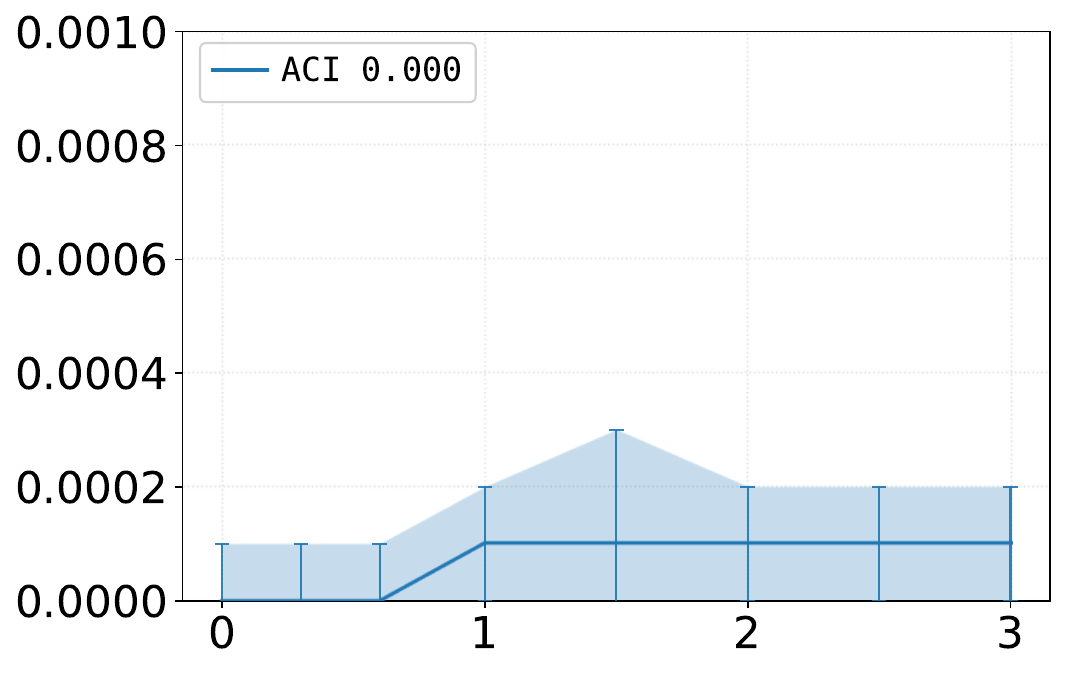}
    \end{subfigure}
    \rowlabel{\qquad\quad\quad \normalsize{CCov Gap}}%
    \begin{subfigure}[b]{0.45\textwidth}\centering
        \includegraphics[width=\textwidth]{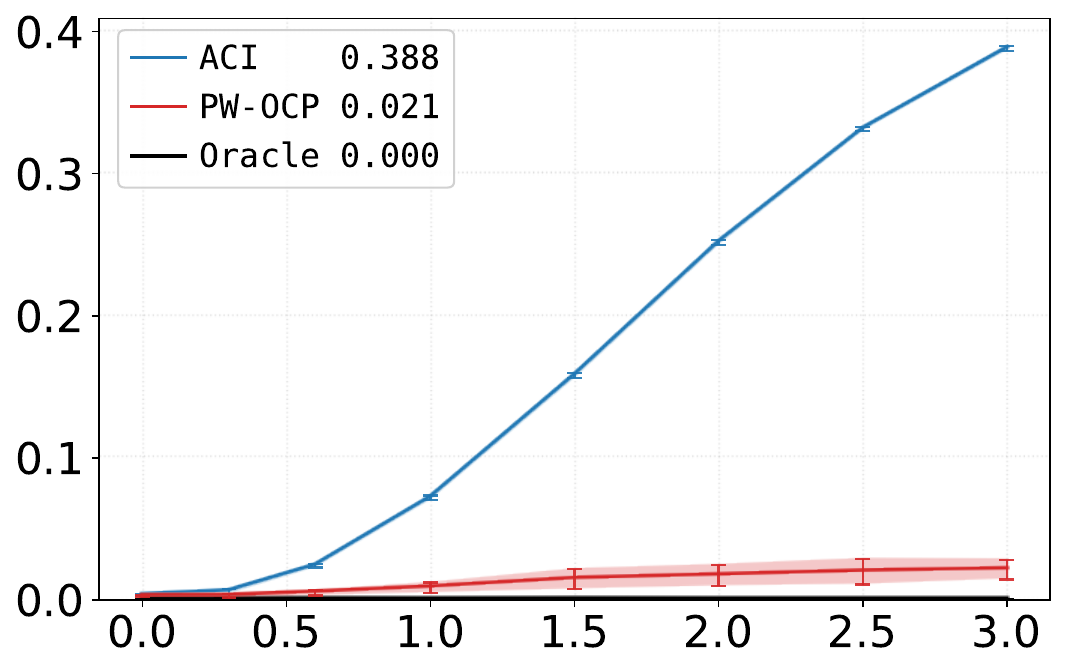}
    \end{subfigure}
    \vskip -0.05in
    \begin{subfigure}{0.45\textwidth}
        \centering
        \qquad \normalsize{\qquad\qquad Endogeneity strength $\Delta$}
    \end{subfigure}
    \begin{subfigure}{0.45\textwidth}
        \centering
        \setlength{\parindent}{0.5em}
        \qquad \normalsize{\qquad\qquad\quad Endogeneity strength $\Delta$}
    \end{subfigure}
    \vskip -0.05in
    \caption{\textbf{Logged coverage can hide counterfactual failure.} Both panels sweep the mean-shift strength $\Delta\in[0,3]$ at $T=20{,}000$ and $\alpha=0.1$. The left panel shows that ACI's logged marginal gap stays essentially zero on a $10^{-3}$ scale. The right panel shows that ACI's counterfactual coverage gap increases with $\Delta$, whereas PW-OCP remains close to Oracle. This illustrates the failure mode described by Theorem~\ref{thm:impossibility} and the effect of propensity weighting. Bands are $95\%$ confidence intervals over $10$ seeds.}
\label{fig:section4-coverage-punchline}
\end{figure}

\begin{lemma}[Unbiasedness]\label{lem:unbiased}
Fix $a\in\cA$ and assume $\cC_t(a)$ is $\sigma(\cH_{t-1},X_t)$-measurable. Under Assumptions~\ref{ass:nuc}--\ref{ass:pos},
\begin{align}\label{eq:ipw-unbiased}
    \E[Z_t(a)\mid X_t,\cH_{t-1}]
    \;=\;
    \Prob\bigl(Y_t(a)\notin\cC_t(a)\mid X_t,\cH_{t-1}\bigr).
\end{align}
\end{lemma}
The propensity weight cancels the logger's selection probability under no unmeasured confounding (Assumption~\ref{ass:nuc}), so $Z_t(a)$ is an unbiased estimator of the round-$t$ counterfactual miscoverage probability. The proof is given in Appendix~\ref{app:unbiased}.

\subsection{Algorithm}
PW-OCP maintains one miscoverage index $\alpha_t(a)\in\R$ per action. Let $\cS_t(a):=\{s(X_\tau,Y_\tau;a):\tau<t,\,a_\tau=a\}$ be the logged score set for action $a$, and let $\hat Q_t^{(a)}$ denote its empirical quantile under the boundary convention~\eqref{eq:aci-boundary}. If $\cS_t(a)=\emptyset$, set $\hat Q_t^{(a)}(u)=+\infty$ for $0<u<1$ and retain~\eqref{eq:aci-boundary} otherwise. The threshold and recursion are
\begin{align}\label{eq:pwocp-update}
    \hat q_t(a):=\hat Q_t^{(a)}\bigl(1-\alpha_t(a)\bigr),
    \qquad
    \alpha_{t+1}(a) := \alpha_t(a) + \gamma\bigl(\alpha - Z_t(a)\bigr),
\end{align}
where $Z_t(a)$ is the IPW signal in Definition~\ref{def:ipw}, evaluated at the prediction set $\cC_t(a):=\{y:s(X_t,y;a)\leq\hat q_t(a)\}$. Each threshold is computed from logged scores for the corresponding action, whereas the recursion is driven by the unbiased IPW estimator of counterfactual miscoverage from Lemma~\ref{lem:unbiased}. Scores are stored in sorted per-action buffers, with $O(n_a)$ insertion for $n_a$ stored scores. Algorithm~\ref{alg:pwocp} summarizes the online procedure. The DR-OCP variant, the Online COPP baseline, and nuisance-estimation details are deferred to Appendix~\ref{app:algorithms}.

\begin{algorithm}[htbp]
\caption{PW-OCP: Propensity-Weighted Online Conformal Prediction}\label{alg:pwocp}
% \small
\begin{algorithmic}[1]
% \STATE \textbf{Input}: target miscoverage $\alpha$, step size $\gamma$
\STATE \textbf{Input}: target miscoverage $\alpha$, step size $\gamma$, logging propensities $\pi_t$
\STATE Initialize $\alpha_1(a)\gets\alpha$ and $\cS_1(a)\gets\emptyset$ for all $a\in\cA$
\FOR{$t=1,\ldots,T$}
    \STATE \textcolor{blue}{{// Construct action-wise prediction sets}}
    \STATE Observe $X_t$
    \FOR{each $a\in\cA$}
        \STATE $\hat q_t(a)\gets\hat Q_t^{(a)}(1-\alpha_t(a))$
        \STATE $\cC_t(a)\gets\{y:s(X_t,y;a)\leq\hat q_t(a)\}$
    \ENDFOR
    \STATE \textcolor{blue}{{// Observe logged feedback}}
    \STATE Draw $a_t\sim\pi_t(\cdot\mid X_t,\cH_{t-1})$ and observe $Y_t$
    % \STATE Let $E_t\gets\1\{s(X_t,Y_t;a_t)>\hat q_t(a_t)\}$
    \STATE $\cS_{t+1}(a_t)\gets\cS_t(a_t)\cup\{s(X_t,Y_t;a_t)\}$
    \STATE $\cS_{t+1}(a)\gets\cS_t(a)$ for all $a\neq a_t$
    \STATE \textcolor{blue}{{// Propensity-weighted calibration}}
    \FOR{each $a\in\cA$}
        \STATE $p_t(a)\gets\pi_t(a\mid X_t,\cH_{t-1})$
        \STATE $Z_t(a)\gets\frac{\1\{a_t=a\}}{p_t(a)}\cdot \1\{s(X_t,Y_t;a_t)>\hat q_t(a_t)\}$
        \STATE $\alpha_{t+1}(a)\gets\alpha_t(a)+\gamma(\alpha-Z_t(a))$
    \ENDFOR
\ENDFOR
\end{algorithmic}
\end{algorithm}

\subsection{Doubly Robust Calibration}
% The IPW signal $Z_t(a)$ has variance scaling as $1/\pi_{\min}$ and requires the true logging propensity. Augmenting it with an outcome model yields a doubly robust signal that remains valid under estimation error in either nuisance. Let

The conditional variance of the IPW signal $Z_t(a)$ is bounded in terms of the inverse propensity floor, so the signal can be noisy when $\pi_{\min}$ is small. It also requires the true logging propensity. Augmenting the IPW signal with an outcome model yields a doubly robust signal whose conditional bias is second order in the nuisance errors. Let
\vspace{-0.5em}
\begin{align}\label{eq:mu-star}
\mu_t^*(a):=\Prob\bigl(Y_t(a)\notin\cC_t(a)\mid X_t,\cH_{t-1}\bigr)
\end{align}
denote the round-$t$ counterfactual miscoverage probability, which is the target identified by~\eqref{eq:ipw-unbiased}. Let $\hat\mu_t(a,X_t)\in[0,1]$ be a $\sigma(\cH_{t-1},X_t)$-measurable estimator of $\mu_t^*(a)$, and let $\hat\pi_t(a\mid X_t,\cH_{t-1})\in[\hat\pi_{\min},1]$ be a $\sigma(\cH_{t-1},X_t)$-measurable propensity estimator with floor $\hat\pi_{\min}>0$.

\begin{definition}[DR error signal]\label{def:dr}
\begin{align}\label{eq:dr-signal}
    Z_t^{\mathrm{DR}}(a) := \hat\mu_t(a,X_t) + \frac{\1\{a_t = a\}}{\hat\pi_t(a\mid X_t,\cH_{t-1})}\bigl(\1\{Y_t \notin \cC_t(a)\} - \hat\mu_t(a,X_t)\bigr).
\end{align}
\end{definition}

DR-OCP replaces $Z_t(a)$ in~\eqref{eq:pwocp-update} with $Z_t^{\mathrm{DR}}(a)$. The construction reduces to PW-OCP when $\hat\mu_t\equiv 0$ and $\hat\pi_t\equiv\pi_t$. Theorem~\ref{thm:drocp} shows that the residual conditional bias equals $(\hat\mu_t-\mu_t^*)(\hat\pi_t-\pi_t)/\hat\pi_t$. Algorithm~\ref{alg:drocp} in Appendix~\ref{app:drocp-algorithm} gives the implementation.

\section{Theoretical Analysis}\label{sec:theory}
This section establishes counterfactual coverage guarantees for PW-OCP and DR-OCP and a matching information-theoretic lower bound for nondegenerate calibration. All proofs are in Appendix~\ref{app:theory}.

\subsection{Coverage Guarantees}\label{sec:coverage}

\begin{theorem}[Counterfactual coverage of PW-OCP]\label{thm:pwocp}
Under Assumptions~\ref{ass:nuc}--\ref{ass:pos}, Algorithm~\ref{alg:pwocp} (recursion~\eqref{eq:pwocp-update}) satisfies, for every $a\in\cA$ and every sample path, the telescoping identity
\begin{align}\label{eq:pwocp-telescope}
    \frac{1}{T}\sum_{t=1}^T Z_t(a) - \alpha \;=\; \frac{\alpha_1(a) - \alpha_{T+1}(a)}{\gamma T}.
\end{align}
With $K=|\cA|$ and $\gamma=c_\gamma\sqrt{\pi_{\min}/T}$ for fixed $c_\gamma>0$, the maximal envelope in Lemma~\ref{lem:alpha-bound} gives, with probability at least $1-\delta$,
\begin{align}\label{eq:pwocp-rate}
\max_a|\CCov_T(a)-(1-\alpha)|
\leq C_{c_\gamma}\left\{\sqrt{\frac{\log(2K/\delta)}{T\pi_{\min}}}
+\frac{\log(2K/\delta)}{T\pi_{\min}}\right\}.
\end{align}

\end{theorem}
The IPW update restores the counterfactual error signal. The price of observing each action with probability at least $\pi_{\min}$ enters as $(T\pi_{\min})^{-1/2}$. The telescoping identity~\eqref{eq:pwocp-telescope} is exact and requires no clipping or projection of the recursion. The proof is in Appendix~\ref{app:thm-pwocp-proof}.

\begin{theorem}[Counterfactual coverage of DR-OCP]\label{thm:drocp}
Run Algorithm~\ref{alg:drocp} (DR signal~\eqref{eq:dr-signal}) with $\sigma(\cH_{t-1},X_t)$-measurable nuisance estimates $\hat\mu_t(a,X_t)\in[0,1]$ and $\hat\pi_t(a\mid X_t,\cH_{t-1})\in[\hat\pi_{\min},1]$, $\hat\pi_{\min}>0$, and suppose there is a deterministic constant $\kappa\geq1$ such that, almost surely,
\begin{align}\label{eq:kappa}
\sup_{t,a}\frac{\pi_t(a\mid X_t,\cH_{t-1})}{\hat\pi_t(a\mid X_t,\cH_{t-1})}\leq\kappa.
\end{align}
When propensities are estimated, assume that $\Prob(Y_t(a)\notin\cC_t(a)\mid\cF_{t-1},X_t)=\mu_t^*(a)$, where $\cF_{t-1}$ also contains past potential outcomes. This condition is unnecessary when $\hat\pi_t\equiv\pi_t$.
Define the realized nuisance RMSEs against $\mu_t^*(a)$ in~\eqref{eq:mu-star} and $\pi_t(a\mid X_t,\cH_{t-1})$,
\begin{align}\label{eq:nuisance-rmse}
    \epsilon_{\mu,a}\!:=\!\sqrt{\frac{1}{T}\sum_{t=1}^T\bigl(\hat\mu_t(a,X_t)\!-\!\mu_t^*(a)\bigr)^2}, \;
    \epsilon_{\pi,a}\!:=\!\sqrt{\frac{1}{T}\sum_{t=1}^T\bigl(\hat\pi_t(a| X_t,\!\cH_{t-1})\!-\!\pi_t(a| X_t,\!\cH_{t-1})\bigr)^2}.
\end{align}
Then, under Assumptions~\ref{ass:nuc}--\ref{ass:pos}, the unprojected recursion satisfies, with probability at least $1-\delta$,
\begin{align}\label{eq:dr-bound}
\max_a|\CCov_T(a)-(1-\alpha)|
&\leq \frac{1}{\gamma T}
+2\max_a\frac{\epsilon_{\mu,a}\epsilon_{\pi,a}}{\hat\pi_{\min}}\notag\\
&\quad+C\left\{\sqrt{\frac{\kappa\log(2K/\delta)}{T\hat\pi_{\min}}}
+\frac{\log(2K/\delta)}{T\hat\pi_{\min}}\right\}.
\end{align}
The constant $C$ is absolute. If $\hat\pi_t\equiv\pi_t$ and $\gamma=c_\gamma\sqrt{\pi_{\min}/T}$, bounds~\eqref{eq:pwocp-rate} and~\eqref{eq:pwocp-expected} also hold for DR-OCP with any predictable $\hat\mu_t\in[0,1]$.
\end{theorem}

The product term $\epsilon_{\mu,a}\epsilon_{\pi,a}/\hat\pi_{\min}$ is the standard doubly robust bias and vanishes when either nuisance is consistent. When $\hat\pi_t\equiv\pi_t$, the DR signal equals the PW-OCP signal plus a control variate that can reduce variance. Setting $\hat\mu_t\equiv\mu_t^*$ removes the first-order bias from propensity estimation. See Appendix~\ref{app:thm-drocp-proof}.

% The final term is the usual doubly robust product error. If the propensity is correct, DR-OCP reduces to PW-OCP plus a control variate. If the outcome model is correct, the first-order bias from propensity estimation disappears. This is the practical role of the augmentation in Definition~\ref{def:dr}.

\begin{remark}[Scope of the DR guarantee]\label{rem:dr-scope}
Theorem~\ref{thm:drocp} gives a joint high-probability bound expressed in terms of the realized nuisance errors $\epsilon_{\mu,a},\epsilon_{\pi,a}$, following the standard semiparametric convention \citep[e.g.,][]{lei2021conformal,barber2022conformal}. The online nuisance protocol used in the experiments is in Appendix~\ref{app:nuisance-protocol}. We do not claim a general cross-fitting rate for arbitrary adaptive logging. Establishing such rates under endogenous online data remains open.
\end{remark}

\subsection{Information-Theoretic Lower Bound}\label{sec:minimax}
The comparison class excludes persistently uninformative thresholds through an expectation condition.

\begin{theorem}[Minimax lower bound for nondegenerate calibration]\label{thm:lower}
Let $\mathcal P$ be the sequential model class of Section~2, including contextual and history-dependent logging. Its submodel $\mathcal P_G$ has independent Gaussian potential outcomes $Y_t(a)\sim\cN(0,\sigma_a^2)$ across rounds and actions, $\sigma_a\in[1,5/4]$, score $s(y;a)=|y|$, and context-free loggers $\pi_t(a\mid\cH_{t-1})\geq\pi_{\min}$ that may depend on the observed history. Fix $\alpha\in(0,1)$, let $q_\alpha$ be the $(1-\alpha)$ quantile of $|\cN(0,1)|$, and set $J_\alpha=[q_\alpha/2,2q_\alpha]$. Let $\mathfrak A_{\mathrm{nd}}$ be the predictable calibrators based on logged data satisfying
\begin{align}\label{eq:nondegenerate-class}
\inf_{(P,\pi)\in\mathcal P_G}\min_a\E_{P,\pi}L_T(a)\geq\frac23,
\qquad L_T(a):=\frac1T\sum_{t=1}^T\1\{\hat q_t(a)\in J_\alpha\}.
\end{align}
For $K\geq4$, $T\geq1$, and $\pi_{\min}\in(0,1/K]$,
\begin{align}\label{eq:lower-bound}
\inf_{\mathcal M\in\mathfrak A_{\mathrm{nd}}}\sup_{(P,\pi)\in\mathcal P}
\E\max_a|\CCov_T(a)-(1-\alpha)|
\geq c(\alpha)\min\!\left\{1,\sqrt{\frac{\log K}{T\pi_{\min}}}\right\}.
\end{align}
The same lower bound already holds with the supremum restricted to $\mathcal P_G$.
\end{theorem}

Condition~\eqref{eq:nondegenerate-class} allows infinite thresholds. Its interval and fraction define the comparison class, rather than algorithm inputs. The condition is imposed only on $\mathcal P_G$. Appendix~\ref{app:thm-lower-proof} extends the result to any fixed compact interval $[l,u]\subset(0,\infty)$ with $l<u$ and any fixed expected fraction in $(0,1)$.

\begin{corollary}[Minimax rate of PW-OCP and DR-OCP]\label{cor:minimax}
Algorithm~\ref{alg:pwocp} and DR-OCP with known propensities and any predictable $\hat\mu_t\in[0,1]$, both using $\gamma=c_\gamma\sqrt{\pi_{\min}/T}$ for fixed $c_\gamma>0$, have expected worst-action coverage error $O(\min\{1,\sqrt{\log(2K)/(T\pi_{\min})}\})$ uniformly over $\mathcal P$. Both belong to $\mathfrak A_{\mathrm{nd}}$ whenever
\begin{align}\label{eq:minimax-sample-regime}
T\pi_{\min}\geq C_{\alpha,c_\gamma}.
\end{align}
Consequently, their worst-case expected coverage rate matches~\eqref{eq:lower-bound} up to constants in this regime. No threshold truncation or change to empty-buffer predictions is needed. The coverage upper bound itself holds for all $T$.
\end{corollary}

\textbf{Sharpness and decisions.} Coverage alone is not a useful endpoint if the learner attains it by returning uninformative sets. Appendix~\ref{app:additional-theory} makes this distinction formal. Under stationary logging and local score-density regularity, Theorem~\ref{thm:cdq} shows that projected PW-OCP also controls the average threshold error relative to the oracle counterfactual quantile after an initial sampling period, at rate $(T\pi_{\min})^{-1/4}$ up to logarithmic factors. Corollary~\ref{cor:sharpness-comparison} shows that logged-quantile calibrators retain an $\Omega(\rho)$ threshold error on the impossibility instance in expectation when their tracking error is $o(\rho)$, while Theorem~\ref{thm:coverage_regret_bridge} decomposes robust pseudo-regret into sharpness, logged coverage, and exploration terms. Proposition~\ref{prop:unprojected-sharpness} extends these sharpness guarantees to unprojected PW-OCP and DR-OCP with known propensities. Proposition~\ref{prop:oracle-risk} gives an adaptive-logging bound on excess oracle robust risk and a stationary specialization using these sharpness rates.

\section{Experiments}\label{sec:experiments}

We evaluate counterfactual coverage, prediction-set quality, and all-action excess loss under selection-biased logging. PW-OCP and DR-OCP use the coverage-motivated step size $\gamma^\star\asymp\sqrt{\pi_{\min}/T}$ from Theorem~\ref{thm:pwocp}. Full setup details and pseudocode are in Appendices~\ref{app:experiments} and~\ref{app:algorithms}.

\textbf{Methods and benchmarks.} Baselines are ACI \citep{gibbs2021adaptive}, FACI \citep{gibbs2022conformal}, AgACI \citep{zaffran2022adaptive}, Conformal PID \citep{angelopoulos2023conformal}, SAOCP \citep{bhatnagar2023improved}, NEx-CP \citep{barber2022conformal}, ECI \citep{wu2025error}, and Online COPP \citep{taufiq2022conformal,zhang2023conformal} (Algorithm~\ref{alg:online-copp}). Our methods are PW-OCP and DR-OCP; Oracle calibrates directly on counterfactual outcomes and is shown only as an unattainable reference. Synthetic benchmarks cover mean shift, variance shift, multi-arm scaling, propensity misspecification, convergence rate, and prediction-set quality. We construct semi-synthetic benchmarks from ZOZOTOWN Open Bandit Women/Men data \citep{saito2020large} and evaluate financial rebalancing on DJIA returns with momentum-biased logging.

\begin{figure}[htbp]
    \hspace{1.5em}%
    \begin{subfigure}{0.32\textwidth}\centering\small{Mean shift}\end{subfigure}\hfill
    \begin{subfigure}{0.32\textwidth}\centering\small{Variance shift}\end{subfigure}\hfill
    \begin{subfigure}{0.32\textwidth}\centering\small{Multi-arm scaling}\end{subfigure}
    \\[2pt]
    \rowlabel{\qquad\quad CCov Gap}%
    \begin{subfigure}[b]{0.32\textwidth}\centering
        \includegraphics[width=\textwidth]{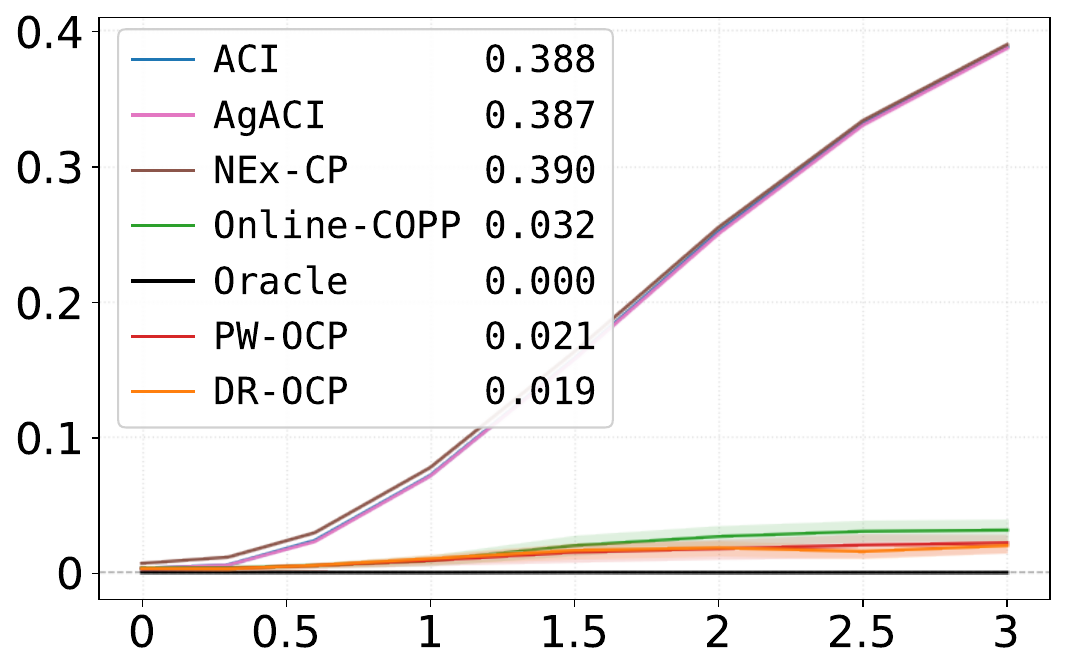}
    \end{subfigure}\hfill
    \begin{subfigure}[b]{0.32\textwidth}\centering
        \includegraphics[width=\textwidth]{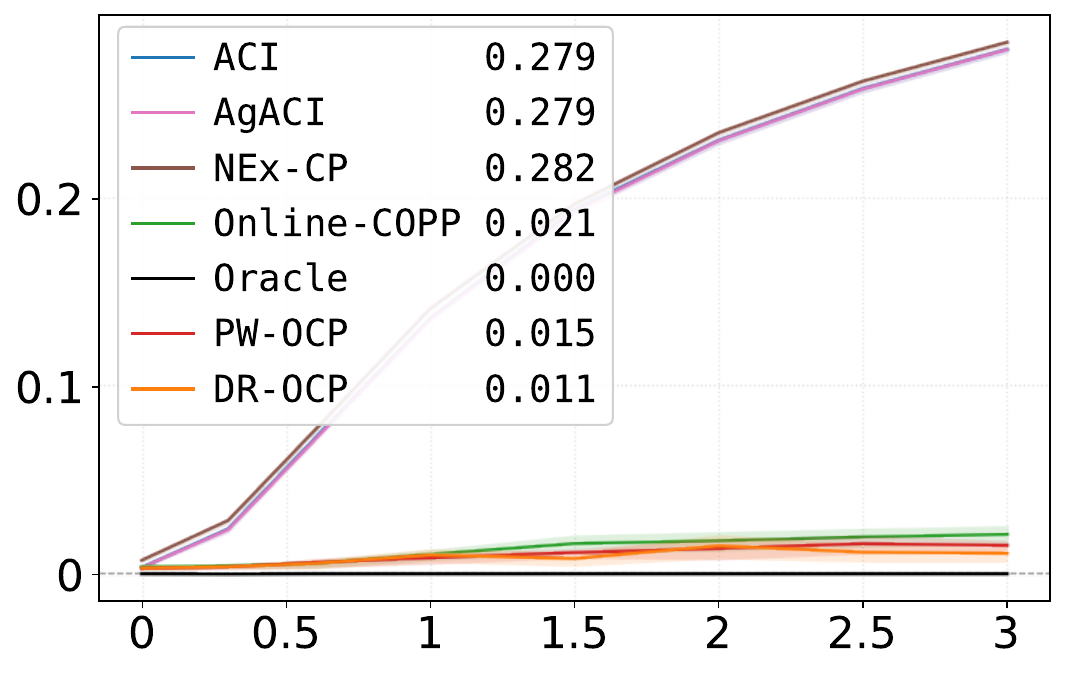}

    \end{subfigure}\hfill
    \begin{subfigure}[b]{0.32\textwidth}\centering
        \includegraphics[width=\textwidth]{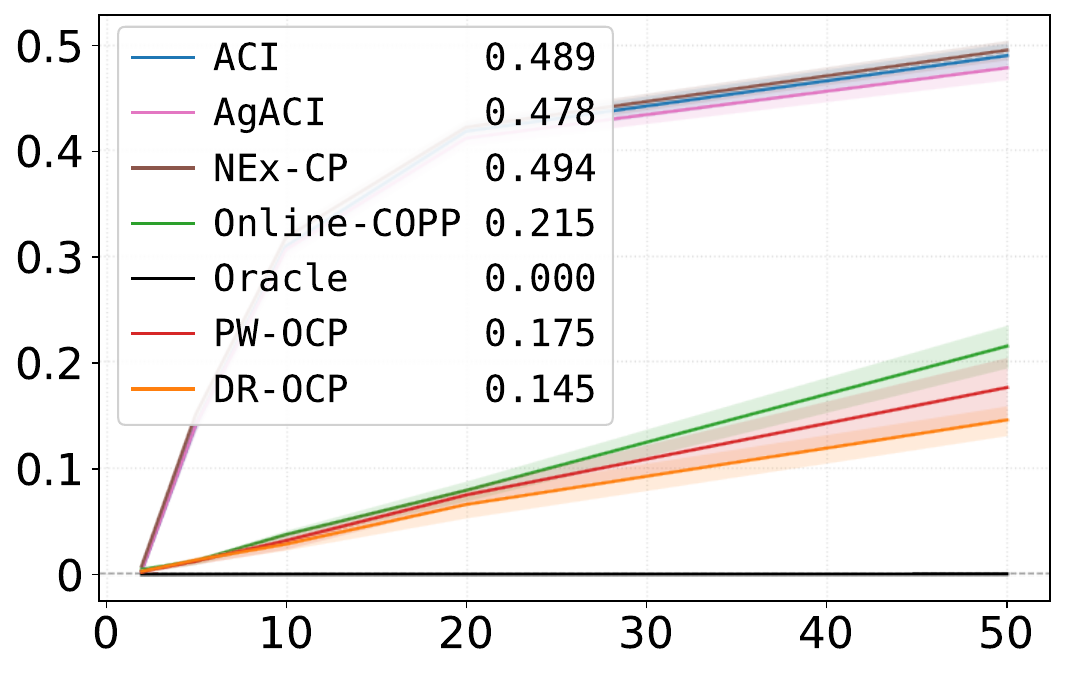}
    \end{subfigure}
    \vskip -0.05in
    \begin{subfigure}{0.32\textwidth}
        \centering
        \qquad \small{\qquad\quad Endogeneity strength}
    \end{subfigure}
    \begin{subfigure}{0.32\textwidth}
        \centering
        \setlength{\parindent}{0.5em}
        \qquad \small{\qquad Endogeneity strength}
    \end{subfigure}
    \begin{subfigure}{0.32\textwidth}
        \centering
        \setlength{\parindent}{1em}
        \quad \small{\quad Number of arms $K$}
    \end{subfigure}
    \vskip -0.05in
    \caption{\textbf{Counterfactual coverage gap across three synthetic stress tests.} Mean-shift and variance-shift simulations sweep endogeneity at $K=2$: in the first panel $\Delta$ is a mean shift, and in the second it is a standard-deviation increment. The multi-arm panel sweeps $K\in\{2,5,10,20,50\}$. Standard online CP baselines retain substantial counterfactual gaps, whereas PW-OCP and DR-OCP remain closer to Oracle across these settings. Bands are $95\%$ confidence intervals over $10$ seeds.}
\label{fig:synthetic-coverage}
\end{figure}

\begin{table}[htbp]
\centering
\caption{\textbf{Master results across all benchmarks.} Columns report counterfactual coverage gaps (lower is better) and all-action excess-loss reduction in the mean-shift experiment at $\Delta=3$, normalized so ACI is $0\%$ and Oracle is $100\%$ (higher is better). Entries report means with $95\%$ confidence intervals over $10$ seeds. \textcolor{red}{Red} and \textcolor{blue}{blue} mark the best and second-best non-oracle methods.}
\label{tab:master}
\vspace{0.3em}
\setlength{\tabcolsep}{2.0pt}
\renewcommand{\arraystretch}{1.10}
\scriptsize
\resizebox{\textwidth}{!}{
\begin{tabular}{l|ccccccc|c}
\hline \hline
\textbf{Method} & \textbf{Mean shift} & \textbf{Variance shift} & \textbf{Multi-arm} & \textbf{Propensity error} & \textbf{OBD Women} & \textbf{OBD Men} & \textbf{DJIA} & \textbf{Loss vs ACI}\\
\hline \hline
ACI \citep{gibbs2021adaptive} & 0.252 {\scriptsize$\pm$ 0.002} & 0.230 {\scriptsize$\pm$ 0.002} & 0.418 {\scriptsize$\pm$ 0.009} & 0.252 {\scriptsize$\pm$ 0.004} & 0.023 {\scriptsize$\pm$ 0.005} & \textcolor{blue}{0.025 {\scriptsize$\pm$ 0.005}} & 0.024 {\scriptsize$\pm$ 0.003} & --- \\
FACI \citep{gibbs2022conformal} & 0.275 {\scriptsize$\pm$ 0.001} & 0.250 {\scriptsize$\pm$ 0.002} & 0.503 {\scriptsize$\pm$ 0.009} & 0.275 {\scriptsize$\pm$ 0.004} & 0.041 {\scriptsize$\pm$ 0.006} & 0.043 {\scriptsize$\pm$ 0.005} & 0.055 {\scriptsize$\pm$ 0.007} & -7.0\% {\scriptsize$\pm$ 1.1\%} \\
AgACI \citep{zaffran2022adaptive} & 0.250 {\scriptsize$\pm$ 0.001} & 0.230 {\scriptsize$\pm$ 0.002} & 0.411 {\scriptsize$\pm$ 0.007} & \textcolor{blue}{0.251 {\scriptsize$\pm$ 0.004}} & 0.027 {\scriptsize$\pm$ 0.005} & 0.033 {\scriptsize$\pm$ 0.008} & 0.024 {\scriptsize$\pm$ 0.003} & +0.5\% {\scriptsize$\pm$ 1.1\%} \\
Conformal PID \citep{angelopoulos2023conformal} & 0.290 {\scriptsize$\pm$ 0.002} & 0.261 {\scriptsize$\pm$ 0.002} & 0.689 {\scriptsize$\pm$ 0.003} & 0.251 {\scriptsize$\pm$ 0.004} & 0.245 {\scriptsize$\pm$ 0.005} & 0.262 {\scriptsize$\pm$ 0.011} & 0.024 {\scriptsize$\pm$ 0.004} & -20.5\% {\scriptsize$\pm$ 1.2\%} \\
SAOCP \citep{bhatnagar2023improved} & 0.269 {\scriptsize$\pm$ 0.004} & 0.245 {\scriptsize$\pm$ 0.002} & 0.450 {\scriptsize$\pm$ 0.008} & 0.270 {\scriptsize$\pm$ 0.005} & 0.042 {\scriptsize$\pm$ 0.003} & 0.047 {\scriptsize$\pm$ 0.006} & 0.074 {\scriptsize$\pm$ 0.004} & -4.5\% {\scriptsize$\pm$ 1.4\%} \\
NEx-CP \citep{barber2022conformal} & 0.255 {\scriptsize$\pm$ 0.002} & 0.234 {\scriptsize$\pm$ 0.002} & 0.422 {\scriptsize$\pm$ 0.007} & 0.257 {\scriptsize$\pm$ 0.003} & 0.025 {\scriptsize$\pm$ 0.005} & 0.028 {\scriptsize$\pm$ 0.005} & 0.024 {\scriptsize$\pm$ 0.007} & -1.0\% {\scriptsize$\pm$ 1.1\%} \\
ECI \citep{wu2025error} & 0.369 {\scriptsize$\pm$ 0.001} & 0.327 {\scriptsize$\pm$ 0.001} & 0.607 {\scriptsize$\pm$ 0.004} & 0.369 {\scriptsize$\pm$ 0.003} & 0.122 {\scriptsize$\pm$ 0.005} & 0.121 {\scriptsize$\pm$ 0.006} & 0.027 {\scriptsize$\pm$ 0.005} & -35.0\% {\scriptsize$\pm$ 1.0\%} \\
Online COPP \citep{taufiq2022conformal,zhang2023conformal} & 0.027 {\scriptsize$\pm$ 0.008} & 0.018 {\scriptsize$\pm$ 0.005} & 0.079 {\scriptsize$\pm$ 0.010} & 0.256 {\scriptsize$\pm$ 0.005} & 0.043 {\scriptsize$\pm$ 0.010} & 0.052 {\scriptsize$\pm$ 0.011} & 0.039 {\scriptsize$\pm$ 0.009} & \textcolor{red}{+94.9\% {\scriptsize$\pm$ 0.9\%}} \\
\textbf{PW-OCP} (ours) & \textcolor{red}{0.017 {\scriptsize$\pm$ 0.007}} & \textcolor{red}{0.013 {\scriptsize$\pm$ 0.006}} & \textcolor{blue}{0.074 {\scriptsize$\pm$ 0.007}} & 0.252 {\scriptsize$\pm$ 0.002} & \textcolor{blue}{0.023 {\scriptsize$\pm$ 0.003}} & 0.028 {\scriptsize$\pm$ 0.005} & \textcolor{red}{0.019 {\scriptsize$\pm$ 0.004}} & +91.1\% {\scriptsize$\pm$ 1.8\%} \\
\textbf{DR-OCP} (ours) & \textcolor{blue}{0.018 {\scriptsize$\pm$ 0.006}} & \textcolor{blue}{0.014 {\scriptsize$\pm$ 0.006}} & \textcolor{red}{0.065 {\scriptsize$\pm$ 0.012}} & \textcolor{red}{0.030 {\scriptsize$\pm$ 0.008}} & \textcolor{red}{0.021 {\scriptsize$\pm$ 0.003}} & \textcolor{red}{0.024 {\scriptsize$\pm$ 0.003}} & \textcolor{blue}{0.019 {\scriptsize$\pm$ 0.004}} & \textcolor{blue}{+93.5\% {\scriptsize$\pm$ 1.9\%}} \\
\hline \hline
Oracle & 0.000 {\scriptsize$\pm$ 0.000} & 0.000 {\scriptsize$\pm$ 0.000} & 0.000 {\scriptsize$\pm$ 0.000} & 0.000 {\scriptsize$\pm$ 0.000} & 0.000 {\scriptsize$\pm$ 0.000} & 0.000 {\scriptsize$\pm$ 0.000} & 0.007 {\scriptsize$\pm$ 0.000} & +100.0\% \\
\hline \hline
\end{tabular}
}
\end{table}

\textbf{Counterfactual coverage.} Figure~\ref{fig:synthetic-coverage} and Table~\ref{tab:master} show that exogenous online CP baselines meet their logged targets while failing counterfactual coverage. With mean-shift magnitude $\Delta=2$ and $K=2$ actions, baseline counterfactual gaps cluster in $[0.25,0.37]$, whereas PW-OCP reduces the gap to $0.017\pm0.007$ and DR-OCP to $0.018\pm0.006$. In the variance-shift simulation, PW-OCP and DR-OCP achieve gaps near $0.01$ against $\geq 0.23$ for standard baselines. At $K=20$ arms with $\pi_{\min}=0.005$, ACI has gap $0.418\pm0.009$, compared with $0.074\pm0.007$ for PW-OCP and $0.065\pm0.012$ for DR-OCP. The larger residual gap is consistent with the $1/\sqrt{T\pi_{\min}}$ rate in Theorem~\ref{thm:pwocp}. At smaller exploration floors, Online COPP can outperform DR-OCP in finite samples, as shown in Appendix~\ref{app:cr-positivity}.

\textbf{Propensity robustness.} This experiment isolates the contribution of the doubly robust augmentation by feeding the learner increasingly biased logging probabilities. The propensity-misspecification column of Table~\ref{tab:master} reports the extreme case (uniform propensity model): PW-OCP collapses to the ACI gap, while DR-OCP retains gap $0.030\pm0.008$. Figure~\ref{fig:propensity} sweeps the misspecification level $\xi$ from $0$ (correct) to $1$ (uniform); DR-OCP remains comparatively stable, consistent with the product-form nuisance-error bound in Theorem~\ref{thm:drocp}.

AgACI does not use propensities and is unchanged by this corruption.

\textbf{Prediction-set quality, rates, and excess loss.} Figure~\ref{fig:prediction-quality} shows decreasing counterfactual coverage gaps and finite-threshold errors as $T$ grows. Table~\ref{tab:cr-width} shows that PW-OCP and DR-OCP achieve near-oracle mean widths. Appendix Figure~\ref{fig:rate-verification} independently varies $T$ and $\pi_{\min}$ and confirms the effective-sample-size dependence predicted by Theorem~\ref{thm:pwocp}. The last column of Table~\ref{tab:master} reports all-action excess-loss reduction under a shared logger, as defined in Appendix~\ref{app:experiments}. Online COPP gives the largest reduction, while PW-OCP and DR-OCP also achieve large loss reductions in this mean-shift experiment.

\begin{figure}[htbp]
    \centering
    \vspace{-0.6em}
    \hspace{1.5em}%
    \begin{subfigure}{0.29\textwidth}\centering\footnotesize{\qquad Propensity misspecification}\end{subfigure}\hfill
    \begin{subfigure}{0.29\textwidth}\centering\footnotesize{\qquad Counterfactual coverage}\end{subfigure}\hfill
    \begin{subfigure}{0.29\textwidth}\centering\footnotesize{Finite-threshold error}\end{subfigure}
    % \\[0.5pt]
    \rowlabel{\qquad\; \footnotesize{CCov Gap}}%
    \begin{subfigure}{0.29\textwidth}
        \centering
        \includegraphics[width=\textwidth]{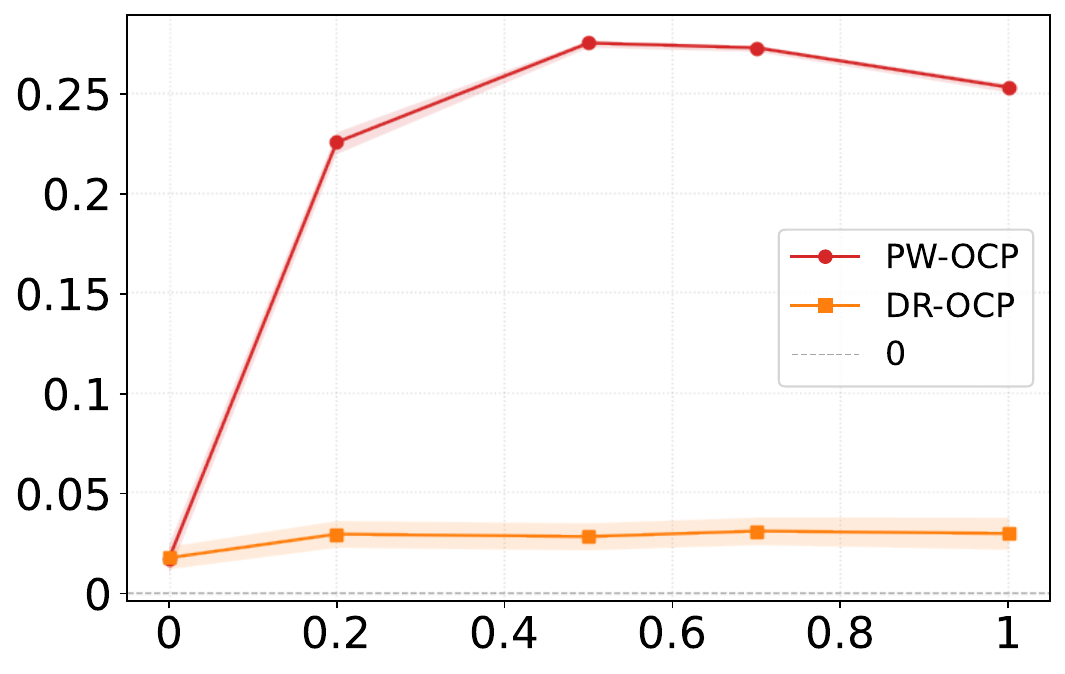}
    \end{subfigure}
    % \\[0.5pt]
    \rowlabel{\qquad\; \footnotesize{CCov Gap}}%
    \begin{subfigure}{0.29\textwidth}
        \centering
        \includegraphics[width=\textwidth]{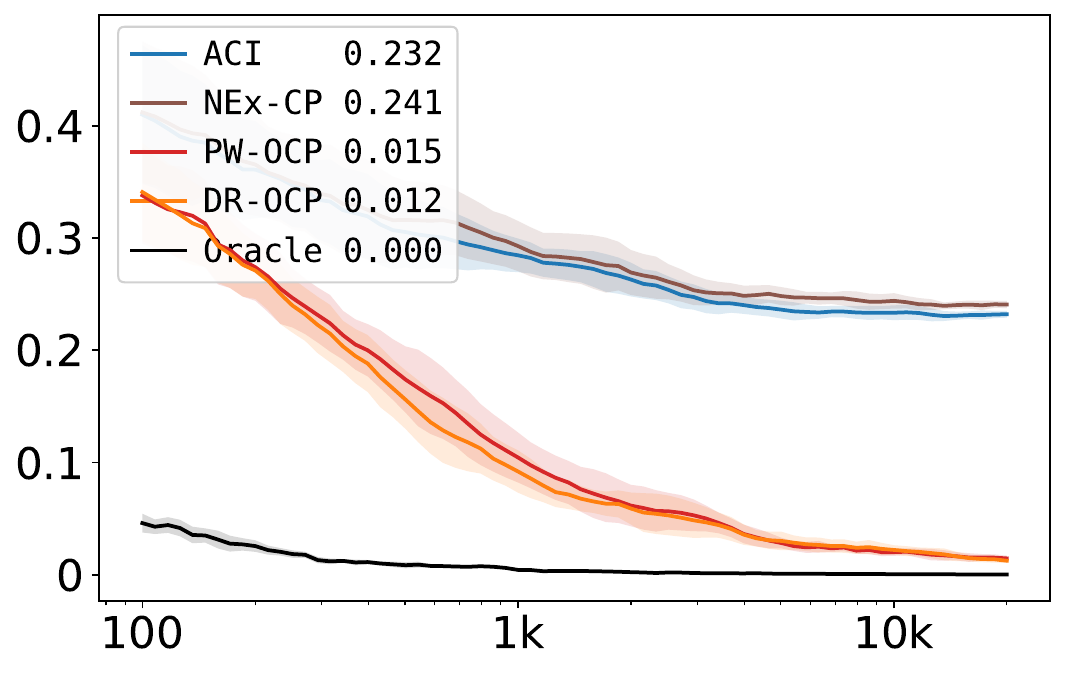}
    \end{subfigure}
    \rowlabel{\qquad \footnotesize{Avg finite $\eta_t$}}%
    \begin{subfigure}{0.29\textwidth}
        \centering
        \includegraphics[width=\textwidth]{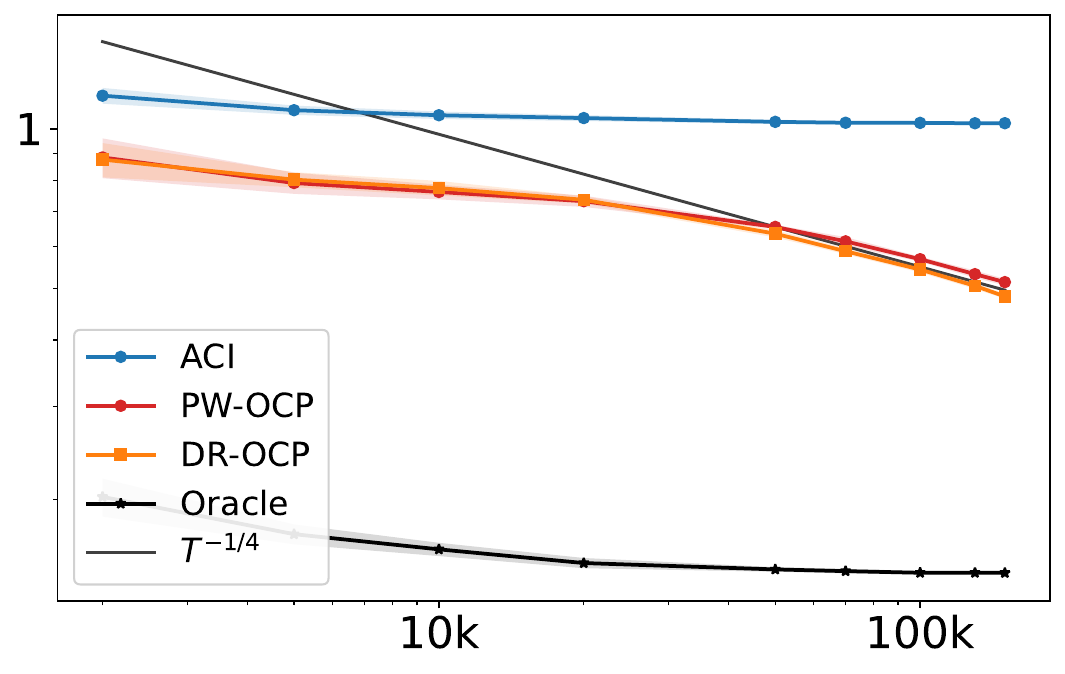}
    \end{subfigure}
    \vskip -0.05in
    \begin{subfigure}{0.29\textwidth}
        \centering
        \footnotesize{Propensity noise level $\xi$}
    \end{subfigure}
    \begin{subfigure}{0.29\textwidth}
        \centering
        \setlength{\parindent}{1em}
        \footnotesize{\qquad Number of rounds $T$}
    \end{subfigure}
    \begin{subfigure}{0.29\textwidth}
        \centering
        \setlength{\parindent}{1em}
        \footnotesize{\qquad\qquad\quad Number of rounds $T$}
    \end{subfigure}
    \vskip -0.05in
    \caption{\textbf{Robustness, rates, and prediction-set quality.} \textbf{Left}: propensity misspecification is induced by giving the learner $\hat\pi=(1-\xi)\pi+\xi/K$, where $\xi=0$ corresponds to the true propensities and $\xi=1$ to uniform propensities. PW-OCP degrades as the supplied propensities become less accurate, while DR-OCP remains comparatively stable, consistent with the product-form nuisance-error bound in Theorem~\ref{thm:drocp}. \textbf{Middle}: as the horizon $T$ grows, the counterfactual coverage gaps of PW-OCP and DR-OCP decrease toward the Oracle gap, whereas logged-calibration baselines retain persistent bias. \textbf{Right}: the average maximum oracle-quantile error over finite thresholds decreases with $T$. Rounds with no finite thresholds are excluded. The $T^{-1/4}$ curve provides a reference slope.}
\label{fig:prediction-quality}
\label{fig:propensity}
\end{figure}

PW-OCP and DR-OCP produce intervals with near-oracle widths despite logged overcoverage, which also occurs for Oracle. ACI's narrower sets instead under-cover counterfactual outcomes. When prediction sets determine a softmin logger, Appendix~\ref{app:cr-policy} shows a similar counterfactual coverage advantage for PW-OCP and DR-OCP in the heteroscedastic setting.

\begin{table}[htbp]
\centering
\caption{\textbf{Logged coverage and interval efficiency.} $K=5$, nominal coverage $0.90$, and evaluation after $1{,}000$ rounds. Entries are means over $20$ seeds, with $95\%$ confidence intervals for relative width and counterfactual gap. Width, gap, and interval-score excess use finite sets. All PW-OCP and DR-OCP prediction sets in the evaluation period are finite.}
\label{tab:cr-width}
\vspace{0.3em}
\setlength{\tabcolsep}{2.0pt}
\renewcommand{\arraystretch}{1.10}
\small
\begin{tabular}{l|cccc}
\hline \hline
\textbf{Method} & \textbf{Logged coverage} & \textbf{Width / oracle} & \textbf{CCov gap} & \textbf{Score excess} \\
\hline \hline
ACI \citep{gibbs2021adaptive} & 0.900 & 71.9\% {\scriptsize$\pm$ 0.3\%} & 0.2330 {\scriptsize$\pm$ 0.0032} & 1.77 \\
NEx-CP \citep{barber2022conformal} & 0.895 & 71.5\% {\scriptsize$\pm$ 0.3\%} & 0.2393 {\scriptsize$\pm$ 0.0032} & 1.99 \\
Online COPP \citep{taufiq2022conformal,zhang2023conformal} & 0.972 & 99.6\% {\scriptsize$\pm$ 0.4\%} & 0.0155 {\scriptsize$\pm$ 0.0021} & 0.03 \\
\textbf{PW-OCP} (ours) & 0.969 & 100.6\% {\scriptsize$\pm$ 0.4\%} & 0.0248 {\scriptsize$\pm$ 0.0033} & 0.28 \\
\textbf{DR-OCP} (ours) & 0.970 & 101.6\% {\scriptsize$\pm$ 0.4\%} & 0.0213 {\scriptsize$\pm$ 0.0027} & 0.29 \\
\hline \hline
Oracle & 0.973 & 100.0\% & 0.00033 {\scriptsize$\pm$ 0.00004} & 0.00 \\
\hline \hline
\end{tabular}
\end{table}

\textbf{Open Bandit and financial benchmarks.} On Open Bandit and DJIA, counterfactual gaps are smaller than in the synthetic stress tests, and several baselines already achieve near-nominal counterfactual coverage. PW-OCP and DR-OCP nonetheless remain among the best non-oracle methods across Women, Men, and financial rebalancing. Appendix Figure~\ref{fig:realdata} shows counterfactual coverage for these tasks. Logged $\MCov,\ACov$ diagnostics for the synthetic benchmarks are provided in Appendix Figure~\ref{fig:observable-diagnostics}. When selection bias is small, weighting can incur a variance cost. With uniform supplied propensities, the OBD Men gap rises from $0.0270$ to $0.2284$ for PW-OCP but only from $0.0255$ to $0.0372$ for DR-OCP.

\section{Related Work}\label{sec:related-work}

\textbf{Online CP under distribution shift.} Adaptive conformal inference \citep{gibbs2021adaptive,gibbs2022conformal} and related methods, including conformal PID \citep{angelopoulos2023conformal}, SAOCP \citep{bhatnagar2023improved}, AgACI \citep{zaffran2022adaptive}, ECI \citep{wu2025error}, and beyond-exchangeable CP \citep{barber2022conformal}, deliver marginal coverage under exogenous, learner-independent distribution shift by tracking the realized error stream. None of these methods detects the selection bias created when prediction sets drive the logging policy: the stream they calibrate on is itself policy-biased. Our work provides counterfactual coverage for every action, including those rarely selected, through an IPW recursion that decouples the calibration target from the logged context distribution.

\textbf{Partial-feedback CP.} Recent work on partial-feedback conformal prediction \citep{ge2025stochastic,wang2025generalized,yang2026advsemibandit} studies calibration with restricted or selectively observed outcome feedback. Our focus is the distinction between calibration on the selected context distribution and coverage under the marginal context distribution. Context-dependent logging can separate these targets even when $\rho=0$. PW-OCP and its doubly robust extension correct the selection bias through the recorded propensities.

\textbf{Counterfactual, batch, and decision-theoretic CP.} Counterfactual coverage has been established for batch treatment effects \citep{lei2021conformal} and for off-policy conformal prediction from logged batch data \citep{taufiq2022conformal,zhang2023conformal}, but these methods require a fixed offline sample and do not adapt as the logging policy evolves. Decision-theoretic conformal methods place prediction sets inside downstream optimization \citep{kiyani2025decision,lekeufack2024conformal,patel2023conformal,hu2026conformal,yeh2025conformal,zhou2025calibrating,bao2025optimal}. These works assume calibration data are exogenous and do not analyze the feedback loop in which predictions shape the data. Building on IPW \citep{horvitz1952generalization} and doubly robust estimation \citep{dudik2011doubly}, we establish sequential counterfactual coverage under adaptive logging, together with an impossibility result and a lower bound for nondegenerate calibration matching the expected coverage rate. Related work on conformal $e$-prediction \citep{vovk2025conformal,vovkwang2026confounding,wang2025stoppedebh,gauthier2025evalues,ramdas2024hypothesis}, adversarial coverage objectives \citep{ramalingam2025relationship,sale2025online}, feedback-loop CP \citep{wang2025fbcp}, and performative distribution shift \citep{perdomo2020performative} addresses other forms of robustness, but none isolates action-induced selection bias on counterfactual coverage.

\section{Discussion and Limitations}\label{sec:limitations}

This paper identifies endogenous logging as an obstacle for online conformal prediction. Marginal online validity can survive adaptive logging, yet logged-quantile per-action calibration retains a persistent $\Omega(\rho)$ counterfactual gap. PW-OCP removes this gap, with a coverage rate matching the minimax benchmark for nondegenerate calibration, and DR-OCP reduces nuisance sensitivity. The experiments show the same separation in finite samples: methods that appear calibrated on logged outcomes can remain unreliable on rarely selected actions, whereas propensity-weighted calibration reduces counterfactual gaps in the selection-biased settings. The decision decompositions characterize how prediction-set accuracy can support downstream decisions.

\textbf{Limitations and open directions.} The framework addresses contextual bandits rather than reinforcement learning: it identifies one-step potential outcomes at the realized history, not hypothetical action-sequence trajectories. Our coverage guarantees rely on positivity, and maintaining an exploration floor may be costly. Explicit sharpness and decision rates require stationary logging and local score-density regularity. The decision decompositions permit adaptive logging. The DR-OCP bound with estimated propensities depends on the realized nuisance errors. Exact pivotal normalization can invalidate the construction of Theorem~\ref{thm:impossibility} even when $\rho>0$. Whether an analogous impossibility result holds for general normalized scores remains open. Open problems include diagnostics for $\rho$, adaptive cross-fitting under endogenous streams, tighter coverage--regret lower bounds beyond Theorem~\ref{thm:pareto}, adaptive exploration budgets, delay-adapted step sizes, and anytime counterfactual CP via conformal $e$-processes \citep{vovk2025conformal,vovkwang2026confounding,wang2025stoppedebh,gauthier2025evalues,ramdas2024hypothesis}.

%%%%%%%%%%%%%%%%%%%%%%%%%%%%%%%%%%%%%%%%%%%%%%%%%%%%%%%%%%%%
\newpage
\begin{ack}
This research was supported by the National Science Foundation (NSF) under grant DMS~1820702 and the National Natural Science Foundation of China (NSFC) under grants W2441021, 72371172, 72342023, and 71929101.
\end{ack}

\bibliographystyle{plain}
\bibliography{neurips_2026}

%%%%%%%%%%%%%%%%%%%%%%%%%%%%%%%%%%%%%%%%%%%%%%%%%%%%%%%%%%%%
\newpage
\appendix
\renewcommand{\theequation}{EC.\arabic{equation}}
\renewcommand{\theHequation}{EC.\arabic{equation}}
\setcounter{equation}{0}
%%%%%%%%%%%%%%%%%%%%%%%%%%%%%%%%%%%%%%%%%%%%%%%%%%%%%%%%%%%%
\section{Broader impacts}\label{appendix:impacts}
Reliable counterfactual coverage benefits decision systems where rarely selected actions matter, such as clinical trials, dynamic pricing, and recommendation. Without validation, miscalibrated propensities or vanishing exploration produce biased prediction sets on underrepresented contexts. In deployment, practitioners should audit propensities, prefer the doubly robust variant when propensity estimates are uncertain, maintain an exploration floor, and monitor counterfactual coverage using propensity-weighted estimates or additional outcome collection.

\section{Additional Algorithms and Implementation}\label{app:algorithms}
\subsection{Doubly Robust Online Conformal Prediction}\label{app:drocp-algorithm}

DR-OCP is the operational version of Definition~\ref{def:dr}. It uses the same score buffers and action-wise threshold construction as PW-OCP, but fits nuisance models on past data and evaluates them at the current context before observing $Y_t$. After the played outcome is revealed, each action receives an imputed miscoverage estimate $\hat\mu_t(a,X_t)$ plus an inverse-propensity residual when that action was actually played. This step removes first-order sensitivity to either the propensity or outcome model, as shown in Theorem~\ref{thm:drocp}.

\begin{algorithm}[H]
\caption{DR-OCP: Doubly Robust Online Conformal Prediction}\label{alg:drocp}
% \small
\begin{algorithmic}[1]
\STATE \textbf{Input}: target $\alpha$, step size $\gamma$, predictable estimators $\hat\mu_t,\hat\pi_t$
\STATE Initialize $\alpha_1(a)\gets\alpha$ and $\cS_1(a)\gets\emptyset$ for all $a\in\cA$
\FOR{$t=1,\ldots,T$}
    \STATE \textcolor{blue}{{// Construct action-wise prediction sets}}
    \STATE Observe $X_t$ and compute $\hat\mu_t(a,X_t)$, $\hat\pi_t(a\mid X_t,\cH_{t-1})$
    \FOR{each $a\in\cA$}
        \STATE $\hat q_t(a)\gets\hat Q_t^{(a)}(1-\alpha_t(a))$
        \STATE $\cC_t(a)\gets\{y:s(X_t,y;a)\leq\hat q_t(a)\}$
    \ENDFOR
    \STATE \textcolor{blue}{{// Observe logged feedback}}
    \STATE Draw $a_t\sim\pi_t(\cdot\mid X_t,\cH_{t-1})$ and observe $Y_t$
    \STATE $E_t\gets\1\{s(X_t,Y_t;a_t)>\hat q_t(a_t)\}$
    \STATE $\cS_{t+1}(a_t)\gets\cS_t(a_t)\cup\{s(X_t,Y_t;a_t)\}$
    \STATE $\cS_{t+1}(a)\gets\cS_t(a)$ for all $a\neq a_t$
    \STATE \textcolor{blue}{{// Doubly robust calibration}}
    \FOR{each $a\in\cA$}
        \STATE $\hat p_t(a)\gets\hat\pi_t(a\mid X_t,\cH_{t-1})$
        \STATE $R_t(a)\gets E_t-\hat\mu_t(a,X_t)$
        \STATE $Z_t^{\mathrm{DR}}(a)\gets\hat\mu_t(a,X_t)+\frac{\1\{a_t=a\}}{\hat p_t(a)}R_t(a)$
        \STATE $\alpha_{t+1}(a)\gets\alpha_t(a)+\gamma(\alpha-Z_t^{\mathrm{DR}}(a))$
    \ENDFOR
\ENDFOR
\end{algorithmic}
\end{algorithm}

\subsection{Online COPP Baseline}\label{app:online-copp}
Online COPP is a natural online adaptation of batch off-policy conformal prediction \citep{taufiq2022conformal,zhang2023conformal}: an IPW empirical quantile computed from a sliding window of the $W=500$ most recent logged scores for each action. Unlike PW-OCP, it does not maintain an adaptive miscoverage index, which isolates the contribution of the online stochastic-approximation recursion in the comparison.

\begin{algorithm}[h]
\caption{Online COPP baseline}\label{alg:online-copp}
% \small
\begin{algorithmic}[1]
% \STATE \textbf{Input}: target miscoverage $\alpha$, window size $W$
\STATE \textbf{Input}: target miscoverage $\alpha$, window size $W$, logging propensities $\pi_t$
\STATE Initialize score buffers $\cS_1(a)\gets\emptyset$ and weight buffers $\cW_1(a)\gets\emptyset$
\FOR{$t=1,\ldots,T$}
    \STATE \textcolor{blue}{{// Weighted off-policy quantile}}
    \STATE Observe $X_t$
    \FOR{each $a\in\cA$}
        \STATE $\hat q_t(a)\gets$ weighted empirical $(1-\alpha)$ quantile of $\cS_t(a)$ using $\cW_t(a)$
        \STATE $\cC_t(a)\gets\{y:s(X_t,y;a)\leq\hat q_t(a)\}$
    \ENDFOR
    \STATE \textcolor{blue}{{// Observe logged feedback and update buffers}}
    \STATE Draw $a_t\sim\pi_t(\cdot\mid X_t,\cH_{t-1})$ and observe $Y_t$
    \STATE Append $s(X_t,Y_t;a_t)$ to $\cS_t(a_t)$
    \STATE Append $1/\pi_t(a_t\mid X_t,\cH_{t-1})$ to $\cW_t(a_t)$
    \STATE Keep the most recent $W$ score--weight pairs for action $a_t$
\ENDFOR
\end{algorithmic}
\end{algorithm}

\subsection{Nuisance Estimation for DR-OCP}\label{app:nuisance-protocol}\label{app:cr-propensity}
Theorem~\ref{thm:drocp} bounds coverage in terms of the realized nuisance RMSEs in~\eqref{eq:nuisance-rmse}, leaving the choice of estimators to the implementer. The experiments use the following protocol, which ensures that $\hat\mu_t$ and $\hat\pi_t$ are both measurable with respect to $\sigma(\cH_{t-1},X_t)$.

\begin{enumerate}[leftmargin=1.5em,itemsep=2pt,topsep=2pt]
    \item \textbf{Fit on past data only.} Before round $t$, fit $\hat\mu_t$ and $\hat\pi_t$ using $\cH_{t-1}$ alone; the current outcome $Y_t$ is never used.
    \item \textbf{Propensity estimator.} When the logging rule is recorded, set $\hat\pi_t=\pi_t$. For tabular contexts, estimate the probabilities from past context-action counts with add-one smoothing, then project the vector onto $\{p:p_a\geq\hat\pi_{\min},\sum_a p_a=1\}$. Update counts only after observing the current action.
    \item \textbf{Outcome model.} Regress the past miscoverage indicators $\1\{s(X_\tau,Y_\tau;a)>\hat q_\tau(a)\}$ on $(X_\tau,a)$ using only rounds with $a_\tau=a$, with pooled or hierarchical smoothing when an arm has few observations.
    \item \textbf{Block updates.} Refit the nuisance models in time blocks rather than at every round; this stabilizes the residual term in $Z_t^{\mathrm{DR}}(a)$ from~\eqref{eq:dr-signal}.
\end{enumerate}

The protocol is an implementation device. The DR coverage bound~\eqref{eq:dr-bound} remains valid for any predictable nuisance sequence through the explicit product term $\max_a \epsilon_{\mu,a}\epsilon_{\pi,a}/\hat\pi_{\min}$.

With balanced logging and $T=20{,}000$, median prediction-and-update time at $K=10$ and $50$ is $0.011/0.035$ ms for DR-OCP, $0.034/0.017$ ms for Online COPP, and $0.064/0.015$ ms for NEx-CP, over 3 seeds and 3 repetitions. The latter two benefit from shorter per-action buffers as $K$ increases.

\section{Additional Theoretical Results}\label{app:additional-theory}
This appendix presents complementary results on prediction-set sharpness (Theorem~\ref{thm:cdq} and Corollary~\ref{cor:sharpness-comparison}), the tradeoff between coverage and regret (Theorem~\ref{thm:pareto}), and the connection between prediction sets and downstream decisions (Theorem~\ref{thm:coverage_regret_bridge} and Corollary~\ref{cor:regret-pair}).

\subsection{Prediction Set Sharpness}\label{sec:efficiency}
Coverage alone does not guarantee useful prediction sets, since the trivial set $\cY$ achieves perfect coverage \citep{gneiting2007probabilistic}. We measure efficiency by the time-averaged threshold error relative to the oracle counterfactual quantile:
\begin{align}\label{eq:cdq}
    \eta_t := \max_{a\in\cA} |\hat q_t(a) - q^*(a)|, \qquad q^*(a) := (F_{\mathrm{true}}^{(a)})^{-1}(1-\alpha),
\end{align}
where $F_{\mathrm{true}}^{(a)}$ is the CDF of the score $s(X_t,Y_t(a);a)$ under the marginal context distribution. Bounding $\eta_t$ requires stationary logging and density regularity (Assumption~\ref{ass:regularity}).

\begin{assumption}[Stationary logging and quantile regularity]\label{ass:regularity}
\textbf{(a) Stationarity.} The sequence $(X_t,\{Y_t(a)\}_{a\in\cA})_{t\geq 1}$ is i.i.d., and the logging policy is time-invariant: $\pi_t(a\mid X_t,\cH_{t-1})=\pi(a\mid X_t)\geq \pi_{\min}$ almost surely.

\textbf{(b) Score densities.} For each $a\in\cA$, let $F_{\mathrm{true}}^{(a)}$ and $F_{\mathrm{bias}}^{(a)}$ denote the CDFs of $s(X_t,Y_t(a);a)$ under the marginal context distribution and conditional on $a_t=a$, respectively, with continuous densities $f_{\mathrm{true}}^{(a)}$ and $f_{\mathrm{bias}}^{(a)}$. Let $\alpha^\star(a):=1-F_{\mathrm{bias}}^{(a)}(q^*(a))$ denote the population fixed point of the PW-OCP recursion under stationary logging.

\textbf{(c) Local regularity.} For each $a$, there exist a compact interval $\mathcal I_a\subset(0,1)$ with $\alpha^\star(a)\in\mathrm{int}(\mathcal I_a)$ and a radius $\varepsilon_q>0$ such that, defining
\begin{align}\label{eq:Q-tilde}
    \mathcal Q_a:=\{(F_{\mathrm{bias}}^{(a)})^{-1}(1-\beta):\beta\in\mathcal I_a\},\qquad
    \widetilde{\mathcal Q}_a:=\{q:\mathrm{dist}(q,\mathcal Q_a)\leq\varepsilon_q\},
\end{align}
the logged-score density satisfies $0<f_{\min}\leq f_{\mathrm{bias}}^{(a)}\leq f_{\max}<\infty$ on $\widetilde{\mathcal Q}_a$, and the density ratio is bounded:
\begin{align}\label{eq:density-ratio}
    0<\lambda_{\min}\leq \frac{f_{\mathrm{true}}^{(a)}(q)}{f_{\mathrm{bias}}^{(a)}(q)}\leq \lambda_{\max}<\infty, \qquad q\in\widetilde{\mathcal Q}_a.
\end{align}
\end{assumption}

\begin{remark}[Scope of Assumption~\ref{ass:regularity}]\label{rem:stationarity-scope}
Assumption~\ref{ass:regularity} is used by the sharpness results below, including Theorem~\ref{thm:cdq} and Proposition~\ref{prop:unprojected-sharpness}, and by the decision bounds in Section~\ref{sec:decision-bridge} that invoke these rates. The basic coverage guarantees in Theorems~\ref{thm:pwocp} and~\ref{thm:drocp} allow history-dependent logging. The minimax comparison in Theorem~\ref{thm:lower} ranges over the full sequential model, with condition~\eqref{eq:nondegenerate-class} imposed only on its Gaussian submodel. For sharpness, stationarity covers fixed randomized policies and time-invariant contextual logging. Extending these sharpness results to drifting policies requires additional control.
\end{remark}

\begin{theorem}[Sharpness rate of projected PW-OCP]\label{thm:cdq}
Run PW-OCP with each index projected onto $\mathcal I_a$ after the update and $\gamma=c_\gamma\sqrt{\pi_{\min}/T}$ for any fixed $c_\gamma>0$. Under Assumptions~\ref{ass:nuc}--\ref{ass:regularity}, let $0<\delta<1/2$, $\Lambda=\log(2KT/\delta)$, and $t_0=1+\lceil C_0\Lambda/\pi_{\min}\rceil$. With probability at least $1-\delta$,
\begin{align}\label{eq:cdq-rate}
\frac1T\sum_{t=t_0}^T\eta_t
\leq C(T\pi_{\min})^{-1/4}\sqrt{\Lambda}.
\end{align}
An empty sum is zero. The constants depend on the local regularity parameters, uniform index-interval margins, and $c_\gamma$, but not on $K,T,\pi_{\min},\delta$. No score-tail assumption or change to empty-buffer predictions is required.
\end{theorem}
The startup time $t_0$ is an analysis quantity, not an algorithm input. If $t_0\leq T/2$, the average over rounds $t_0,\ldots,T$ satisfies the same rate up to a factor of two. A full-horizon raw threshold bound additionally includes $T^{-1}\sum_{t<t_0}\eta_t$, which may be infinite under an empty-buffer convention. The decision analysis below instead bounds startup loss directly. The exponent $1/4$ comes from stochastic-approximation fluctuations, while the local quantile map is Lipschitz.

\begin{theorem}[Worst-case coverage--regret lower bound]\label{thm:pareto}
Fix $\alpha\in(0,1)$. There exist constants $\delta_0(\alpha)>0$ and $c(\alpha)>0$ such that, for every target accuracy $\delta\in(0,\delta_0]$, the following holds. There exist two context-free two-arm instances $P_0,P_1$ with squared loss $\ell(a,y)=y^2$ such that, for any predictable threshold predictor using score $s(y;a)=|y|$, coupled with any decision rule that observes only the outcome of the played action and satisfies
\begin{align}\label{eq:pareto-nondegeneracy}
\min_{i\in\{0,1\}}\min_{a\in\{1,2\}}\E_i L_T(a)\geq\frac23,
\end{align}
with $L_T(a)$ defined in~\eqref{eq:nondegenerate-class} for this coupled action process, the condition
\begin{align}\label{eq:pareto-condition}
    \sup_{P\in\{P_0,P_1\}}\E_P\!\left[\max_a|\CCov_T(a)-(1-\alpha)|\right]\;\leq\; \delta
\end{align}
implies
\begin{align}\label{eq:pareto-tradeoff}
    \delta\cdot \sup_{P\in\{P_0,P_1\}}\E_P[\mathrm{Regret}_T]\;\geq\; c(\alpha),
\end{align}
where $\mathrm{Regret}_T:=\sum_{t=1}^T\bigl(\ell(a_t,Y_t)-\min_{a\in\cA}\E_P[\ell(a,Y_t(a))]\bigr)$ is the realized excess loss whose expectation is the standard pseudo-regret against the best fixed arm under $P$.
\end{theorem}

Theorem~\ref{thm:pareto} gives a worst-case regret lower bound proportional to $1/\delta$ under the same per-action nondegeneracy condition, quantifying the exploration cost of more accurate counterfactual calibration.

\subsection{From Prediction Sets to Decisions}\label{sec:decision-bridge}
The following decomposition bounds decision regret using sharpness, logged coverage, and exploration. It permits adaptive logging, while the sharpness rates below use stationary logging.

\begin{theorem}[Sharpness--coverage decomposition]\label{thm:coverage_regret_bridge}
The loss may depend on the revealed context, with this dependence suppressed in $\ell(a,y)$ below. Its Lipschitz and range bounds are uniform in the context. Let $d_t(a):=\sup_{y\in\cC_t(a)}\ell(a,y)$ be the worst-case loss inside $\cC_t(a)$ and let $\bar a_t\in\argmin_{a\in\cA} d_t(a)$ be the robust exploitation action (ties broken arbitrarily). The learner plays $a_t=\bar a_t$ with probability $1-\varepsilon$ and explores uniformly with probability $\varepsilon$; the played action is also the logged action. Assume:
\begin{enumerate}[leftmargin=1.5em,itemsep=1pt,topsep=2pt]
\item[(i)] $\ell(a,\cdot)$ is $L$-Lipschitz on $\cY$ for $L>0$.
\item[(ii)] $\cC_t(a)$ is non-empty, $d_t(a)$ and $d_t^*(a)$ are finite, and $\ell(a,y)-d_t(a)\leq B$ for all $y\notin\cC_t(a)$.
\item[(iii)] $d_t(a)-d_t(\bar a_t)\leq\Delta_{\max}$ for every action $a$.
\end{enumerate}
Let $\cC_t^*(a)$ be the oracle prediction sets at level $1-\alpha$, $d_t^*(a):=\sup_{y\in\cC_t^*(a)}\ell(a,y)$, $a_t^*\in\argmin_a d_t^*(a)$, and define the set discrepancy
\begin{align}\label{eq:eta-set}
    \eta_t^{\mathrm{set}}:=\max_{a\in\cA} d_H(\cC_t(a),\cC_t^*(a)),
\end{align}
where $d_H$ is the Hausdorff distance on $\cY$. The robust pseudo-regret
\begin{align}\label{eq:robust-pseudoregret}
    \widetilde R_T := \sum_{t=1}^T \bigl[\ell(a_t,Y_t(a_t))- d_t^*(a_t^*)\bigr]
\end{align}
satisfies
\begin{align}\label{eq:regret-bridge}
\E[\widetilde R_T]
\;\leq\;
\underbrace{L\sum_{t=1}^T \E[\eta_t^{\mathrm{set}}]}_{\text{sharpness}}
+\underbrace{\alpha B T+B T\,\E|\MCov_T-(1-\alpha)|}_{\text{logged coverage}}
+\underbrace{\varepsilon T\Delta_{\max}}_{\text{exploration}}.
\end{align}
\end{theorem}

For a loss in $[0,M]$ and nonempty sets, conditions (ii) and (iii) hold with $B=\Delta_{\max}=M$, and the sharpness term in~\eqref{eq:regret-bridge} can be replaced by $\sum_t\E\min\{M,L\eta_t^{\mathrm{set}}\}$. The capped sharpness term remains finite even when empty buffers yield infinite thresholds. A bounded loss on another interval uses the width of that interval. For unbounded losses such as squared loss under Gaussian outcomes, the theorem applies post-truncation, with the residual tail error controlled separately. The decomposition~\eqref{eq:regret-bridge} isolates three sources of decision loss. The exploration term is fixed by design and the logged-coverage term is controlled by any method satisfying Theorem~\ref{thm:marginal}. The sharpness term is where counterfactual calibration enters, since biased calibration produces persistently wide or narrow sets on rarely selected actions.

The bridge uses logged coverage for the realized action and counterfactual calibration through sharpness. For absolute-residual sets with the same centers, Hausdorff distance equals threshold distance, so
\begin{align}\label{eq:threshold-set-transfer}
c_H\max_a|\hat q_t(a)-q^*(a)|\leq\eta_t^{\mathrm{set}}
\leq C_H\max_a|\hat q_t(a)-q^*(a)|
\end{align}
holds with $c_H=C_H=1$. Other score sets require the stated comparison.

For a stationary logger, write $w_a=\Prob(a_t=a)$ and define its oracle logged miscoverage
\begin{align}\label{eq:coverage-transfer}
b^*:=\sum_a w_a\{1-F_{\mathrm{bias}}^{(a)}(q^*(a))\}.
\end{align}
This is a property of the instance, not an assumed coverage-transfer inequality. On the local quantile event in Theorem~\ref{thm:cdq}, the existing logged-density bound gives
\begin{align}\label{eq:one-sided-logged}
\sum_a w_a\{1-F_{\mathrm{bias}}^{(a)}(\hat q_t(a))\}
\leq b^*+f_{\max}\eta_t,\qquad t\geq t_0.
\end{align}
No bound on densities conditional on individual contexts is needed. On the Gaussian construction of Theorem~\ref{thm:impossibility}, $b^*\leq\alpha$ follows directly from the logger and outcome laws.

\begin{corollary}[Comparison of decision regret]\label{cor:regret-pair}
Assume Theorem~\ref{thm:coverage_regret_bridge} and~\eqref{eq:threshold-set-transfer}.

\textbf{(a) Logged-quantile calibrators.} Define the bounded set error $\bar\eta_t^{\mathrm{set}}:=\min\{1,\eta_t^{\mathrm{set}}\}$. On the instance of Theorem~\ref{thm:impossibility}, any predictable thresholds with tracking error $\varepsilon_T$ in~\eqref{eq:eps-T} satisfy
\begin{align}\label{eq:cor-a-sharpness}
\E\!\left[\frac1T\sum_{t=1}^T\bar\eta_t^{\mathrm{set}}\right]
\geq c'(\alpha)\rho-C_{\mathrm{tr}}\varepsilon_T.
\end{align}
If the decision rule additionally satisfies
\begin{align}\label{eq:decision-margin}
\E[\widetilde R_T]\geq c_{\mathrm{dec}}L\sum_t\E[\bar\eta_t^{\mathrm{set}}]-C_{\mathrm{dec}}/\gamma,
\end{align}
then $\E[\widetilde R_T]\geq c_{\mathrm{dec}}LT\{c'(\alpha)\rho-C_{\mathrm{tr}}\varepsilon_T\}-C_{\mathrm{dec}}/\gamma$. The constants are independent of $T$ and $\rho$ and depend only on the Gaussian family, the set comparison, and the stated decision condition. The cap is used only in the analysis and leaves predictions unchanged. Condition~\eqref{eq:decision-margin} must be verified for the particular decision problem and does not follow from calibration bias alone. If $\varepsilon_T=o(\rho)$ and $1/\gamma=o(\rho T)$, it yields an $\Omega(\rho T)$ regret lower bound.

\textbf{(b) PW-OCP.} Consider the bounded-loss case $\ell\in[0,M]$ of Theorem~\ref{thm:coverage_regret_bridge}, with $B=\Delta_{\max}=M$. If the actual logger also satisfies Assumption~\ref{ass:regularity}, projected PW-OCP with the step size of Theorem~\ref{thm:cdq} satisfies
\begin{align}\label{eq:pwocp-decomp}
\E[\widetilde R_T]
\leq B b^*T+\varepsilon T\Delta_{\max}
+O\!\left(T^{3/4}\pi_{\min}^{-1/4}\sqrt{\log(2KT)}\right).
\end{align}
The implicit constant depends on the loss and local regularity parameters. On the Gaussian construction, $B b^*T\leq\alpha BT$, recovering that leading term without an additional oracle assumption. The bound includes all startup rounds and permits infinite empty-buffer thresholds.
\end{corollary}

The stationary specialization in part (b) requires the robust decision rule and logger to describe the same action process. With deterministic tie-breaking and $\varepsilon<1$, this requires the exploitation action to be a fixed function of the current context almost surely. A policy whose exploitation action changes with the calibration history instead uses the adaptive decomposition in Theorem~\ref{thm:coverage_regret_bridge}, without the stationary sharpness rate.

\begin{corollary}[Sharpness comparison]\label{cor:sharpness-comparison}
Let $t_0$ be as in Theorem~\ref{thm:cdq} and $b_T=\min\{T,t_0-1\}/T$. On the instance of Theorem~\ref{thm:impossibility}, logged-quantile thresholds satisfy
\begin{align}\label{eq:cor-sharpness-lb}
\E\!\left[\frac1T\sum_{t=t_0}^T\eta_t\right]
\geq(1-b_T)c_\eta(\alpha)\rho-C_{\mathrm{tr}}\varepsilon_T.
\end{align}
Projected PW-OCP under Assumption~\ref{ass:regularity}, with any fixed step-size constant $c_\gamma>0$, satisfies with probability at least $1-\delta$
\begin{align}\label{eq:cor-sharpness-ub}
\frac1T\sum_{t=t_0}^T\eta_t
\leq C(T\pi_{\min})^{-1/4}\sqrt{\log(2KT/\delta)}.
\end{align}
\end{corollary}
When $T\pi_{\min}/\log(2KT/\delta)\to\infty$, the startup fraction vanishes. Logged-quantile tracking with $\varepsilon_T=o(\rho)$ then retains a persistent threshold error, while PW-OCP has the stated high-probability rate. The regret lower bound additionally requires~\eqref{eq:decision-margin}. The bounded-loss upper bound uses localization directly, rather than asserting an unconditional moment bound for raw threshold error.

\begin{corollary}[Full-horizon counterfactual probability error]\label{cor:probability-sharpness}
Under the assumptions and projected PW-OCP recursion of Theorem~\ref{thm:cdq}, define
\begin{align}\label{eq:probability-sharpness}
\mathcal D_T:=\frac1T\sum_{t=1}^T\max_a
\left|F_{\mathrm{true}}^{(a)}(\hat q_t(a))-(1-\alpha)\right|.
\end{align}
Then, for every $T\geq1$,
\begin{align}
\E[\mathcal D_T]\leq C(T\pi_{\min})^{-1/4}\sqrt{\log(2KT)}.
\end{align}
The constant has the same permitted dependence as in Theorem~\ref{thm:cdq}. This quantity averages absolute probability errors before any cancellation across rounds. It includes infinite-threshold predictions and averages over the current context distribution, rather than conditioning on each context value.
\end{corollary}

\begin{proposition}[Sharpness without projection]\label{prop:unprojected-sharpness}
Under Assumption~\ref{ass:regularity}, run unprojected PW-OCP or DR-OCP with known propensities and any predictable $\hat\mu_t\in[0,1]$, with $\alpha_1(a)=\alpha$ and $\gamma=c_\gamma\sqrt{\pi_{\min}/T}$. Write $N=T\pi_{\min}$ and $\Lambda=\log(2KT/\delta)$, where $0<\delta<1/2$. There are constants $A,C_0,B,C>0$ such that, if $N\geq A\Lambda^2$, then with probability at least $1-\delta$,
\begin{align}\label{eq:unproj-raw}
\frac1T\sum_{t=t_1}^T\eta_t\leq CN^{-1/4}\sqrt\Lambda,
\qquad
 t_1=1+\left\lceil\frac{C_0\Lambda}{\pi_{\min}}\right\rceil+\left\lceil\frac B\gamma\right\rceil.
\end{align}
For every $T\geq1$, the full-horizon probability error in~\eqref{eq:probability-sharpness} satisfies
\begin{align}\label{eq:unproj-probability}
\E\mathcal D_T\leq C\min\{1,N^{-1/4}\sqrt{\log(2KT)}\}.
\end{align}
Constants have the same permitted dependence as in Theorem~\ref{thm:cdq}. The time $t_1$ is used only in the analysis. The algorithm retains its original buffers, initialization, and infinite-threshold convention.
\end{proposition}

\begin{proposition}[Excess oracle robust risk]\label{prop:oracle-risk}
Let the loss be $L$-Lipschitz for $L>0$ and lie in $[0,M]$, uniformly in the context, and let the oracle sets be nonempty. Use the robust exploitation action $\bar a_t$ and uniform exploration probability $\varepsilon$ from Theorem~\ref{thm:coverage_regret_bridge}. If a learned set is empty, allow any exploitation action and set $\eta_t^{\mathrm{set}}=+\infty$. Define
\begin{align}\label{eq:oracle-risk}
R_T^*:=\sum_{t=1}^T\{d_t^*(a_t)-d_t^*(a_t^*)\}.
\end{align}
\textbf{(a) Adaptive logging.} For any logging process generated by this rule,
\begin{align}\label{eq:oracle-risk-bound}
\E R_T^*\leq\sum_{t=1}^T\E\min\{M,2L\eta_t^{\mathrm{set}}\}+\varepsilon MT.
\end{align}
\textbf{(b) Stationary specialization.} If this same process satisfies Assumption~\ref{ass:regularity}, and the upper set comparison in~\eqref{eq:threshold-set-transfer} holds whenever the thresholds lie in the local quantile neighborhoods, projected PW-OCP from Theorem~\ref{thm:cdq} and either unprojected method from Proposition~\ref{prop:unprojected-sharpness} satisfy
\begin{align}\label{eq:oracle-risk-rate}
\E R_T^*\leq\varepsilon MT+CT\min\{1,(T\pi_{\min})^{-1/4}\sqrt{\log(2KT)}\}.
\end{align}
Here $C$ depends only on the loss, set comparison, and local regularity parameters and on $c_\gamma$.
\end{proposition}
\begin{proof}
When all learned sets are nonempty, Lipschitz continuity and optimality of $\bar a_t$ give
\begin{align}
d_t^*(\bar a_t)-d_t^*(a_t^*)
&\leq |d_t^*(\bar a_t)-d_t(\bar a_t)|+|d_t(a_t^*)-d_t^*(a_t^*)|\notag\\
&\leq2L\eta_t^{\mathrm{set}}.
\end{align}
The same gap is always at most $M$. Exploration contributes at most $\varepsilon M$ per round, proving~\eqref{eq:oracle-risk-bound}, including empty-set rounds. For~\eqref{eq:oracle-risk-rate}, use the localization event of the relevant sharpness proof with $\delta_T=(2KT)^{-2}$. On this event, post-startup set errors are at most $C_H\eta_t$. Startup rounds and the event's complement contribute at most $M t_1+MT\delta_T$, with $t_0$ in place of $t_1$ for projection. These terms are absorbed by $CT(T\pi_{\min})^{-1/4}\sqrt{\log(2KT)}$. Outside the effective-sample regime, this rate expression is bounded below by a positive constant, so $R_T^*\leq MT$ suffices.
\end{proof}
This criterion measures decisions using the oracle prediction sets. It is distinct from the realized-loss pseudo-regret $\widetilde R_T$, and requires no decision-margin or coverage-transfer condition. Part (a) permits adaptive logging. Part (b) has the same stationarity restriction as Corollary~\ref{cor:regret-pair}(b), and does not give this rate for a policy that changes its exploitation action with the calibration history.

\section{Proofs for Section~\ref{sec:impossibility}}\label{app:impossibility}
% For Theorem~\ref{thm:marginal}, we use the unclipped ACI recursion \citep{gibbs2021adaptive}. Let $\mathrm{err}_t=\1\{Y_t\notin\cC_t(a_t)\}$ and update
% \begin{align}
% \alpha_{t+1}=\alpha_t+\gamma(\alpha-\mathrm{err}_t).
% \end{align}
% The empirical quantile $\hat Q_t$ uses the boundary convention $\hat Q_t(x)=+\infty$ for $x\geq 1$ and $\hat Q_t(x)=-\infty$ for $x\leq 0$. Hence the induced prediction set is $\cY$ when $\alpha_t<0$ and empty when $\alpha_t>1$. The recursion is therefore forced back toward $[0,1]$, and the standard ACI envelope lemma \citep[Lemma~4.1]{gibbs2021adaptive} gives
% \begin{align}\label{eq:gc-bound}
% \alpha_t\in[-\gamma,1+\gamma]\qquad\text{for all }t.
% \end{align}
\subsection{Proof of Theorem~\ref{thm:marginal}}\label{app:marginal_proof}
\begin{proof}
The argument is the pathwise empirical-coverage bound of adaptive conformal inference \citep{gibbs2021adaptive}, adapted to the logged sequence. The boundary convention implies that if $\alpha_t<0$, then $\hat q_t=+\infty$, $\cC_t(a_t)=\cY$, and hence $\mathrm{err}_t=0$; if $\alpha_t>1$, then $\hat q_t=-\infty$, $\cC_t(a_t)=\emptyset$, and hence $\mathrm{err}_t=1$. Since $\alpha_1\in[0,1]$ and each update has magnitude at most $\gamma$, the standard ACI envelope argument gives
\begin{align}\label{eq:gc-bound}
    \alpha_t\in[-\gamma,1+\gamma]\qquad\text{for all }t .
\end{align}

Telescoping the recursion yields
\begin{align}
    \alpha_{T+1}
    = \alpha_1 + \sum_{t=1}^T \gamma(\alpha-\mathrm{err}_t).
\end{align}
Since $\MCov_T=1-T^{-1}\sum_{t=1}^T\mathrm{err}_t$, we have
\begin{align}
    \MCov_T-(1-\alpha)
    =\frac{1}{T}\sum_{t=1}^T(\alpha-\mathrm{err}_t)
    =\frac{\alpha_{T+1}-\alpha_1}{\gamma T}.
\end{align}
Using~\eqref{eq:gc-bound},
\begin{align}
    |\alpha_{T+1}-\alpha_1|
    \leq \max\{\alpha_1,1-\alpha_1\}+\gamma ,
\end{align}
which proves the claim. The proof is pathwise and uses no exchangeability, exogeneity, stationarity, or distributional assumption.
\end{proof}

\subsection{Proof of Theorem~\ref{thm:impossibility}}\label{app:impossibility-proof}
\begin{proof}
We give the construction for a small variance shift $\Delta>0$ and convert it to the endogeneity coefficient at the end. Let $\cA=\{0,1\}$ and $\cX=\{L,H\}$. Across rounds, draw $X_t$ i.i.d. uniformly on $\cX$, draw the potential outcomes conditionally independently given $X_t$, and draw the logged action from
\begin{align}
\pi(0\mid L)=\pi(1\mid H)=\frac{3}{4},
\qquad
\pi(1\mid L)=\pi(0\mid H)=\frac{1}{4}.
\end{align}
Thus Assumption~\ref{ass:pos} holds with $\pi_{\min}=1/4$, and Assumption~\ref{ass:nuc} holds because the logging draw depends only on $X_t$. Let $r=1+\Delta$ and set
\begin{align}
    X=L:\quad &Y(0)\sim\cN(5,1),\qquad Y(1)\sim\cN(5,r^2),\\
    X=H:\quad &Y(0)\sim\cN(5,r^2),\qquad Y(1)\sim\cN(5,1).
\end{align}
Use the standard absolute-residual score $s(x,y;a)=|y-5|$.

Let $G\sim\cN(0,1)$, let $F(q)=\Prob(|G|\leq q)$, and let $f(q)=F'(q)=2\phi(q)$ for $q>0$. Write $p=1-\alpha$ and $q_0=F^{-1}(p)$. We analyze arm $0$; arm $1$ is symmetric. Conditional on $a_t=0$, context $L$ has probability $3/4$ and context $H$ has probability $1/4$. Hence the logged and counterfactual score CDFs for arm $0$ are
\begin{align}
F_{\log}(q;\Delta)=\frac{3}{4}F(q)+\frac{1}{4}F\!\left(\frac{q}{r}\right), \qquad F_{\mathrm{cf}}(q;\Delta)=\frac{1}{2}F(q)+\frac{1}{2}F\!\left(\frac{q}{r}\right).
\end{align}

Let $q_\Delta=q_0^{\log}$ be the unique solution of $F_{\log}(q_\Delta;\Delta)=p$. Since $\partial_qF_{\log}(q_0;0)=f(q_0)>0$, the implicit-function theorem gives
\begin{align}
q_\Delta=q_0+\frac{q_0}{4}\Delta+O(\Delta^2).
\end{align}
More directly, for $0<\Delta\leq1/4$, $q_0\leq q_\Delta\leq rq_0$ and
\begin{align}
p-F_{\mathrm{cf}}(q_\Delta;\Delta)
=\tfrac14\{F(q_\Delta)-F(q_\Delta/r)\}\geq c_0(\alpha)\Delta.
\end{align}
The lower bound follows from the positive minimum of $qf(q/s)/s^2$ for $q\in[q_0,5q_0/4]$ and $s\in[1,5/4]$.

It remains to pass from the population logged quantile to the time-averaged random thresholds. 
Let $F_{0,x}$ be the score CDF for arm $0$ conditional on $X_t=x$. Predictability permits dependence on the current context, so
\begin{align}
\E[\CCov_T(0)]
=\frac1T\sum_{t=1}^T\E[F_{0,X_t}(\hat q_t(0))],
\qquad
\E[F_{0,X_t}(q_\Delta)]=F_{\mathrm{cf}}(q_\Delta;\Delta).
\end{align}
Each conditional score density is at most four times the logged mixture density, since both context weights in that mixture are at least $1/4$. Integrating between $q$ and $q_\Delta$ gives
\begin{align}
|F_{0,x}(q)-F_{0,x}(q_\Delta)|
\leq4|F_{\log}(q;\Delta)-p|.
\end{align}
The inequality extends to infinite thresholds by limits. Hence, with $L=4$,
\begin{align}
\left|\E[\CCov_T(0)]-F_{\mathrm{cf}}(q_\Delta;\Delta)\right|
\leq4\E\!\left[\frac1T\sum_t|F_{\log}(\hat q_t(0);\Delta)-p|\right]
\leq L\varepsilon_T.
\end{align}
This tracking error is bounded by one, including empty-buffer predictions. On this Gaussian family, a uniform upper density bound also makes it no larger than a constant times the mean absolute quantile-tracking error.

By Jensen's inequality and the definition of the maximum over actions,
\begin{equation}
\begin{aligned}
\E\!\left[\max_{a\in\cA}|\CCov_T(a)-p|\right]
&\geq \E\!\left[|\CCov_T(0)-p|\right]\\
&\geq \left|\E[\CCov_T(0)]-p\right|\\
&\geq c_0(\alpha)\Delta-L\varepsilon_T.
\end{aligned}
\end{equation}
Finally, the endogeneity coefficient of the construction is
\begin{align}
    \rho(\Delta)=\DTV(\cN(5,1),\cN(5,r^2)).
\end{align}
A direct Gaussian calculation gives
\begin{align}
    \rho(\Delta)=2\Phi(d_\Delta)-2\Phi(d_\Delta/r), \qquad d_\Delta=r\sqrt{\frac{2\log r}{r^2-1}},
\end{align}
and hence $\rho(\Delta)=2\phi(1)\Delta+O(\Delta^2)$ as $\Delta\downarrow0$. The displayed formula is continuous and strictly increasing for $\Delta\in[0,1/4]$, and $\rho(\Delta)/\Delta$ has positive finite bounds there by continuous extension at zero. Taking $\rho_0=\rho(1/4)>0$, independently of $\alpha$, every $\rho\in(0,\rho_0]$ is therefore attained by the construction. Absorbing the proportionality constant into $c(\alpha)$ gives
\begin{align}
\E\!\left[\max_{a\in\cA}|\CCov_T(a)-(1-\alpha)|\right]
\geq c(\alpha)\rho-L\varepsilon_T.
\end{align}
The stated $O(T^{-1/2})$ form follows when $\varepsilon_T=O(T^{-1/2})$.
\end{proof}

\subsection{Proof of Lemma~\ref{lem:unbiased}}\label{app:unbiased}

\begin{proof}
Let $\cF_{t-1}^X:=\sigma(\cH_{t-1},X_t)$. By predictability, $\cC_t(a)$ and
$\pi_t(a\mid X_t,\cH_{t-1})$ are fixed conditional on $\cF_{t-1}^X$, and positivity ensures that the inverse-propensity weight is well defined. By consistency,
\begin{align}
    \1\{a_t=a\}\1\{Y_t\notin\cC_t(a)\} = \1\{a_t=a\}\1\{Y_t(a)\notin\cC_t(a)\}.
\end{align}
Therefore
\begin{align}
    \E[Z_t(a)\mid \cF_{t-1}^X]&= \E\!\left[\frac{\1\{a_t=a\}}{\pi_t(a\mid X_t,\cH_{t-1})}\1\{Y_t(a)\notin\cC_t(a)\}
\,\middle|\,\cF_{t-1}^X\right] \nonumber
    \\&=\frac{\Prob(a_t=a\mid \cF_{t-1}^X)}{\pi_t(a\mid X_t,\cH_{t-1})}\Prob\bigl(Y_t(a)\notin\cC_t(a)\mid \cF_{t-1}^X,a_t=a\bigr).
\end{align}
Assumption~\ref{ass:nuc}, together with the $\cF_{t-1}^X$-measurability of $\cC_t(a)$, gives
\begin{align}
    \Prob\bigl(Y_t(a)\notin\cC_t(a)\mid \cF_{t-1}^X,a_t=a\bigr)=\Prob\bigl(Y_t(a)\notin\cC_t(a)\mid \cF_{t-1}^X\bigr).
\end{align}
The definition of the logging policy gives
\begin{align}
    \Prob(a_t=a\mid \cF_{t-1}^X)=\pi_t(a\mid X_t,\cH_{t-1}),
\end{align}
so the propensity factor cancels, yielding the claim.
\end{proof}

\subsection{Envelope for the Weighted Recursion}
The IPW recursion lacks the deterministic envelope of standard ACI. A direct comparison with the realized counterfactual error gives a maximal martingale bound without a union bound over time.

\begin{lemma}[High-probability envelope for $\alpha_t$]\label{lem:alpha-bound}
Under Assumptions~\ref{ass:nuc}--\ref{ass:pos} and the sequential randomization convention of Section~2, Algorithm~\ref{alg:pwocp} satisfies, with probability at least $1-\delta$,
\begin{align}\label{eq:alpha-envelope}
\max_a\max_{t\leq T+1}|\alpha_t(a)-\alpha_1(a)|
\leq1+\frac{\gamma}{\pi_{\min}}
+C\gamma\left\{\sqrt{\frac{T\log(2K/\delta)}{\pi_{\min}}}
+\frac{\log(2K/\delta)}{\pi_{\min}}\right\}.
\end{align}
\end{lemma}

\begin{proof}
Let $E_t(a)=\1\{Y_t(a)\notin\cC_t(a)\}$, $I_t(a)=\1\{a_t=a\}$, and $p_t(a)=\pi_t(a\mid X_t,\cH_{t-1})$. Define
\begin{align}\label{eq:direct-randomization}
U_t(a)=Z_t(a)-E_t(a)=\left(\frac{I_t(a)}{p_t(a)}-1\right)E_t(a).
\end{align}
Use a filtration containing the full past, including past potential outcomes and previously used algorithmic randomness. Conditional also on $X_t$ and the current potential outcomes, fresh assignment gives
\begin{align}
\E[U_t(a)\mid\cF_{t-1},X_t,\{Y_t(b)\}_b]&=0,\\
\E[U_t(a)^2\mid\cF_{t-1},X_t,\{Y_t(b)\}_b]
&=(p_t(a)^{-1}-1)E_t(a)\leq\pi_{\min}^{-1}.
\end{align}
Thus $U_t(a)$ is a martingale difference by iterated conditioning, and $|U_t(a)|\leq\pi_{\min}^{-1}$. The maximal form of Freedman's inequality~\citep{freedman1975tail}, with a union bound only over signs and actions, gives
\begin{align}\label{eq:randomization-maximal}
S_T:=\max_a\max_{u\leq T}\left|\sum_{t=1}^u U_t(a)\right|
\leq C\left\{\sqrt{\frac{T\log(2K/\delta)}{\pi_{\min}}}
+\frac{\log(2K/\delta)}{\pi_{\min}}\right\}.
\end{align}

Pathwise, the update is $\alpha_{t+1}(a)=\alpha_t(a)+\gamma\{\alpha-E_t(a)\}-\gamma U_t(a)$. Outside $[0,1]$, its first drift term points inward by the quantile boundary convention. Each actual step has magnitude at most $\gamma/\pi_{\min}$. During any excursion above $1$, the initial overshoot is at most $\gamma/\pi_{\min}$ and the subsequent outward displacement is at most $2\gamma S_T$. The corresponding argument below $0$ is identical. Since $\alpha_1(a)\in[0,1]$,
\begin{align}\label{eq:pathwise-envelope}
\max_{a,t\leq T+1}|\alpha_t(a)-\alpha_1(a)|
\leq1+\gamma/\pi_{\min}+2\gamma S_T.
\end{align}
This proves the claim without an assumption on the temporal outcome law.
\end{proof}

\section{Proofs for the Theoretical Analysis}\label{app:theory}

\subsection{Proof of Theorem~\ref{thm:pwocp}}\label{app:thm-pwocp-proof}
\begin{proof}
Summing the update proves~\eqref{eq:pwocp-telescope}. Since $Z_t(a)=E_t(a)+U_t(a)$, the same identity gives
\begin{align}
\CCov_T(a)-(1-\alpha)
=\frac{\alpha_{T+1}(a)-\alpha_1(a)}{\gamma T}
+\frac1T\sum_{t=1}^T U_t(a).
\end{align}
Equation~\eqref{eq:randomization-maximal} first gives
\begin{align}\label{eq:ccov-freedman}
    \max_{a\in\cA} |\CCov_T(a) - (1-\alpha)|
    \;\leq\; \max_{a\in\cA}\frac{|\alpha_1(a)-\alpha_{T+1}(a)|}{\gamma T}
    + C\!\left\{\!\sqrt{\frac{\log(2K/\delta)}{T\pi_{\min}}}
    +\frac{\log(2K/\delta)}{T\pi_{\min}}\!\right\},
\end{align}
Combining this with~\eqref{eq:pathwise-envelope} yields the pathwise bound
\begin{align}\label{eq:pw-pathwise-coverage}
G_T:=\max_a|\CCov_T(a)-(1-\alpha)|
\leq\frac1{\gamma T}+\frac1{T\pi_{\min}}+\frac{3S_T}{T},
\end{align}
and hence~\eqref{eq:pwocp-rate} for the stated step size. Integrating the exponential tail in~\eqref{eq:randomization-maximal} gives
\begin{align}
\E S_T\leq C\left\{\sqrt{\frac{T\log(2K)}{\pi_{\min}}}+\frac{\log(2K)}{\pi_{\min}}\right\}.
\end{align}
If $T\pi_{\min}\geq\log(2K)$, the linear term is absorbed by the square-root term after division by $T$. Otherwise use $G_T\leq1$. This proves
\begin{align}\label{eq:pwocp-expected}
\E\max_a|\CCov_T(a)-(1-\alpha)|
\leq C_{c_\gamma}\min\!\left\{1,\sqrt{\frac{\log(2K)}{T\pi_{\min}}}\right\}.
\end{align}
These are finite-horizon maximal bounds for the chosen $T$, rather than a simultaneous guarantee over all horizons.
\end{proof}

\subsection{Proof of Theorem~\ref{thm:drocp}}\label{app:thm-drocp-proof}
\begin{proof}
Use the full-history filtration from Lemma~\ref{lem:alpha-bound}. Abbreviate the estimated propensity and outcome mean by $\hat p_t(a)$ and $\hat\mu_t(a)$. The direct error difference is
\begin{align}
D_t(a):=Z_t^{\mathrm{DR}}(a)-E_t(a)
=\left(\frac{I_t(a)}{\hat p_t(a)}-1\right)\{E_t(a)-\hat\mu_t(a)\}.
\end{align}
For estimated propensities, the stated conditional-mean assumption implies
\begin{align}\label{eq:dr-bias}
b_t(a):=\E[D_t(a)\mid\cF_{t-1},X_t]
=\{\hat\mu_t(a)-\mu_t^*(a)\}\left(1-\frac{p_t(a)}{\hat p_t(a)}\right).
\end{align}
Then $M_t(a)=D_t(a)-b_t(a)$ is a martingale difference. Fresh randomization and $|E_t(a)-\hat\mu_t(a)|\leq1$ give
\begin{align}
\E[D_t(a)^2\mid\cF_{t-1},X_t]
&\leq \frac{p_t(a)}{\hat p_t(a)^2}-\frac{2p_t(a)}{\hat p_t(a)}+1
\leq\frac{2\kappa}{\hat\pi_{\min}}.
\end{align}
Consequently $|M_t(a)|\leq2/\hat\pi_{\min}$ and its conditional variance is at most $2\kappa/\hat\pi_{\min}$. Maximal Freedman gives, with probability at least $1-\delta$,
\begin{align}\label{eq:dr-martingale-bound}
S_T^{\mathrm{DR}}:=\max_a\max_{u\leq T}\left|\sum_{t=1}^u M_t(a)\right|
\leq C\left\{\sqrt{\frac{\kappa T\log(2K/\delta)}{\hat\pi_{\min}}}
+\frac{\log(2K/\delta)}{\hat\pi_{\min}}\right\}.
\end{align}
The exact update is
\begin{align}
\alpha_{t+1}(a)=\alpha_t(a)+\gamma\{\alpha-E_t(a)\}
-\gamma b_t(a)-\gamma M_t(a).
\end{align}
Its actual one-step magnitude is at most $\gamma/\hat\pi_{\min}$. Applying the excursion argument of Lemma~\ref{lem:alpha-bound}, now allowing the total bias displacement, gives
\begin{align}
\max_{t\leq T+1}|\alpha_t(a)-\alpha_1(a)|
\leq1+\frac{\gamma}{\hat\pi_{\min}}
+2\gamma S_T^{\mathrm{DR}}+\gamma\sum_{t=1}^T|b_t(a)|.
\end{align}
Telescoping against $E_t(a)$ therefore yields
\begin{align}
G_T^{\mathrm{DR}}
\leq\frac1{\gamma T}+\frac1{T\hat\pi_{\min}}
+\frac{3S_T^{\mathrm{DR}}}{T}
+2\max_a\frac1T\sum_{t=1}^T|b_t(a)|.
\end{align}
Cauchy--Schwarz gives
\begin{align}\label{eq:dr-product}
\frac1T\sum_{t=1}^T|b_t(a)|
\leq\frac{\epsilon_{\mu,a}\epsilon_{\pi,a}}{\hat\pi_{\min}}.
\end{align}
This proves~\eqref{eq:dr-bound}, with $1/(T\hat\pi_{\min})$ absorbed into its logarithmic term.

If $\hat p_t=p_t$, conditioning on the current potential outcomes already gives $\E[D_t(a)\mid\cF_{t-1},X_t,\{Y_t(b)\}_b]=0$. Thus $b_t=0$ without any outcome-kernel restriction, and the same maximal bound and tail integration prove~\eqref{eq:pwocp-rate} and~\eqref{eq:pwocp-expected} for DR-OCP.

For a projected update define $r_t(a)=\alpha_{t+1}(a)-\{\alpha_t(a)+\gamma(\alpha-Z_t^{\mathrm{DR}}(a))\}$. Its separate exact identity is
\begin{align}
\frac1T\sum_t Z_t^{\mathrm{DR}}(a)-\alpha
=\frac{\alpha_1(a)-\alpha_{T+1}(a)+\sum_t r_t(a)}{\gamma T}.
\end{align}
It implies a coverage bound with the absolute value of this endpoint-plus-residual term, plus $S_T^{\mathrm{DR}}/T$ and the product bias in~\eqref{eq:dr-product}. A bounded projection interval alone does not control the accumulated residual or justify the unprojected envelope.
\end{proof}

\paragraph{Variance of the augmentation.}
For the unprojected recursion with known propensities, let $\bar\mu_t(a)=\E[E_t(a)\mid\cF_{t-1},X_t]$. Fresh randomization gives the exact identity
\begin{align}\label{eq:dr-conditional-variance}
\E[D_t(a)^2\mid\cF_{t-1},X_t]
=(p_t(a)^{-1}-1)\{\bar\mu_t(a)(1-\bar\mu_t(a))+(\hat\mu_t(a)-\bar\mu_t(a))^2\}.
\end{align}
At fixed thresholds, the corresponding PW-OCP error difference $Z_t(a)-E_t(a)$ has conditional second moment $(p_t(a)^{-1}-1)\bar\mu_t(a)$. Thus the DR error difference has no larger conditional variance whenever $(\hat\mu_t(a)-\bar\mu_t(a))^2\leq\bar\mu_t(a)^2$. The oracle choice $\hat\mu_t=\bar\mu_t$ minimizes this variance when $\bar\mu_t$ is measurable from the observed history and current context. Validity with known propensities does not require this accuracy condition.

Put $V_T(a)=\sum_{t=1}^T\E[D_t(a)^2\mid\cF_{t-1}]$ and $x=\log(2K/\delta)$. For any deterministic $v\geq0$, maximal Freedman and the same excursion argument give
\begin{align}\label{eq:dr-variance-bound}
\Prob\!\left(G_T^{\mathrm{DR}}>\frac1{\gamma T}+\frac1{T\pi_{\min}}
+\frac C T\left\{\sqrt{vx}+\frac{x}{\pi_{\min}}\right\},\ 
\max_a V_T(a)\leq v\right)\leq\delta.
\end{align}
Indeed, $D_t$ is a martingale difference with $|D_t|\leq1/\pi_{\min}$, and the pathwise envelope gives $G_T^{\mathrm{DR}}\leq1/(\gamma T)+1/(T\pi_{\min})+3\max_{a,u\leq T}|\sum_{t\leq u}D_t(a)|/T$. Apply the variance-budget form of Freedman's inequality and union over actions and signs. This is a joint-event bound with a fixed budget $v$, rather than conditioning on a future variance event. It quantifies a possible reduction in the stochastic term without changing the recursion or its worst-case rate.

\subsection{Proof of Theorem~\ref{thm:lower} and Corollary~\ref{cor:minimax}}\label{app:thm-lower-proof}
\begin{proof}[Proof of Theorem~\ref{thm:lower}]
Write $p=1-\alpha$, $F(q)=\Prob(|\cN(0,1)|\leq q)$, $J=J_\alpha$, and $N=T\pi_{\min}$. It suffices to work inside $\mathcal P_G$ with the fixed logger $\pi(a)=\pi_{\min}$ for $a<K$ and $\pi(K)=1-(K-1)\pi_{\min}$. Let $Q_v$ denote the complete logged-data law, including an independent algorithmic seed, when arm $v<K$ has scale $1+\Delta$ and all other arms have scale $1$. Let $Q_0$ be the law when all scales equal $1$. Here $0<\Delta\leq1/4$.

\textbf{Separation in expectation.} Positivity of the half-normal density on compact subsets of $(0,\infty)$ gives
\begin{align}\label{eq:fano-gap}
F(q)-F(q/(1+\Delta))\geq d_\alpha\Delta
\quad\text{for every }q\in J,
\end{align}
for a constant $d_\alpha>0$. This follows by integrating $qF'(q/s)/s^2$ over $s\in[1,1+\Delta]$. The difference is nonnegative outside $J$, including infinite thresholds.
Define $A_a=T^{-1}\sum_t F_a(\hat q_t(a))$ and $M_a=T^{-1}\sum_t F(\hat q_t(a))$. Under $Q_v$, $A_a=M_a$ for $a\ne v$ and $M_v-A_v\geq d_\alpha\Delta L_T(v)$. Since $0\leq L_T(v)\leq1$ and $\E_v L_T(v)\geq2/3$,
\begin{align}
\Prob_v(L_T(v)\geq1/3)\geq1/2.
\end{align}
On this event, $\max_a|A_a-p|<d_\alpha\Delta/12$ makes $\arg\max_{a<K}M_a$ identify $v$. Therefore
\begin{align}
\Prob_v(\text{testing error})\leq\frac12+
\frac{12}{d_\alpha\Delta}\E_v\max_a|A_a-p|.
\end{align}
The Gaussian KL divergence is at most $C\Delta^2$, so the logged-data KL chain rule gives $\mathrm{KL}(Q_v\|Q_0)\leq CN\Delta^2$. For $K\geq257$, choose
\begin{align}
\Delta=c_\alpha\min\!\left\{1,\sqrt{\frac{\log K}{N}}\right\}
\end{align}
small enough that the mutual information under a uniform choice of $v<K$ is at most $\log(K-1)/8$. Since $\log2/\log(K-1)\leq1/8$, Fano's inequality gives testing error at least $3/4$ and hence predictable coverage risk at least a constant times $\Delta$.

For $4\leq K<257$, use $Q_0,Q_1$. On the common part $Q_0\wedge Q_1$ of their densities, the triangle inequality gives
\begin{align}
\E_0|M_1-p|+\E_1|A_1-p|\geq d_\alpha\Delta\int L_T(1)\,d(Q_0\wedge Q_1) \geq d_\alpha\Delta\{2/3-\mathrm{TV}(Q_0,Q_1)\}.
\end{align}
Choose $\Delta=c_\alpha\min\{1,N^{-1/2}\}$ so that Pinsker's inequality makes the total variation at most $1/3$. The bounded range of $\log K$ then gives the same rate, after adjusting constants. In both cases,
\begin{align}\label{eq:lower-predictable}
\sup_{(P,\pi)\in\mathcal P_G}\E\max_a|A_a-p|\geq c_1r,
\qquad r=\min\!\left\{1,\sqrt{\frac{\log K}{N}}\right\}.
\end{align}

\textbf{Passage to realized coverage.} Write $s_T=\min\{1,\sqrt{\log K/T}\}$. Fresh independent potential outcomes make $\1\{Y_t(a)\in\cC_t(a)\}-F_a(\hat q_t(a))$ a bounded martingale difference. Integrating Azuma's inequality over actions gives $\E\max_a|\CCov_T(a)-A_a|\leq C_1s_T$. Thus the worst-case realized risk $R$ satisfies $R\geq c_1r-C_1s_T$.

A second bound controls the hidden-outcome fluctuation. Under $Q_0$, let
\begin{align}
H_a=\sum_{t=1}^T\1\{a_t\ne a,\ \hat q_t(a)\in J\},\qquad a<K.
\end{align}
Fresh fixed-probability assignment gives $\E H_a=(1-\pi_{\min})T\E L_T(a)\geq T/2$. Since $H_a\leq T$, $\Prob(H_a\geq T/4)\geq1/3$. Consequently, if $B=\sum_{a<K}\1\{H_a\geq T/4\}$, then $\E B\geq(K-1)/3$ and
\begin{align}
\Prob(B\geq(K-1)/6)\geq1/5.
\end{align}
On this event there are at least $M=\max\{1,\lfloor(K-1)/6\rfloor\}$ arms with at least $m=\lceil T/4\rceil$ hidden outcomes at informative thresholds. Conditional on the complete log and the algorithmic seed, these outcomes are independent, and their coverage probabilities lie in $[b,1-b]$ for a constant $b=b(\alpha)>0$.

For mutually independent Bernoulli variables $B_{ji}$ with success probabilities $p_{ji}\in[b,1-b]$, let $S_j=\sum_{i=1}^m B_{ji}$. For arbitrary offsets $z_j$,
\begin{align}\label{eq:bernoulli-max-lower}
\E\max_{j\leq M}|S_j-z_j|
\geq c_b\min\{m,\sqrt{m\log(M+1)}\}.
\end{align}
To prove this, introduce an independent copy $S'$ of $S$ and write $f(x)=\max_j|x_j|$. Each summand of $S_j-S_j'$ equals either sign with probability $p_{ji}(1-p_{ji})\geq v:=b(1-b)$. For a convex function, increasing the common probability of the two signs increases its expectation, since $f(x+e_j)+f(x-e_j)\geq2f(x)$ for each coordinate vector $e_j$. Applying this fact one summand at a time, together with the triangle inequality, gives
$\E f(S-z)\geq\tfrac12\E f(S-S')\geq\tfrac12\E f(W),$
where the $W_j$ are independent sums of $m$ i.i.d. variables taking $1,-1,0$ with probabilities $v,v,1-2v$.

For a sum $V_n$ of $n\geq1$ such variables, $\E V_n^2=2vn$ and $\E V_n^4\leq2vn+12v^2n^2$. Symmetry and the Paley--Zygmund inequality imply

$\Prob(V_n\geq\sqrt{vn})\geq\frac{(vn)^2}{2(2vn+12v^2n^2)}\geq\beta_v:=\frac{v}{4+24v}>0.$
For any $1\leq k\leq m$, partition each sum $W_j$ into $k$ independent blocks of lengths at least $\lfloor m/k\rfloor\geq m/(2k)$. With probability at least $\beta_v^k$, every block of length $n$ has sum at least $\sqrt{vn}$. Hence
$\Prob\!\left(W_j\geq\sqrt{vmk/2}\right)\geq\beta_v^k.$
Put $C_v=\log(1/\beta_v)$ and choose
$k=\min\!\left\{m,\max\!\left\{1,\left\lfloor\frac{\log(M+1)}{2C_v}\right\rfloor\right\}\right\}.$
Then $k\geq(4C_v)^{-1}\min\{m,\log(M+1)\}$. If $k=1$, a single sum reaches the displayed threshold with probability at least $\beta_v$. If $k\geq2$, then $\beta_v^k\geq(M+1)^{-1/2}$, so independence gives
$\Prob\!\left(\max_{j\leq M}W_j\geq\sqrt{vmk/2}\right)\geq1-(1-\beta_v^k)^M\geq1-e^{-1/\sqrt2}.$
Combining these bounds with symmetrization proves~\eqref{eq:bernoulli-max-lower} for all $m,M\geq1$, including the saturated regime $\log(M+1)\geq m$.

Condition further on the remaining hidden outcomes, which together with the revealed outcomes contribute only fixed offsets, and apply~\eqref{eq:bernoulli-max-lower} to the selected sums. Since $\log(M+1)$ is bounded below by an absolute positive multiple of $\log K$ for $K\geq4$, averaging over the event above gives $R\geq c_2s_T$. Finally,
\begin{align}
R\geq\max\{c_2s_T,c_1r-C_1s_T\}
\geq\frac{c_1c_2}{C_1+c_2}\,r.
\end{align}
All hard instances lie in $\mathcal P_G\subset\mathcal P$, proving both assertions.
\end{proof}

\begin{lemma}[Nondegeneracy under adaptive logging]\label{lem:threshold-occupation}
On $\mathcal P_G$, run Algorithm~\ref{alg:pwocp} or DR-OCP with known propensities and any predictable $\hat\mu_t\in[0,1]$, using $\gamma=c_\gamma\sqrt{\pi_{\min}/T}$ for fixed $c_\gamma>0$. Put $\nu_T(a)=1-L_T(a)$. For each action $a$ and $0<\delta<1/2$,
\begin{align}\label{eq:threshold-occupation}
T\pi_{\min}\geq C_{\alpha,c_\gamma}\log(2/\delta)
\quad\Longrightarrow\quad
\Prob(\nu_T(a)>1/4)\leq\delta.
\end{align}
The constant is uniform over the Gaussian instance class and history-dependent loggers. The algorithm retains its original quantile rule and empty-buffer predictions.
\end{lemma}

\begin{proof}
Fix an action and suppress its index. Write $p_0=\pi_{\min}$, $N=Tp_0$, $L_\delta=\log(6/\delta)$, and $F_a(q)=\Prob(|\cN(0,\sigma_a^2)|\leq q)$. Define
\begin{align}
e_t&=\alpha_t-\alpha,\qquad u_t=\Pi_{[0,1]}(\alpha_t)-\alpha,\\
\zeta_t&=F_a(\hat q_t)-\{1-\Pi_{[0,1]}(\alpha_t)\}.
\end{align}
Here $\Pi_{[0,1]}$ describes the quantile boundary convention in the analysis, not a projection of the algorithm. For a nonempty buffer of size $n$, $|\zeta_t|$ is at most its uniform empirical-CDF error $D_n$ plus an empirical-rank correction $1/n$. For an empty buffer, $|\zeta_t|\leq1$.

\textbf{A fixed-accuracy empirical-CDF event.} The logger selects an action before seeing its current outcome. Under the context-free Gaussian model, successive observed scores of this action form an i.i.d. stream from $F_a$, even when the selection probabilities depend on past observed scores. We use a uniform event on this stream, without conditioning on its random sample count. Fix a small $\epsilon>0$, to be chosen using only $\alpha$. By DKW and a union bound over all $n\geq n_0$,
\begin{align}
\Prob(\exists n\geq n_0:D_n>\epsilon)
\leq\frac{2e^{-2n_0\epsilon^2}}{1-e^{-2\epsilon^2}}\leq\delta/3
\end{align}
for $n_0\geq C_\epsilon L_\delta$, chosen also so that $1/n_0\leq\epsilon$.
Let $N_m$ count observations of this action through round $m$. For $\lambda>0$, iterated conditioning gives
\begin{align}
\E[e^{-\lambda N_m}]\leq(1-p_0+p_0e^{-\lambda})^m,
\qquad
\Prob(N_m<mp_0/2)\leq e^{-mp_0/8}.
\end{align}
Choose $m=\lceil C_\epsilon L_\delta/p_0\rceil$ large enough that $mp_0/2\geq n_0$ and the last probability is at most $\delta/3$. If $m\leq T$, monotonicity of the count then gives
\begin{align}\label{eq:occupation-cdf}
\frac1T\sum_{t=1}^T\zeta_t^2
\leq\frac mT+4\epsilon^2
\leq\frac{C_\epsilon L_\delta}{N}+4\epsilon^2
\end{align}
with probability at least $1-2\delta/3$. The sample-size condition below ensures $m\leq T$.

\textbf{Controlling the index without restricting its range.} Use the nonnegative, $1$-Lipschitz function
\begin{align}
V(e)=
\begin{cases}
e^2/2,& |e|\leq1,\\
|e|-1/2,& |e|>1,
\end{cases}
\qquad g(e)=V'(e)=\max\{-1,\min\{e,1\}\}.
\end{align}
Its derivative is $1$-Lipschitz. Put $a_0=\min\{\alpha,1-\alpha\}>0$. For every real $e$, $g(e)\{\Pi_{[0,1]}(\alpha+e)-\alpha\}\geq a_0g(e)^2$.
With $W_t=\alpha-Z_t(a)$ for PW-OCP, or $W_t=\alpha-Z_t^{\mathrm{DR}}(a)$ for DR-OCP with known propensities, unbiasedness gives
\begin{align}
\E[W_t\mid\cH_{t-1}]&=-u_t+\zeta_t,\\
|W_t|&\leq1/p_0,\qquad
\E[W_t^2\mid\cH_{t-1}]\leq C/p_0.
\end{align}
Smoothness and Young's inequality yield
\begin{align}\label{eq:occupation-drift}
\E[V(e_{t+1})-V(e_t)\mid\cH_{t-1}]
\leq-\frac{a_0\gamma}{2}g(e_t)^2
+\frac{\gamma}{2a_0}\zeta_t^2+\frac{C\gamma^2}{p_0}.
\end{align}
The centered increments of $V(e_t)$ have magnitude at most $2\gamma/p_0$ and conditional variance at most $C\gamma^2/p_0$. Freedman's inequality controls their sum by $C\gamma\{\sqrt{TL_\delta/p_0}+L_\delta/p_0\}$ with failure probability at most $\delta/3$. Summing~\eqref{eq:occupation-drift}, using $e_1=0$ and $V\geq0$, gives
\begin{align}\label{eq:occupation-index}
\frac1T\sum_{t=1}^Tg(e_t)^2
\leq C_\alpha\left\{
\frac1T\sum_{t=1}^T\zeta_t^2+
\frac{\gamma}{p_0}+\sqrt{\frac{L_\delta}{N}}
+\frac{L_\delta}{N}\right\}.
\end{align}

Each oracle quantile $\sigma_aq_\alpha$ lies strictly inside $J_\alpha$. Compactness of the scale range gives
\begin{align}
h_\alpha:=\inf_{\sigma\in[1,5/4]}
\min\{p-F_\sigma(q_\alpha/2),\ F_\sigma(2q_\alpha)-p\}>0,
\qquad p=1-\alpha.
\end{align}
If $\hat q_t\notin J_\alpha$, then $|F_a(\hat q_t)-p|\geq h_\alpha$. Since $F_a(\hat q_t)-p=-u_t+\zeta_t$ and $u_t^2\leq g(e_t)^2$,
\begin{align}
\nu_T(a)\leq\frac{2}{h_\alpha^2T}\sum_{t=1}^T\{g(e_t)^2+\zeta_t^2\}
\leq C_\alpha\left\{4\epsilon^2+
\frac{C_\epsilon L_\delta}{N}+
\frac{c_\gamma}{\sqrt N}+\sqrt{\frac{L_\delta}{N}}\right\}
\end{align}
on an event of probability at least $1-\delta$. Choose $\epsilon$ so that the first term is at most $1/8$, and then $N\geq C_{\alpha,c_\gamma}L_\delta$ so that the remaining terms sum to at most $1/8$. This also ensures $m\leq T$. Absorbing the fixed ratio between $L_\delta$ and $\log(2/\delta)$ proves~\eqref{eq:threshold-occupation}. These choices are proof constants, not algorithm parameters.
\end{proof}

\begin{proof}[Proof of Corollary~\ref{cor:minimax}]
Theorems~\ref{thm:pwocp} and~\ref{thm:drocp} give~\eqref{eq:pwocp-expected} uniformly over $\mathcal P$, without threshold restrictions. On $\mathcal P_G$, Lemma~\ref{lem:threshold-occupation} with $\delta=0.05$ gives, separately for every action,
\begin{align}
\E L_T(a)\geq(1-0.05)\cdot\frac34=\frac{57}{80}>\frac23
\end{align}
whenever~\eqref{eq:minimax-sample-regime} holds. Thus both algorithms belong to $\mathfrak A_{\mathrm{nd}}$, uniformly over adaptive context-free loggers on the regular submodel. Combining their full-model upper bounds with Theorem~\ref{thm:lower} proves the minimax rate up to constants. This is a worst-case comparison over $\mathcal P$, not an assertion of instancewise optimality on every outcome law.
\end{proof}

\paragraph{Other informative intervals and fractions.}
Fix $J=[l,u]$ with $0<l<u<\infty$ and $0<\vartheta<1$. Replace~\eqref{eq:nondegenerate-class} by
\begin{align}\label{eq:general-nondegeneracy}
\inf_{(P,\pi)\in\mathcal P_G}\min_a\E_{P,\pi}
\left[\frac1T\sum_t\1\{\hat q_t(a)\in J\}\right]\geq\vartheta.
\end{align}
Theorem~\ref{thm:lower} retains the rate in~\eqref{eq:lower-bound}, with constant $c(\alpha,J,\vartheta)>0$. Thus neither the interval $J_\alpha$ nor the fraction $2/3$ is essential to this lower bound.

To verify the extension, set $L_T^J(a)=T^{-1}\sum_t\1\{\hat q_t(a)\in J\}$ and
$d_J=\inf_{q\in J,\,s\in[1,5/4]}qF'(q/s)/s^2>0$.
Then the scale separation is at least $d_J\Delta L_T^J(a)$, and
$\Prob(L_T^J(a)\geq\vartheta/2)\geq\vartheta/(2-\vartheta)\geq\vartheta/2$.
Writing $R_{\mathrm{pred}}=\sup\E\max_a|A_a-p|$, the same decoder has error at most
$1-\vartheta/2+8R_{\mathrm{pred}}/(d_J\vartheta\Delta)$.
For $K\geq K_\vartheta$, choose $K_\vartheta$ so that $\log2/\log(K-1)\leq\vartheta/8$, and choose $\Delta=c_\vartheta\min\{1,\sqrt{\log K/N}\}\leq1/4$ so that the normalized KL term is at most $\vartheta/8$. Fano then gives error at least $1-\vartheta/4$, hence $R_{\mathrm{pred}}\geq d_J\vartheta^2\Delta/32$.
For $4\leq K<K_\vartheta$, the common-measure argument gives separation at least $d_J\Delta(\vartheta-\mathrm{TV})$. Choosing $\Delta=c_\vartheta\min\{1,N^{-1/2}\}$ to ensure $\mathrm{TV}\leq\vartheta/2$ proves the same rate, since $\log K$ is bounded in terms of $\vartheta$.

For the realized-coverage step, fixed assignment gives $\E H_a\geq hT$ with $h=3\vartheta/4$. Hence $\Prob(H_a\geq hT/2)\geq h/2$. With probability at least $h/4$, at least $h(K-1)/4$ arms meet this count. Apply~\eqref{eq:bernoulli-max-lower} to $m=\lceil hT/2\rceil$ hidden outcomes in each of $M=\max\{1,\lfloor h(K-1)/4\rfloor\}$ such arms. On $J$, their Bernoulli probabilities are bounded away from zero and one by constants depending on $J$. This gives $R\geq c_{J,\vartheta}s_T$. Combining it with $R\geq c_{\alpha,J,\vartheta}r-C_1s_T$ completes the extension.

If $[q_\alpha,(5/4)q_\alpha]\subset\mathrm{int}(J)$, both original algorithms also satisfy~\eqref{eq:general-nondegeneracy} for $N\geq C_{\alpha,c_\gamma,J,\vartheta}$. In Lemma~\ref{lem:threshold-occupation}, replace $h_\alpha$ by the positive uniform CDF margin of $J$. Choose the empirical-CDF tolerance and then $N$ so that $\Prob(1-L_T^J(a)>\tau)\leq\delta$, with $\tau=\delta=(1-\vartheta)/2$. Then $\E L_T^J(a)\geq(1-\delta)(1-\tau)>\vartheta$. The constants depend on the interval's separation and CDF margins and on $\vartheta$ and $c_\gamma$, but not on $K,T,\pi_{\min}$. These choices remain comparison criteria and proof constants, not algorithm inputs.

\subsection{Proof of Theorem~\ref{thm:cdq}}\label{app:cdq}
\begin{proof}
Write $p_0=\pi_{\min}$, $N=Tp_0$, and $\Lambda=\log(2KT/\delta)$. For an arm, abbreviate the logged and counterfactual score CDFs by $F_b,F_c$, and put $Q_b=F_b^{-1}$. The fixed point is $\alpha^*=1-F_b(q^*)$, and
\begin{align}
h(\beta)=F_c(Q_b(1-\beta))-(1-\alpha),\qquad
-\lambda_{\max}\leq h'(\beta)\leq-\lambda_{\min}<0.
\end{align}
The projected indices belong to $\mathcal I_a$ from round two onward, regardless of the initial empty-buffer prediction.

\textbf{Localizing the empirical quantiles.} Let $t_0=1+\lceil C_0\Lambda/p_0\rceil$. The claim is immediate if $t_0>T$. Under the fixed logger, logged scores at their arrival times are i.i.d. from $F_b$. Chernoff's inequality for action counts and DKW uniformly over arms and sample sizes imply, with probability at least $1-\delta/2$,
\begin{align}\label{eq:dkw-qt}
\max_a|r_t(a)|\leq C\sqrt{\frac{\Lambda}{tp_0}},\quad t_0\leq t\leq T,\qquad
r_t(a)=\hat q_t(a)-Q_b^{(a)}(1-\alpha_t(a)).
\end{align}
Choose $C_0$ using the local density lower bound, $\varepsilon_q$, and the distance of the index intervals from $0$ and $1$. The empirical quantiles then lie in $\widetilde{\mathcal Q}_a$, uniformly over their data-dependent indices. This follows by inverting the uniform CDF bound locally, with any empirical-rank correction absorbed into its constant. In particular, $\eta_t\leq D_q:=1/f_{\min}+\varepsilon_q$ for $t\geq t_0$ on this event. This argument uses no global score-tail bound.

\textbf{Controlling the recursion on time blocks.} First suppose $\gamma\lambda_{\max}\leq1$. The noiseless projected map contracts toward $\alpha^*$ by at most $1-\gamma\lambda_{\min}$. Put
\begin{align}
e_t(a)&=\alpha_t(a)-\alpha^*(a),\qquad m_t(a)=\E[Z_t(a)\mid\cH_{t-1}],\\
M_t(a)&=m_t(a)-Z_t(a),\qquad
\xi_t(a)=F_c^{(a)}(\hat q_t(a))-F_c^{(a)}(Q_b^{(a)}(1-\alpha_t(a))).
\end{align}
The threshold depends on past scores. Hence $M_t(a)$ is a martingale difference, with magnitude and conditional variance at most $C/p_0$. On the quantile event, $|\xi_t(a)|\leq C|r_t(a)|$, and
\begin{align}
\alpha_{t+1}(a)=\Pi_{\mathcal I_a}\bigl(\alpha_t(a)+\gamma\{h_a(\alpha_t(a))+\xi_t(a)+M_t(a)\}\bigr).
\end{align}
Partition $\{t_0,\ldots,T\}$ into blocks of length $m=\lceil\log2/(\lambda_{\min}\gamma)\rceil$, with a possibly shorter last block. Freedman's maximal inequality and a union bound over arms and blocks give, with probability at least $1-\delta/2$,
\begin{align}
\max_a\max_{1\leq u\leq m}\left|\sum_{i=s}^{s+u-1}M_i(a)\right|
\leq C\left\{\sqrt{\frac{m\Lambda}{p_0}}+\frac{\Lambda}{p_0}\right\}
\end{align}
for every block start $s$, with sums stopping at $T$.

Compare the noisy trajectory on a block with the noiseless projected trajectory started at the same index. Projection secants lie in $[0,1]$, and the secants of $\beta+\gamma h(\beta)$ lie in $[0,1-\gamma\lambda_{\min}]$. Expanding the difference gives a noise term $\gamma\sum_iw_iM_i$, where $0\leq w_i\leq1$ and $w_i$ are nondecreasing in $i$. Abel summation bounds its magnitude by twice the displayed maximal partial sum times $\gamma$. The weights may depend on the whole block, since this comparison is pathwise. The bias contribution is at most $C\sup_{i\text{ in block}}|r_i|$.

Writing $E_j$ for the largest index error at block start $s_j$ and $\nu=C\{\sqrt{\gamma\Lambda/p_0}+\gamma\Lambda/p_0\}$ yields, for full blocks,
\begin{align}
E_{j+1}\leq\tfrac12E_j+C\sqrt{\frac{\Lambda}{s_jp_0}}+\nu,
\qquad
\max_{a,\,t\text{ in block }j}|e_t(a)|
\leq E_j+C\sqrt{\frac{\Lambda}{s_jp_0}}+\nu.
\end{align}
Summing the geometric recursion, and accounting for the first block by $O(m/T)$, gives
\begin{align}\label{eq:mean-square-alpha}
\frac1T\sum_{t=t_0}^T\max_a|e_t(a)|
\leq C\left\{\frac1{\gamma T}+\sqrt{\frac{\Lambda}{Tp_0}}
+\sqrt{\frac{\gamma\Lambda}{p_0}}+\frac{\gamma\Lambda}{p_0}\right\}.
\end{align}
Local inverse-CDF regularity gives $|\hat q_t(a)-q^*(a)|\leq|r_t(a)|+|e_t(a)|/f_{\min}$. Substitution of $\gamma=c_\gamma\sqrt{p_0/T}$ therefore proves~\eqref{eq:cdq-rate} when $N\geq\Lambda^2$ and $\gamma\lambda_{\max}\leq1$.

\textbf{The remaining regimes.} If $N<\Lambda^2$, the claimed right-hand side is at least a positive constant, and the localized bound $\eta_t\leq D_q$ suffices. If $\gamma\lambda_{\max}>1$, then $N<(c_\gamma\lambda_{\max})^2$, so the same localized bound proves the claim after adjusting $C$. Thus neither an effective-sample restriction nor a density-dependent choice of $c_\gamma$ is needed. The two failure probabilities sum to at most $\delta$.
\end{proof}

\subsection{Proof of Theorem~\ref{thm:pareto}}\label{app:pareto}
\begin{proof}
Under $P_0$, both arms have law $\cN(0,1)$. Under $P_1$, arm $1$ has the same law and arm $2$ has law $\cN(0,(1+\Delta)^2)$. Potential outcomes are independent across rounds and arms. Write $p=1-\alpha$ and $F(q)=\Prob(|\cN(0,1)|\leq q)$. For $0<\Delta\leq1/4$, equation~\eqref{eq:fano-gap} supplies the uniform separation on $J_\alpha$.

First the coverage premise controls the horizon needed for testing. Let $H_a$ count unobserved outcomes of arm $a$ whose thresholds lie in $J_\alpha$. Since only one outcome is observed each round,
\begin{align}
H_1+H_2\geq T\{L_T(1)+L_T(2)-1\},\qquad
\E_i(H_1+H_2)\geq T/3.
\end{align}
Also $H_1+H_2\leq T$. Hence $H_1+H_2\geq T/6$ with probability at least $1/5$, and then one arm has at least $\lceil T/12\rceil$ such outcomes. Conditional on the complete log and the algorithmic seed, hidden outcomes remain independent even under the adaptive decision rule. Their coverage probabilities on $J_\alpha$ are bounded away from zero and one. The single-arm case of~\eqref{eq:bernoulli-max-lower}, allowing arbitrary offsets, therefore gives
\begin{align}
\sup_{i\in\{0,1\}}\E_i\max_a|\CCov_T(a)-p|\geq c_0(\alpha)/\sqrt T.
\end{align}
The accuracy premise implies $T^{-1/2}\leq\delta/c_0(\alpha)$.

Let $Q_0,Q_1$ be the complete logged-data laws, and put
\begin{align}
M=\frac1T\sum_tF(\hat q_t(2)),\qquad
M^\Delta=\frac1T\sum_tF(\hat q_t(2)/(1+\Delta)).
\end{align}
The realized-minus-predictable coverage martingale has expected absolute value at most $1/(2\sqrt T)$. Thus, for a constant $C_1=C_1(\alpha)$,
\begin{align}
\E_0|M-p|\leq C_1\delta,\qquad
\E_1|M^\Delta-p|\leq C_1\delta.
\end{align}
On every logged history, $M-M^\Delta\geq d_\alpha\Delta L_T(2)$. Integrating the triangle inequality over the common part of $Q_0,Q_1$ gives
\begin{align}
2C_1\delta\geq d_\alpha\Delta
\{2/3-\mathrm{TV}(Q_0,Q_1)\}.
\end{align}
Choose $\Delta=12C_1\delta/d_\alpha$ and $\delta_0(\alpha)$ small enough that $\Delta\leq1/4$. Then $\mathrm{TV}(Q_0,Q_1)\geq1/2$. Pinsker's inequality and the adaptive KL chain rule yield
\begin{align}
c\leq\mathrm{KL}(Q_1\|Q_0)
=\E_1[N_2]\,\mathrm{KL}(\cN(0,(1+\Delta)^2)\|\cN(0,1))
\leq C\Delta^2\E_1[N_2].
\end{align}
Each pull of arm $2$ under $P_1$ has excess expected squared loss at least $2\Delta$. Therefore
\begin{align}
\E_1[\mathrm{Regret}_T]\geq2\Delta\E_1[N_2]
\geq c/\Delta=c'(\alpha)/\delta.
\end{align}
This proves the implication without a separate horizon condition. It does not assert attainability when the accuracy premise is infeasible.
\end{proof}

\subsection{Proof of Theorem~\ref{thm:coverage_regret_bridge}}
\label{app:bridge}

\begin{proof}
For every action $a$, the Lipschitz assumption gives
\begin{align}\label{eq:hausdorff-risk}
    |d_t(a)-d_t^*(a)|
    \leq
    L\,d_H(\cC_t(a),\cC_t^*(a))
    \leq
    L\eta_t^{\mathrm{set}} .
\end{align}
On exploitation rounds, $a_t=\bar a_t$ and $\bar a_t$ minimizes $d_t(a)$, so
\begin{align}
    d_t(a_t)=d_t(\bar a_t)\leq d_t(a_t^*)\leq d_t^*(a_t^*)+L\eta_t^{\mathrm{set}},
\end{align}
where the second inequality uses~\eqref{eq:hausdorff-risk}. If $Y_t(a_t)\in\cC_t(a_t)$, then $\ell(a_t,Y_t(a_t))\leq d_t(a_t)$. If the outcome is outside the set, assumption~(ii) adds at most $B$. Therefore
\begin{align}
    \ell(a_t,Y_t(a_t))-d_t^*(a_t^*)
    \leq
    L\eta_t^{\mathrm{set}}
    +B\1\{Y_t(a_t)\notin\cC_t(a_t)\}
\end{align}
on exploitation rounds. On exploration rounds, assumption~(iii) gives $d_t(a_t)\leq d_t(\bar a_t)+\Delta_{\max}$, so the same inequality holds with this additional term. Summing pathwise, with $I_t^{\mathrm{exp}}$ denoting the exploration indicator, gives
\begin{align}
    \widetilde R_T
    \leq
    L\sum_{t=1}^T\eta_t^{\mathrm{set}}
    +B\sum_{t=1}^T\1\{Y_t(a_t)\notin\cC_t(a_t)\}
    +\Delta_{\max}\sum_{t=1}^T I_t^{\mathrm{exp}}.
\end{align}
The middle sum equals $T(1-\MCov_T)$. Taking expectations and using $\E[\sum_t I_t^{\mathrm{exp}}]=\varepsilon T$ yields
\begin{align}
    \E[\widetilde R_T]
    &\leq
    L\sum_{t=1}^T\E[\eta_t^{\mathrm{set}}]
    +BT(1-\E[\MCov_T])
    +\varepsilon T\Delta_{\max} \nonumber\\
    &\leq
    L\sum_{t=1}^T\E[\eta_t^{\mathrm{set}}]
    +\alpha BT
    +BT\,\E|\MCov_T-(1-\alpha)|
    +\varepsilon T\Delta_{\max}.
\end{align}
This proves~\eqref{eq:regret-bridge}. If $\ell\in[0,M]$, the perturbation $d_t(a_t^*)-d_t^*(a_t^*)$ is also at most $M$. Replacing its Lipschitz bound by $\min\{M,L\eta_t^{\mathrm{set}}\}$ throughout the argument proves the bounded-loss refinement, including infinite set distances.
\end{proof}

\subsection{Proof of Corollary~\ref{cor:regret-pair}}\label{app:regret}
\begin{proof}
The Gaussian construction has $q_a^{\log}=q_0+q_0\Delta/4+O(\Delta^2)$ and $q_a^*=q_0+q_0\Delta/2+O(\Delta^2)$. Thus
$|F_{\log}^{(a)}(q_a^*)-(1-\alpha)|\geq c_0(\alpha)\rho$.
The Gaussian logged densities have a uniform upper bound $L_0$, so for every threshold, including infinite values,
\begin{align}
L_0|\hat q_t(a)-q_a^*|
\geq c_0(\alpha)\rho-|F_{\log}^{(a)}(\hat q_t(a))-(1-\alpha)|.
\end{align}
Choose $c_0$ smaller if necessary so that $c_0\rho\leq1$, and put $e_t(a)=|F_{\log}^{(a)}(\hat q_t(a))-(1-\alpha)|$ and $\kappa=\min\{1,c_H/L_0\}$. The lower set comparison gives, pathwise,
\begin{align}
\bar\eta_t^{\mathrm{set}}
\geq\kappa\{c_0\rho-e_t(a)\}.
\end{align}
Indeed, the right-hand side is nonpositive when $e_t(a)\geq c_0\rho$. Otherwise $0<c_0\rho-e_t(a)\leq1$, and $\min\{1,(c_H/L_0)(c_0\rho-e_t(a))\}\geq\kappa(c_0\rho-e_t(a))$. This also covers infinite thresholds. Average and take expectations to obtain~\eqref{eq:cor-a-sharpness} with $c'=\kappa c_0$ and $C_{\mathrm{tr}}=\kappa$. Substitution into the separately assumed~\eqref{eq:decision-margin} proves part (a). Any fixed positive cap gives the same order with constants depending on that cap.

For part (b), first suppose $N=T\pi_{\min}\geq\log^2(2KT)$. Apply the localization proof of Theorem~\ref{thm:cdq} with failure probability $\delta_T=(2KT)^{-2}$. Its event $E$ has probability at least $1-\delta_T$, startup length $t_0=O(\log(2KT)/\pi_{\min})$, and
\begin{align}
\sum_{t=t_0}^T\eta_t\leq CT r_T\quad\text{on }E,
\qquad r_T=N^{-1/4}\sqrt{\log(2KT)}.
\end{align}
For bounded losses, the robust-risk perturbation in the proof of Theorem~\ref{thm:coverage_regret_bridge} is at most $\min\{M,LC_H\eta_t\}$. This bounds both infinite-threshold startup rounds and the complement of $E$ without a score moment assumption. Its total expectation is at most
\begin{align}
M\min\{T,t_0-1\}+LC_H CT r_T+MT\delta_T.
\end{align}
Here the minimum is a bound on loss in the proof, not truncation of the algorithm's predictions.

Under the stationary logger, the conditional mean logged miscoverage before observing the current context is
\begin{align}
b_t=\sum_a w_a\{1-F_{\mathrm{bias}}^{(a)}(\hat q_t(a))\},
\qquad \E[\1\{Y_t(a_t)\notin\cC_t(a_t)\}]=\E[b_t].
\end{align}
On $E$ and for $t\geq t_0$, the already assumed marginal logged-density bound gives $b_t\leq b^*+f_{\max}\eta_t$. Elsewhere use $b_t\leq1$. Consequently,
\begin{align}
\sum_t\E[b_t]\leq T b^*+t_0+C f_{\max}T r_T+T\delta_T.
\end{align}
Insert these bounds into the signed miscoverage decomposition in Appendix~\ref{app:bridge}, retaining the exploration term. The startup and failure terms are absorbed into $O(T r_T)$. If $N<\log^2(2KT)$, then $r_T>1$ and the direct bound $\widetilde R_T\leq MT$ proves the same result. There is thus no effective-sample requirement in the corollary.

For the Gaussian construction, with $r=1+\Delta$ and $F$ the unit-scale absolute-residual CDF, oracle logged coverage is
\begin{align}
1-b^*=\tfrac34F(q^*)+\tfrac14F(q^*/r)
=1-\alpha+\tfrac14\{F(q^*)-F(q^*/r)\}\geq1-\alpha.
\end{align}
This proves the stated $\alpha BT$ specialization from the model itself. The fixed logger can also coincide with a robust decision rule: take $a^0(L)=0$, $a^0(H)=1$, loss $\ell(x,a,y)=A\1\{a\ne a^0(x)\}+\min\{|y-5|,B_0\}$ with $A>B_0>0$, and uniform exploration probability $1/2$. Robust exploitation always selects $a^0(x)$, giving the required probabilities $3/4$ and $1/4$. This example verifies compatibility of the upper-bound assumptions, not the separate margin in part (a).
\end{proof}

\subsection{Proof of Corollary~\ref{cor:sharpness-comparison}}\label{app:sharpness-comparison}
\begin{proof}
Summing the CDF-based lower bound in Appendix~\ref{app:regret} over $t_0,\ldots,T$ gives
\begin{align}
\frac1T\sum_{t=t_0}^T\E[\eta_t]
&\geq(1-b_T)c_\eta(\alpha)\rho-\frac1{L_0T}\sum_{t=t_0}^T\E|F_{\log}^{(a)}(\hat q_t(a))-(1-\alpha)|\nonumber
\\&\geq(1-b_T)c_\eta(\alpha)\rho-C_{\mathrm{tr}}\varepsilon_T.
\end{align}
The PW-OCP statement is Theorem~\ref{thm:cdq}. This comparison uses the same post-startup threshold quantity for both methods. It does not infer a raw-threshold expectation bound from a high-probability event. For decision regret, bounded loss controls the complement of that event directly, as shown in Appendix~\ref{app:regret}.
\end{proof}

\subsection{Proof of Corollary~\ref{cor:probability-sharpness}}\label{app:probability-sharpness}
\begin{proof}
Set $N=T\pi_{\min}$, $L_T=\log(2KT)$, and $r_T=N^{-1/4}\sqrt{L_T}$. If $N<L_T^2$, then $r_T>1$ and the result follows from $\mathcal D_T\leq1$. Otherwise apply the proof of Theorem~\ref{thm:cdq} with $\delta_T=(2KT)^{-2}$. On its joint localization and recursion event $E$, all empirical thresholds from $t_0$ onward lie in $\widetilde{\mathcal Q}_a$ and their average error is at most $Cr_T$. Since $q^*(a)\in\mathcal Q_a$ and $f_{\mathrm{true}}^{(a)}\leq\lambda_{\max}f_{\max}$ on this interval,
\begin{align}
\max_a|F_{\mathrm{true}}^{(a)}(\hat q_t(a))-(1-\alpha)|
\leq\lambda_{\max}f_{\max}\eta_t,\qquad t\geq t_0.
\end{align}
On startup rounds and on $E^c$, the probability error is at most one, irrespective of threshold magnitude. Therefore
\begin{align}
\E[\mathcal D_T]\leq\frac{\min\{T,t_0-1\}}T+Cr_T+\delta_T
\leq C' r_T,
\end{align}
using $t_0-1=O(L_T/\pi_{\min})$ and $N\geq L_T^2$. The argument also covers $t_0>T$ by the same startup bound. No moment bound for the raw thresholds is used.
\end{proof}

\subsection{Proof of Proposition~\ref{prop:unprojected-sharpness}}\label{app:unprojected-sharpness}
\begin{proof}
Write $p_0=\pi_{\min}$, $p=1-\alpha$, $\beta_t(a)=\alpha_t(a)$, $\beta_a^*=\alpha^\star(a)$, and $I_a=\mathcal I_a$. Use $F_b^{(a)},F_c^{(a)}$ for the logged and counterfactual score CDFs and $q_a^*=q^*(a)$.

Suppress the action index when this creates no ambiguity. Define
$h(\beta)=F_c(F_b^{-1}(1-\beta))-p,\qquad \beta\in I_a.$
Then $h(\beta^*)=0$ and
$-\lambda_{\max}\leq h'(\beta)\leq-\lambda_{\min}$.
Choose a fixed $r>0$ small enough that
$J_{0,a}=[\beta_a^*-r,\beta_a^*+r],\qquad J_{1,a}=[\beta_a^*-2r,\beta_a^*+2r]\subset I_a$
for every action. This is an analysis choice based on the existing margins.

\textbf{A uniform empirical-quantile event.}
Put $t_0=1+\lceil C_0\Lambda/p_0\rceil$. Under fixed contextual logging, each action's observed scores, in arrival order, are independent draws from its fixed logged law $F_b^{(a)}$. Chernoff bounds for action counts and DKW bounds for these arrival streams, unioned over actions, rounds, and buffer sizes at most $T$, give an event $E_Q$ of probability at least $1-\delta/3$ on which, for $t\geq t_0$,
\begin{equation}\label{eq:unproj-cdf}
n_t(a)\geq ctp_0,\qquad
\sup_{\beta\in J_{1,a}}|\hat Q_t^{(a)}(1-\beta)-(F_b^{(a)})^{-1}(1-\beta)|
\leq C\sqrt{\frac\Lambda{tp_0}}.
\end{equation}
The supremum is over all indices, so it also covers the data-dependent index. Choosing $C_0$ sufficiently large makes the CDF error and the rank correction smaller than the fixed local margins. Local inversion then proves~\eqref{eq:unproj-cdf} and keeps these empirical quantiles in the prescribed density neighborhood. No tail bound outside that neighborhood is used.

For all real $\beta$, put $H_t^{(a)}(\beta)=F_c^{(a)}(\hat Q_t^{(a)}(1-\beta))-p$. The map $H_t^{(a)}$ is nonincreasing. Equation~\eqref{eq:unproj-cdf} at the two endpoints of $J_{0,a}$ implies, on $E_Q$,
\begin{equation}\label{eq:unproj-global-drift}
H_t^{(a)}(\beta)\leq-b\ \text{if }\beta\geq\beta_a^*+r,
\qquad
H_t^{(a)}(\beta)\geq b\ \text{if }\beta\leq\beta_a^*-r,
\end{equation}
where $b=\lambda_{\min}r/2$. Monotonicity extends these inequalities to indices outside the local interval, including those producing infinite thresholds.

\textbf{Unconditional martingale events.}
Let $\mathcal H_{t-1}$ be the observed history. For this proof, write $Z_t(a)$ for the IPW or DR signal used by the respective algorithm. In either case, $M_t(a)=\E[Z_t(a)\mid\mathcal H_{t-1}]-Z_t(a)$
is a martingale difference, with magnitude at most $C/p_0$ and conditional second moment at most $C/p_0$. This follows from IPW unbiasedness, or DR unbiasedness with true propensities, after averaging over the independent current context. Thus
\begin{equation}\label{eq:unproj-update}
\beta_{t+1}(a)=\beta_t(a)+\gamma\{H_t^{(a)}(\beta_t(a))+M_t(a)\}.
\end{equation}
Each actual step has magnitude at most $\gamma/p_0$, and $|H_t^{(a)}|\leq1$.

Maximal Freedman up to $t_0-1$, unioned over actions, gives an event of probability at least $1-\delta/3$ on which
$\max_{a,u<t_0}|\beta_{u+1}(a)-\alpha|
\leq\gamma t_0+C\gamma\left\{\sqrt{\frac{t_0\Lambda}{p_0}}+\frac\Lambda{p_0}\right\}
\leq C\Lambda/\sqrt N.$
Increase $A$ so that $N\geq A\Lambda^2$ makes this at most one. In particular, $\beta_{t_0}(a)\in[-1,2]$ for every action.

Choose $B\geq\max\{4/b,\log2/\lambda_{\min}\}$ and partition the updates from $t_0$ into deterministic blocks of length $m=\lceil B/\gamma\rceil$. Maximal Freedman on each block, with a union bound over arms and blocks, gives another event of probability at least $1-\delta/3$ on which every block starting at $s$ satisfies
\begin{equation}\label{eq:unproj-block-noise}
\gamma\max_a\max_{1\leq u\leq m}\left|\sum_{i=s}^{s+u-1}M_i(a)\right|
\leq\epsilon_N:=C_B\left\{\sqrt{\frac{\gamma\Lambda}{p_0}}+\frac{\gamma\Lambda}{p_0}\right\}.
\end{equation}
Sums stop at $T$. Enlarging $A$ ensures $\epsilon_N\leq r/16$, $\gamma/p_0\leq r/16$, and $\gamma\lambda_{\max}\leq1$. It also ensures $t_0+m\leq T$ after enlarging $A$ by another constant if needed. All these events were constructed for the original, unstopped process. We now reason pathwise on their intersection, which has probability at least $1-\delta$.

\textbf{Entrance and retention without projection.}
Each action enters $J_{0,a}$ during the first full block. Otherwise it stays on one side throughout, since a step smaller than $2r$ cannot jump over that interval. Equation~\eqref{eq:unproj-global-drift} then accumulates an inward displacement of at least $b\gamma m\geq4$. The noise displacement is at most $\epsilon_N$. This contradicts remaining on the same side when the block starts in $[-1,2]$ and $J_{0,a}\subset(0,1)$.

After an entrance into $J_{0,a}$, the index cannot exit $J_{1,a}$. For example, consider a candidate upper exit and the last preceding index at most $\beta_a^*+r$. The next index is at most $\beta_a^*+r+\gamma/p_0$. Until the candidate exit, drift is nonpositive and is at most $-b$ whenever an update begins above this boundary. Decompose this excursion into complete deterministic blocks and at most two partial blocks. Each complete block has net displacement at most $-b\gamma m+\epsilon_N\leq0$. The two partial blocks contribute at most $4\epsilon_N$ outward noise in total, by differences of the partial sums in~\eqref{eq:unproj-block-noise}. The whole outward displacement is therefore at most

$\gamma/p_0+4\epsilon_N\leq5r/16<r,$
contradicting the upper exit. The lower exit is identical. Hence all actions remain in $J_{1,a}$ from the deterministic time $t_1=t_0+m$ onward.

\textbf{Local contraction.}
For $t\geq t_1$, write
$H_t^{(a)}(\beta_t(a))=h_a(\beta_t(a))+\xi_t(a),
\qquad |\xi_t(a)|\leq C\sqrt{\Lambda/(tp_0)}.$
The noiseless map $\beta\mapsto\beta+\gamma h_a(\beta)$ preserves $J_{1,a}$ and contracts toward $\beta_a^*$ by a factor at most $1-\gamma\lambda_{\min}$. Compare the noisy trajectory on each block with the noiseless trajectory starting at the same index. Both trajectories stay in $J_{1,a}$ on the event established above. Expanding their difference gives products of secants in $[0,1-\gamma\lambda_{\min}]$. The weights on successive noise increments are nondecreasing and bounded by one. Abel summation bounds the noise contribution by twice the maximal partial sum in~\eqref{eq:unproj-block-noise}. This argument is pathwise and does not require the weights to be predictable. The bias contribution is at most $C\sup_{t\text{ in block}}\sqrt{\Lambda/(tp_0)}$.

Let $E_j=\max_a|\beta_{s_j}(a)-\beta_a^*|$ at successive block starts $s_j\geq t_1$. Complete blocks satisfy

$E_{j+1}\leq\tfrac12E_j+C\sqrt{\frac\Lambda{s_jp_0}}+2\epsilon_N,$
and the largest error within the block is at most
$E_j+C\sqrt{\Lambda/(s_jp_0)}+2\epsilon_N$.
Summing the geometric recursion gives
\begin{equation}\label{eq:unproj-index-average}
\frac1T\sum_{t=t_1}^T\max_a|\beta_t(a)-\beta_a^*|
\leq C\left\{\frac1{\gamma T}+\sqrt{\frac\Lambda N}
+\sqrt{\frac{\gamma\Lambda}{p_0}}+\frac{\gamma\Lambda}{p_0}\right\}.
\end{equation}
For the decreasing bias sequence, the first block contributes at most $Cm/T$ and the later blocks are bounded by the integral of $C\sqrt{\Lambda/(tp_0)}$. This also handles a shorter final block. Local inverse-CDF regularity implies
$|\hat q_t(a)-q_a^*| \leq C\sqrt{\Lambda/(tp_0)}+|\beta_t(a)-\beta_a^*|/f_{\min}.$
Substituting $\gamma=c_\gamma\sqrt{p_0/T}$ into~\eqref{eq:unproj-index-average} and using $N\geq A\Lambda^2$ proves~\eqref{eq:unproj-raw}.

\textbf{Full-horizon probability error.}
Write $L=\log(2KT)$ and take $\delta=(2KT)^{-2}$, so $\Lambda=3L$. If $N\geq9AL^2$, the preceding event controls post-$t_1$ probability errors by the raw threshold error times the local density bound $\lambda_{\max}f_{\max}$. The initial fraction is

$\frac{t_1-1}{T}\leq C\{L/N+N^{-1/2}\}.$
Every probability error is at most one, including startup predictions and the complement of the event. Therefore
$\E\left[\frac1T\sum_t\max_a|F_c^{(a)}(\hat q_t(a))-p|\right]
\leq C\{N^{-1/4}\sqrt L+L/N+N^{-1/2}\}+(2KT)^{-2}
\leq C N^{-1/4}\sqrt L.$
If $N<9AL^2$, the rate expression is bounded below by a positive constant depending only on $A$, and the unit upper bound suffices after adjusting $C$. Combining with that same unit bound proves~\eqref{eq:unproj-probability} for every horizon.
\end{proof}

\section{Experimental Details}
\label{app:experiments}
\raggedbottom

\subsection{Benchmark Design}

\textbf{Common settings.} Unless otherwise stated, all experiments use $\alpha=0.1$, $T=20{,}000$, and $10$ random seeds. Baselines retain their default step size $\gamma=0.005$. PW-OCP and DR-OCP use the coverage-motivated step size $\gamma^\star=\sqrt{\pi_{\min}/T}$ from Theorem~\ref{thm:pwocp}, with $\pi_{\min}$ taken from the logging policy of each benchmark. Figures show means with error bars or bands spanning $\pm1.96$ standard errors across seeds.

DR-OCP fits the contextual outcome model $\hat\mu_t(a,X_t)$ in~\eqref{eq:dr-signal} as an exponential moving average (decay $0.01$) of the per-arm-per-context miscoverage indicator. Both PW-OCP and DR-OCP clip the iterate $\alpha_t(a)$ to $[-0.5,1.5]$ for numerical stability. This implementation choice is distinct from the unprojected coverage analysis. The nonconformity score is $s(x,y;a)=|y|$ for the mean-shift simulation and $s(x,y;a)=|y-\hat\mu(x,a)|$ for the variance-shift and multi-arm simulations.

\textbf{Baseline implementation notes.} ACI uses the original update \citep{gibbs2021adaptive}. Conformal PID implements the proportional$+$tan-saturated-integrator update of \citep{angelopoulos2023conformal} (Scorecaster omitted in pure-bandit settings). SAOCP follows \citep{bhatnagar2023improved} with per-arm scale auto-tuned from the first observed score, matching the salesforce/online\_conformal reference. AgACI uses BOA with pinball loss over a bank of $k=8$ logarithmically spaced step sizes \citep{zaffran2022adaptive}. NEx-CP follows \citep{barber2022conformal} with the recommended geometric decay weight. FACI follows \citep{gibbs2022conformal} with $k=8$ alpha-streams and sigma-mixing. ECI follows the quantile-space update of \citep{wu2025error}; we expose it through the alpha-index parameterization with the corresponding sign-corrected error-quantification term. Online COPP uses the sliding-window inverse-propensity-weighted quantile in Algorithm~\ref{alg:online-copp}.

\textbf{All-action excess loss.} The last column of Table~\ref{tab:master} uses the mean-shift experiment at $\Delta=3$ with a shared logger. For method $m$, its cumulative excess loss is
\begin{align}\label{eq:table-excess-loss}
\mathcal E_m=5\sum_{t=1}^T\sum_{a\in\cA}
\left[(|Y_t(a)|-\hat q_t^m(a))_+-(|Y_t(a)|-\hat q_t^{\mathrm{Oracle}}(a))_+\right].
\end{align}
The reported reduction is normalized as $100(1-\overline{\mathcal E}_m/\overline{\mathcal E}_{\mathrm{ACI}})\%$, with averages over seeds. This diagnostic evaluates threshold exceedances across all actions under the same logged data. It is distinct from the decision criteria $\widetilde R_T$ and $R_T^*$. Since increasing a threshold can only decrease its exceedance loss, this column is interpreted jointly with coverage and prediction-set width.

\textbf{Mean-shift simulation.} Two actions and two contexts, $\cX=\{L,H\}$. Here $\Delta$ denotes the mean-shift magnitude. For $X=L$, $Y(0)\sim\cN(0,1)$ and $Y(1)\sim\cN(\Delta,1)$; the roles swap for $X=H$. The logger mixes a softmax distribution over negative absolute conditional means at temperature $0.3$ with uniform exploration of probability $\varepsilon=0.05$, giving the floor $\pi_{\min}=\varepsilon/2=0.025$.

\textbf{Variance-shift simulation.} Two actions and two contexts with conditional mean fixed at $\mu=5$ throughout. Here $\Delta$ denotes the standard-deviation increment. In context $L$ the standard deviations are $(\sigma_0,\sigma_1)=(1,1+\Delta)$; in context $H$ they swap to $(1+\Delta,1)$. The residual score $|y-\hat\mu(x,a)|$ is paired with a risk-averse logger that applies softmax to negative conditional standard deviations with temperature $0.3$ and mixes in uniform exploration with probability $\varepsilon=0.05$.

\textbf{Multi-arm scaling.} Symmetric setup with $K$ contexts and $K$ actions. Context $k$ favors action $k$ (mean $0$), while every other action has mean-shift magnitude $\Delta=2$. We sweep $K\in\{2,5,10,20,50\}$. The logger uses the same softmax rule as in the mean-shift simulation, with uniform exploration probability $\varepsilon_K=\max\{0.05,2/K\}$ and probability floor $\pi_{\min}=\varepsilon_K/K$. At $K=20$, this yields $\varepsilon_K=0.1$ and $\pi_{\min}=0.005$.

\begin{figure}[!htbp]
\hspace{1.5em}%
    \begin{subfigure}{0.32\textwidth}\centering\small{Mean shift}\end{subfigure}\hfill
    \begin{subfigure}{0.32\textwidth}\centering\small{Variance shift}\end{subfigure}\hfill
    \begin{subfigure}{0.32\textwidth}\centering\small{Multi-arm scaling}\end{subfigure}
    \\[2pt]
    \rowlabel{\qquad\qquad MCov}%
     \begin{subfigure}[b]{0.32\textwidth}\centering
        \includegraphics[width=\textwidth]{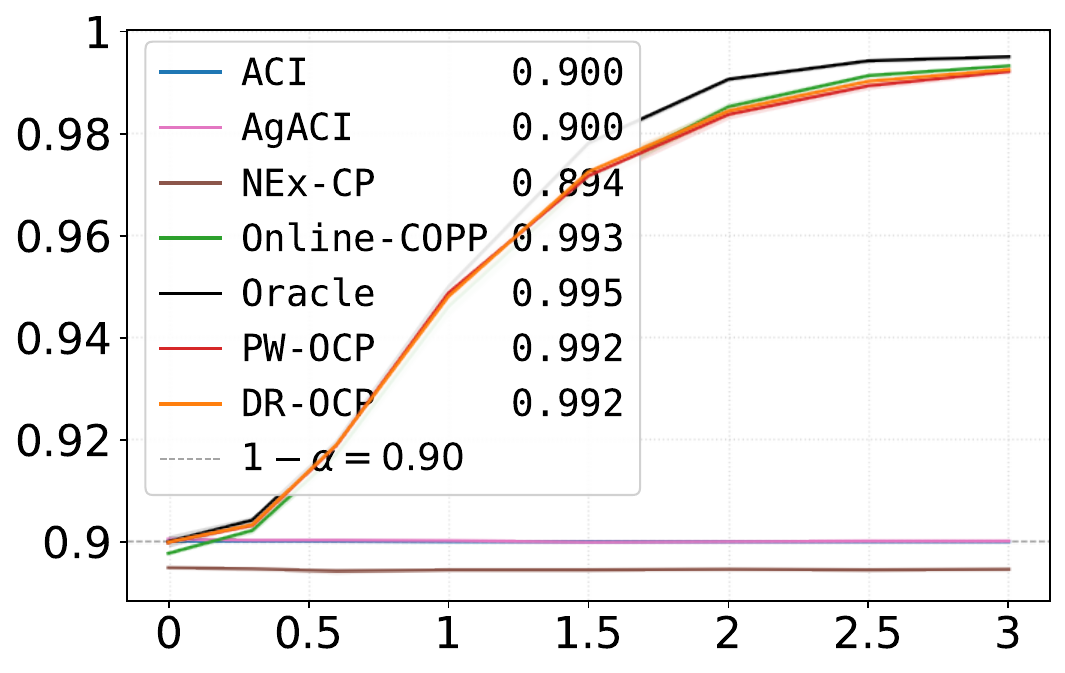}
    \end{subfigure}\hfill
    \begin{subfigure}[b]{0.32\textwidth}\centering
        \includegraphics[width=\textwidth]{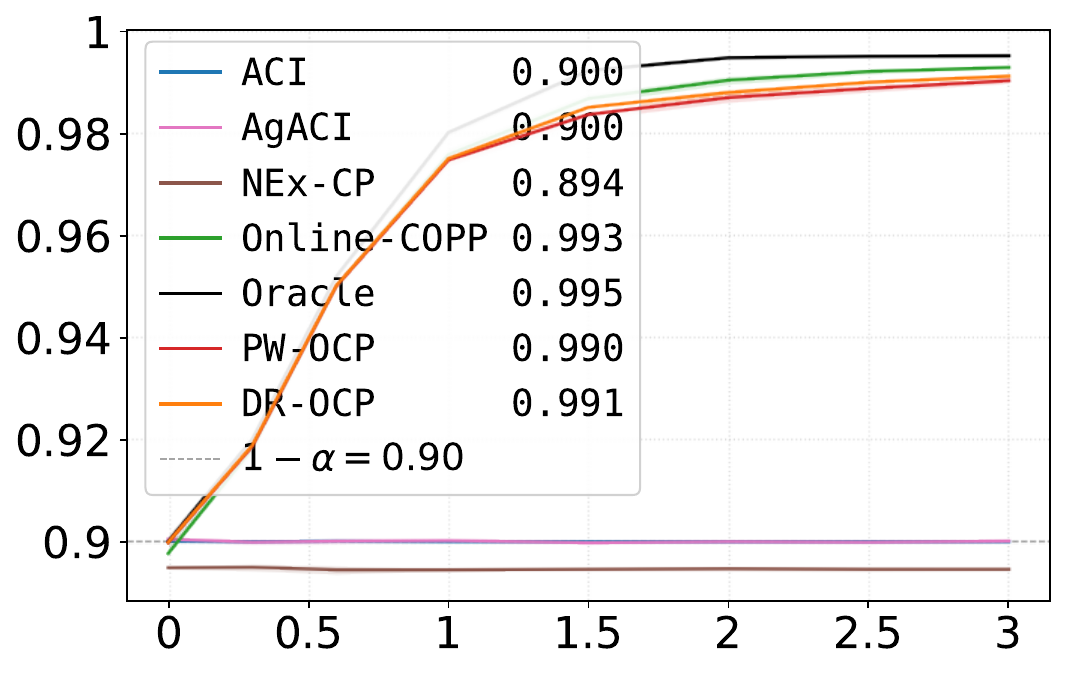}
    \end{subfigure}\hfill
    \begin{subfigure}[b]{0.32\textwidth}\centering
        \includegraphics[width=\textwidth]{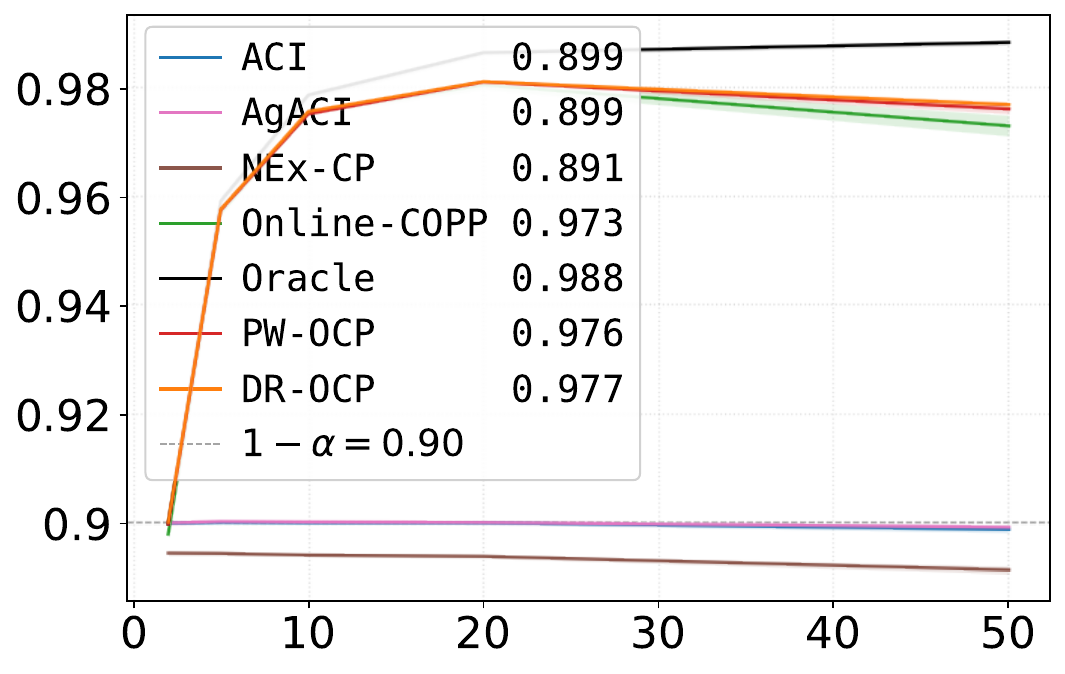}
    \end{subfigure}
    \\[2pt]
    \rowlabel{\qquad\quad ACov Gap}%
     \begin{subfigure}[b]{0.32\textwidth}\centering
        \includegraphics[width=\textwidth]{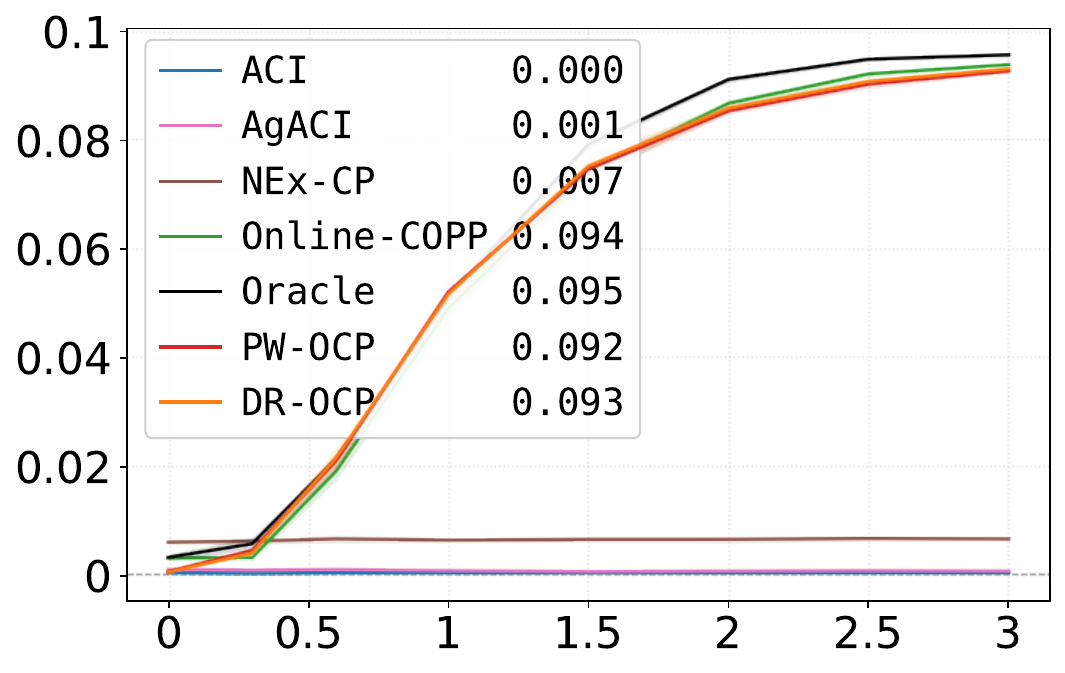}
    \end{subfigure}\hfill
    \begin{subfigure}[b]{0.32\textwidth}\centering
        \includegraphics[width=\textwidth]{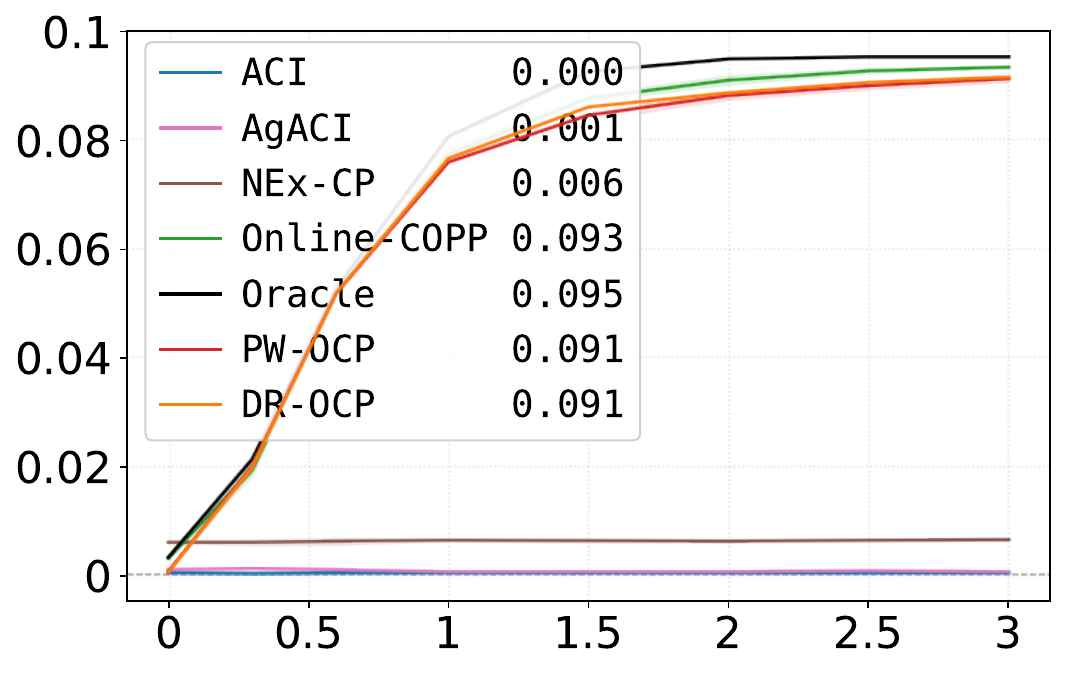}
    \end{subfigure}\hfill
    \begin{subfigure}[b]{0.32\textwidth}\centering
        \includegraphics[width=\textwidth]{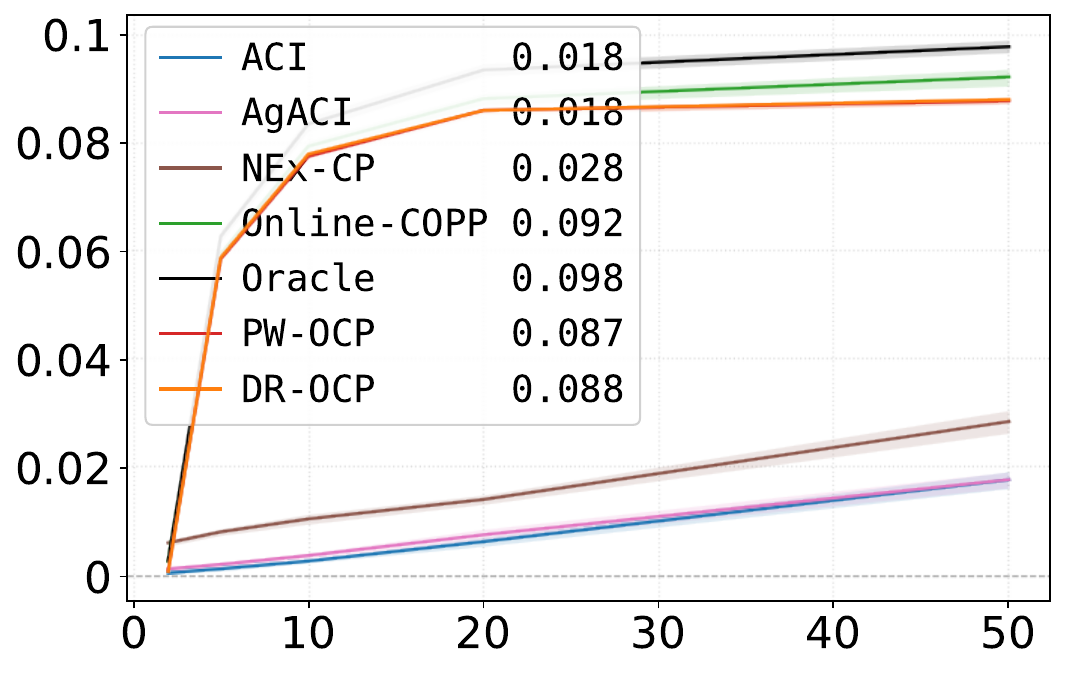}
    \end{subfigure}
    \vskip -0.05in
    \begin{subfigure}{0.32\textwidth}
        \centering
        \qquad \small{\qquad\quad Endogeneity strength}
    \end{subfigure}
    \begin{subfigure}{0.32\textwidth}
        \centering
        \setlength{\parindent}{0.5em}
        \qquad \small{\qquad Endogeneity strength}
    \end{subfigure}
    \begin{subfigure}{0.32\textwidth}
        \centering
        \setlength{\parindent}{1em}
        \quad \small{\quad Number of arms $K$}
    \end{subfigure}
    \vskip -0.05in
    \caption{\textbf{Logged coverage diagnostics across synthetic settings.} The top row reports marginal coverage $\MCov_T$ and the bottom row reports action-conditional coverage gaps for the same mean-shift, variance-shift, and multi-arm settings as Figure~\ref{fig:synthetic-coverage}. These logged diagnostics can appear well calibrated even when the corresponding counterfactual gaps remain large, illustrating why logged metrics alone are insufficient for counterfactual decision-making.}
\label{fig:observable-diagnostics}
\end{figure}

\textbf{Propensity misspecification.} The learner uses a biased propensity model
\begin{align}
    \hat\pi=(1-\xi)\pi+\xi/K,\qquad \xi\in[0,1],
\end{align}
while data are generated under the true logging policy $\pi$. The master table reports the severe case $\xi=1$ (uniform propensity model); Figure~\ref{fig:propensity} sweeps $\xi$.

\textbf{Rate verification and prediction-set quality.} For the convergence-rate study, $K=10$ is fixed while $T$ ranges over $\{2{,}000, 5{,}000, 10{,}000, 20{,}000, 50{,}000\}$ and $\pi_{\min}$ ranges over $\{0.003, 0.005, 0.01, 0.02, 0.05\}$. The prediction-set quality study reports the running counterfactual coverage gap and finite-threshold error. At each round, we take the maximum oracle-quantile error over actions with finite thresholds, then average over rounds with at least one finite threshold. As $T$ increases from $20{,}000$ to $150{,}000$, the mean proportion of non-finite action--round predictions decreases from $10.49\%$ to $0.73\%$ for PW-OCP and from $9.35\%$ to $0.52\%$ for DR-OCP.

\begin{figure}[!htbp]
\hspace{1.5em}%
    \begin{subfigure}{0.48\textwidth}\centering\normalsize{Horizon sweep}\end{subfigure}\hfill
    \begin{subfigure}{0.48\textwidth}\centering\normalsize{Exploration sweep}\end{subfigure}
    \\[2pt] 
    \rowlabel{\qquad\qquad\qquad \qquad \normalsize{CCov Gap}}%
     \begin{subfigure}[b]{0.48\textwidth}\centering
        \includegraphics[width=\textwidth]{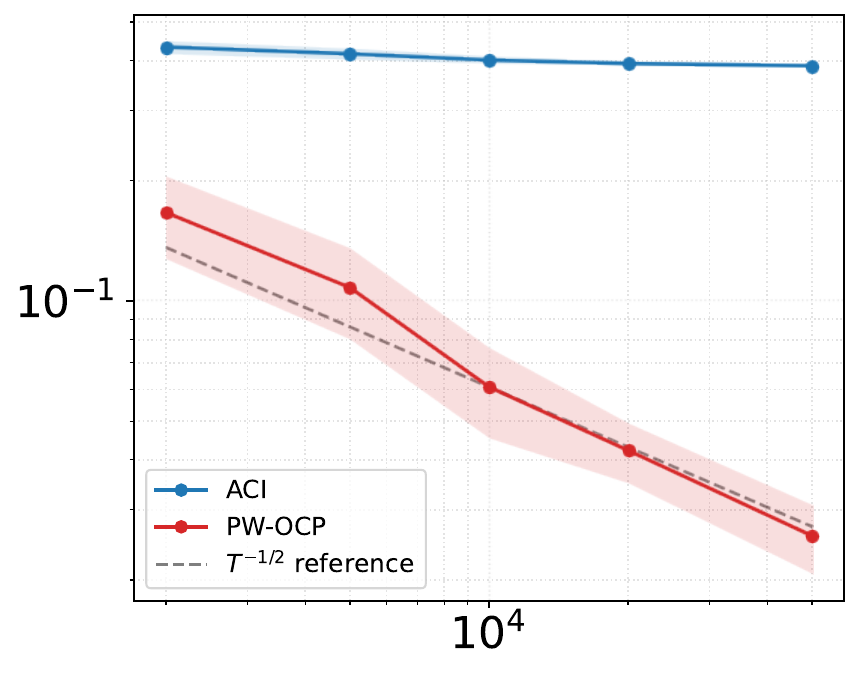}
    \end{subfigure}\hfill
    \begin{subfigure}[b]{0.48\textwidth}\centering
        \includegraphics[width=\textwidth]{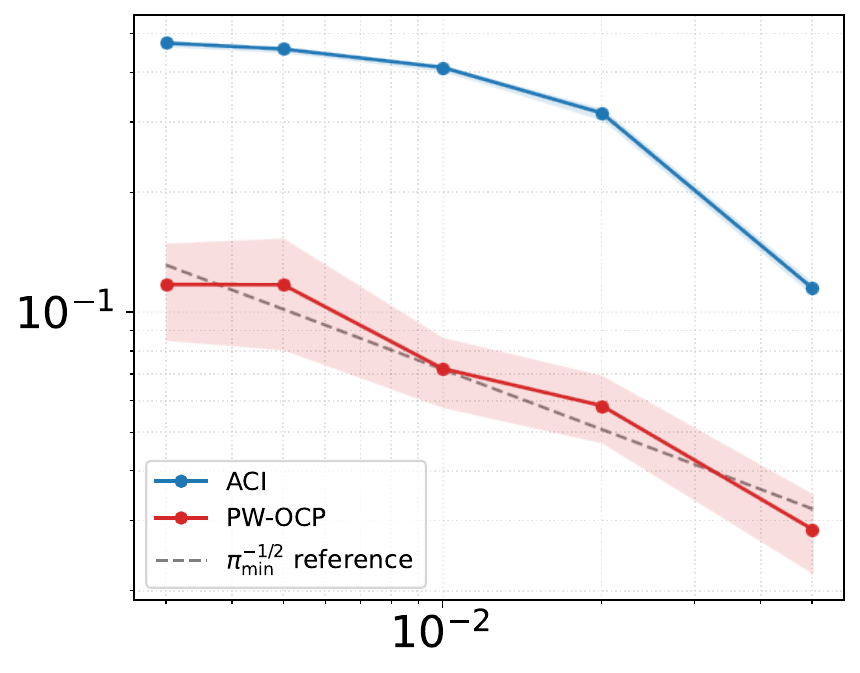}
    \end{subfigure}
    \vskip -0.05in
    \begin{subfigure}{0.48\textwidth}
        \centering
        \qquad \normalsize{\qquad\qquad Number of rounds $T$}
    \end{subfigure}
    \begin{subfigure}{0.48\textwidth}
        \centering
        \setlength{\parindent}{0.5em}
        \qquad \normalsize{\qquad Exploration probability $\pi_{\min}$}
    \end{subfigure}
    \vskip -0.05in
    \caption{\textbf{Coverage-gap scaling for PW-OCP.} With $K=10$, the left panel varies $T\in\{2{,}000,5{,}000,10{,}000,20{,}000,50{,}000\}$ at fixed $\pi_{\min}=0.01$, and the right panel varies $\pi_{\min}\in\{0.003,0.005,0.01,0.02,0.05\}$ at fixed $T=8{,}000$. Curves compare ACI and PW-OCP, with $95\%$ confidence bands over $10$ seeds. Dashed lines indicate inverse-square-root reference slopes.}
\label{fig:rate-verification}
\end{figure}

\begin{figure}[!htbp]
\centering
\includegraphics[width=0.70\textwidth]{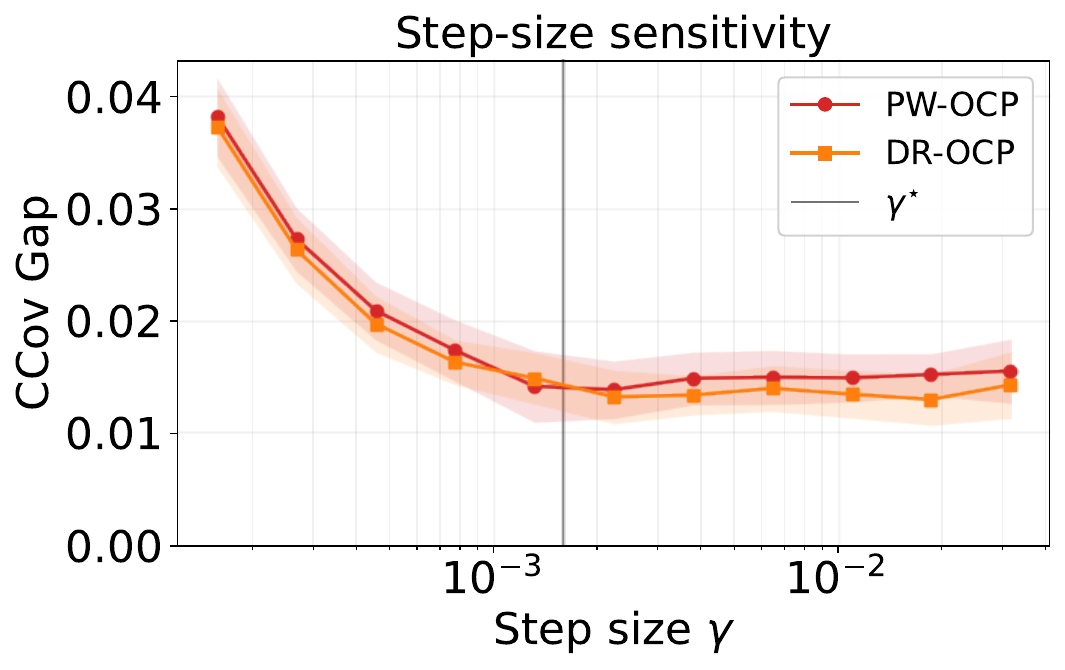}
\caption{\textbf{Step-size sensitivity.} Counterfactual coverage gap is plotted against $\gamma$ for $K=5$, $T=20{,}000$, and $\pi_{\min}=0.05$. The dashed vertical line marks $\gamma^\star=\sqrt{\pi_{\min}/T}\approx1.58\times10^{-3}$. PW-OCP performs best near this scale, and both methods remain stable over a broad neighborhood of $\gamma^\star$. Points show means, and error bars indicate $95\%$ confidence intervals over $20$ seeds.}

\label{fig:gamma-abl}
\end{figure}

\subsection{Open Bandit and Financial Logging}
The Open Bandit experiments use the ZOZOTOWN Open Bandit Dataset \citep{saito2020large}. For both Women and Men campaigns, we keep the top $15$ items by Bernoulli Thompson sampling frequency and aggregate user features into four contexts. We estimate item-context click rates from random-policy data and use these estimates to construct continuous potential rewards $Y(a)\sim\cN(\mu(x,a),\sigma(x,a)^2)$ while preserving observed heterogeneity. The DJIA experiment uses a momentum-biased logging policy over AAPL, MSFT, BA, and KO across $T=2{,}500$ trading days.

% \newcommand{\rowlabel}[1]{\makebox[1.5em]{\rotatebox{90}{\small{#1}}}}

%%%%%%%%%%%%%%%%%%%%%%%%%%%%%%%%%%%%%%%%%%%%%%%%%%%%%%%%%%%%

\begin{figure}[!htbp]
\centering
\makebox[\textwidth][c]{\includegraphics[width=1.0\textwidth]{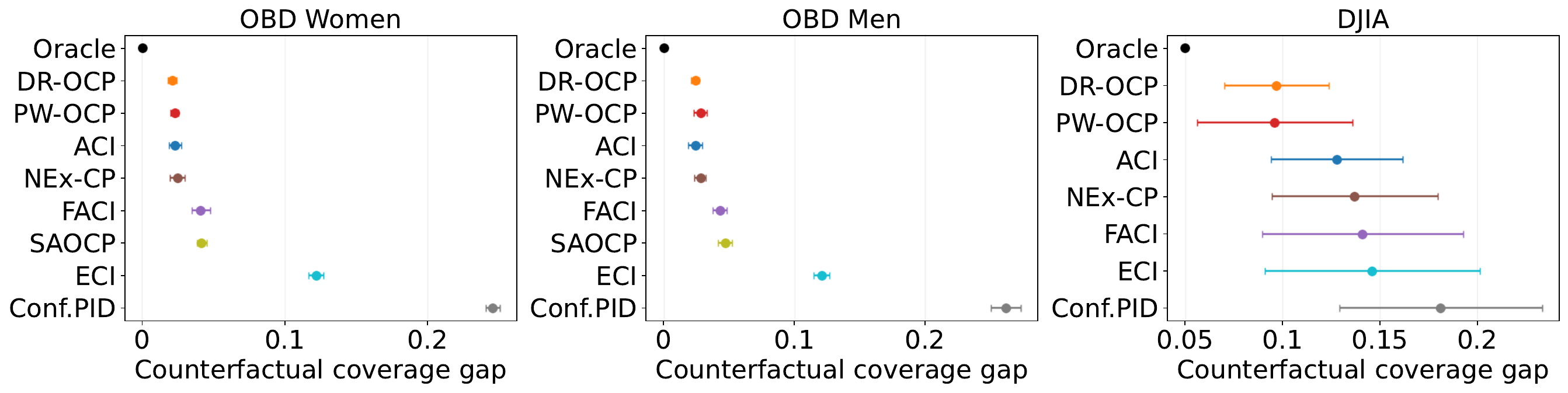}}
\caption{\textbf{Counterfactual coverage on Open Bandit and DJIA tasks.} The three panels report the same counterfactual coverage-gap metric as Table~\ref{tab:master} for OBD Women, OBD Men, and DJIA portfolio rebalancing. Absolute counterfactual gaps are smaller than in the synthetic stress tests, and PW-OCP and DR-OCP remain among the best non-oracle methods across the three benchmarks.}
\label{fig:realdata}
\end{figure}

\subsection{Prediction-set efficiency}\label{app:cr-width}
The experiments below use $20$ seeds, $\alpha=0.1$, and $T=20{,}000$. PW-OCP and DR-OCP use $\gamma=\sqrt{\pi_{\min}/T}$ and indices in $[0.001,0.999]$. Online COPP uses $W=500$. Finite-set gaps average miscoverage over finite-threshold rounds for each action before taking the maximum deviation from $\alpha$. Interval-score excess is relative to Oracle. Full-set rates are computed over action-round pairs.

Table~\ref{tab:cr-width} uses five actions, mean shift $\Delta=2$, and exploration mixture $\varepsilon=0.25$. Full-set rates after the first $1{,}000$ rounds are below $0.001$. ACI under-covers the three context-selected actions, whereas PW-OCP remains near oracle width on every action. A symmetric control isolates frequency alone: the frequent action's width changes by $-0.4\%$ of oracle width under skewed allocation, with normal-approximation $95\%$ CI $[-1.0\%,0.2\%]$, within the prespecified $\pm5\%$ margin.
\begin{table}[H]
\centering
\caption{Per-action pull shares, oracle-relative widths, and finite-set gaps. The first three actions are selected according to context, and the last two through uniform exploration.}
\label{tab:cr-action-width}
\small
\setlength{\tabcolsep}{5pt}
\renewcommand{\arraystretch}{1.08}
\begin{tabular}{lccccc}
\toprule
Action & Pull share & PW width & PW gap & ACI width & ACI gap \\
\midrule
$a_1$ & 0.50 & 95.3\% & 0.030 & 64.8\% & 0.186 \\
$a_2$ & 0.20 & 100.2\% & 0.016 & 71.7\% & 0.222 \\
$a_3$ & 0.20 & 99.9\% & 0.018 & 71.5\% & 0.223 \\
$a_4$ & 0.05 & 101.3\% & 0.009 & 100.9\% & 0.009 \\
$a_5$ & 0.05 & 101.0\% & 0.007 & 100.0\% & 0.007 \\
\bottomrule
\end{tabular}
\end{table}

\subsection{Robustness to limited exploration and imperfect feedback}\label{app:cr-positivity}\label{app:cr-nuisance}
\textbf{Exploration.} With $K=10$, a larger floor improves counterfactual calibration at an increased logging cost. In Table~\ref{tab:cr-floor}, ESS is the minimum action-wise importance-weight effective sample size and logging regret is mean excess loss relative to the context-wise best action. At $\pi_{\min}=0.001$, increasing $T$ from $5{,}000$ to $80{,}000$ reduces DR-OCP's finite-set gap from $0.361$ to $0.144$.
\begin{table}[H]
\centering
\caption{Exploration tradeoff. Gaps are finite-set means with $95\%$ Student-$t$ confidence intervals.}
\label{tab:cr-floor}
\small
\setlength{\tabcolsep}{5pt}
\renewcommand{\arraystretch}{1.08}
\begin{tabular}{lcccc}
\toprule
$\pi_{\min}$ & DR-OCP gap & Online COPP gap & Min-action ESS & Logging regret \\
\midrule
0.001 & 0.226 {\footnotesize$\pm$ 0.011} & 0.149 {\footnotesize$\pm$ 0.018} & 40 & 0.040 \\
0.01 & 0.099 {\footnotesize$\pm$ 0.007} & 0.048 {\footnotesize$\pm$ 0.008} & 222 & 0.199 \\
0.05 & 0.020 {\footnotesize$\pm$ 0.003} & 0.022 {\footnotesize$\pm$ 0.003} & 1062 & 0.910 \\
\bottomrule
\end{tabular}
\end{table}

\textbf{Nuisance misspecification.} The outcome model groups ten contexts into $b$ categories, while the supplied propensities are $(1-\lambda)\pi+\lambda/K$. The largest gaps occur when both nuisances are inaccurate.
\begin{table}[H]
\centering
\caption{Finite-set gaps under nuisance misspecification, with $95\%$ Student-$t$ confidence intervals.}
\label{tab:cr-nuisance}
\small
\setlength{\tabcolsep}{5pt}
\renewcommand{\arraystretch}{1.08}
\begin{tabular}{lcc}
\toprule
Method / nuisance & $\lambda=0$ & $\lambda=1$ \\
\midrule
DR-OCP, $b=1$ & 0.0652 {\footnotesize$\pm$ 0.0064} & 0.3056 {\footnotesize$\pm$ 0.0047} \\
DR-OCP, $b=2$ & 0.0574 {\footnotesize$\pm$ 0.0060} & 0.1313 {\footnotesize$\pm$ 0.0050} \\
DR-OCP, $b=10$ & 0.0575 {\footnotesize$\pm$ 0.0048} & 0.0584 {\footnotesize$\pm$ 0.0051} \\
DR-OCP, analytic nuisance & 0.0555 {\footnotesize$\pm$ 0.0047} & 0.0114 {\footnotesize$\pm$ 0.0013} \\
Online COPP & 0.0360 {\footnotesize$\pm$ 0.0040} & 0.3158 {\footnotesize$\pm$ 0.0062} \\
\bottomrule
\end{tabular}
\end{table}

\textbf{Nonstationarity and delay.} A persistent $0.75$ mean shift after $10{,}000$ rounds raises DR-OCP's whole-horizon gap from $0.0535$ to $0.0673$ and Online COPP's from $0.0349$ to $0.0499$. Over the first $250$ rounds after the shift, the gaps are $0.287$ for DR-OCP and $0.264$ for Online COPP. Online COPP also handles long delays better. Delayed updates use prediction-time thresholds, propensities, and nuisance estimates. The oracle recursion's failure shows that stale feedback matters even without propensity weighting.
\begin{table}[H]
\centering
\caption{Finite-set gaps under delayed feedback. Full-set rates at $d=1{,}000$ include cold start.}
\label{tab:cr-delay}
\small
\setlength{\tabcolsep}{5pt}
\renewcommand{\arraystretch}{1.08}
\begin{tabular}{lcccc}
\toprule
Method & $d=0$ & $d=200$ & $d=1{,}000$ & Full sets \\
\midrule
DR-OCP & 0.0575 & 0.0748 & 0.1374 & 5.0\% \\
Online COPP & 0.0360 & 0.0362 & 0.0377 & 5.0\% \\
Oracle recursion & 0.0003 & 0.0020 & 0.2392 & 68.3\% \\
\bottomrule
\end{tabular}
\end{table}

\textbf{Open Bandit propensity corruption.}\label{app:cr-obd}
The semi-synthetic experiment uses $15$ actions, context-specific logging frequencies with floor $0.02$, and a score center fitted on a separate split. Corruption changes only the supplied propensities. AgACI remains unchanged, while DR-OCP is more stable than PW-OCP.
\begin{table}[H]
\centering
\caption{Open Bandit finite-set gaps under supplied propensities $(1-\lambda)\pi+\lambda/K$.}
\label{tab:cr-obd}
\small
\setlength{\tabcolsep}{5pt}
\renewcommand{\arraystretch}{1.08}
\begin{tabular}{lcccccc}
\toprule
Campaign & $\lambda$ & ACI & AgACI & Online COPP & PW-OCP & DR-OCP \\
\midrule
Men & 0.0 & 0.0257 & 0.0258 & 0.0371 & 0.0270 & 0.0255 \\
Men & 1.0 & 0.0257 & 0.0258 & 0.0369 & 0.2284 & 0.0372 \\
Women & 0.0 & 0.0231 & 0.0229 & 0.0325 & 0.0238 & 0.0216 \\
Women & 1.0 & 0.0231 & 0.0229 & 0.0332 & 0.2023 & 0.0246 \\
\bottomrule
\end{tabular}
\end{table}

\subsection{Score normalization}\label{app:cr-scores}
All variants share the true center $\mu_0=5$ and vary only the scale in $|y-\mu_0|/\widehat\sigma(x,a)$. Exact normalization makes the score pivotal and removes the present impossibility construction without changing outcome-level $\rho$. Scale estimates fitted on past data progressively improve ACI's counterfactual calibration, with structured shrinkage pooling diagonal and off-diagonal context-action groups.
\begin{table}[H]
\centering
\caption{Scale-estimator RMSE and whole-horizon counterfactual gaps. True-scale RMSE is zero by construction.}
\label{tab:cr-scores}
\small
\setlength{\tabcolsep}{5pt}
\renewcommand{\arraystretch}{1.08}
\begin{tabular}{lcccc}
\toprule
Scale choice & Scale RMSE & ACI & PW-OCP & DR-OCP \\
\midrule
Raw residual & --- & 0.2827 & 0.0400 & 0.0289 \\
True scale & 0 & 0.0069 & 0.0137 & 0.0142 \\
EMA scale & 0.373 & 0.1161 & 0.0164 & 0.0149 \\
Shrinkage scale & 0.156 & 0.0350 & 0.0145 & 0.0129 \\
Structured shrinkage & 0.027 & 0.0078 & 0.0136 & 0.0127 \\
\bottomrule
\end{tabular}
\end{table}

\subsection{Prediction-dependent logging}\label{app:cr-policy}
The logger ranks actions by predicted loss plus a weighted conformal half-width:
\begin{align}
    r_t(a)&=\widehat m_t(X_t,a)+0.20\widehat q_t(a),\quad \pi_t(a| X_t,\cH_{t-1})=\frac{\varepsilon}{K}+(1-\varepsilon)\frac{\exp[-r_t(a)/\tau]}{\sum_b\exp[-r_t(b)/\tau]}.
\end{align}
We use $K=5$, $\varepsilon=0.20$, and $\tau=0.05$. The loss center is fitted from past feedback, and the exact sampling probabilities are recorded. The homoscedastic design has standard deviation $0.50$ for all actions. The heteroscedastic design uses $0.25$ for the context-wise best action and $0.75$ otherwise. The same policy formula is used for both designs, with trajectories determined by each method's prediction sets.
\begin{table}[H]
\centering
\caption{Finite-set counterfactual gaps when prediction sets determine logging probabilities.}
\label{tab:cr-policy}
\small
\setlength{\tabcolsep}{5pt}
\renewcommand{\arraystretch}{1.08}
\begin{tabular}{lccc}
\toprule
Outcome scale & ACI & PW-OCP & DR-OCP \\
\midrule
Homoscedastic & 0.0084 & 0.0169 & 0.0145 \\
Heteroscedastic & 0.2895 & 0.0382 & 0.0302 \\
\bottomrule
\end{tabular}
\end{table}

\end{document}